%% file: arxiv.tex
\documentclass{article}
\usepackage{fullpage}
\usepackage{microtype}
\usepackage{graphicx}
\usepackage{subfigure}
\usepackage{booktabs} 

\usepackage{etoolbox}

\usepackage[utf8]{inputenc} 
\usepackage[T1]{fontenc}    
\usepackage{hyperref}       
\usepackage{url}            
\usepackage{booktabs}       
\usepackage{amsfonts}       
\usepackage{nicefrac}   
\usepackage[square,sort,comma,numbers]{natbib}
\usepackage{microtype}      
\usepackage{todonotes}
\usepackage{microtype,enumitem}
\usepackage{mathtools} 
\usepackage{amsmath}
\usepackage{bm}
\usepackage{algorithm2e}
\usepackage{algorithm}
\usepackage{amsthm}
\input{preambule}
\usepackage{wrapfig, framed, caption}
\usepackage{multirow}

\title{A Spectral Analysis of Dot-product Kernels}

\author{
	Meyer Scetbon$^{1}$\qquad
	Zaid Harchaoui$^{2}$\\$^{1}$CREST, ENSAE\\
	$^{2}$ University of Washington
	}
\date{}

\begin{document}

\maketitle

\begin{abstract}
We present eigenvalue decay estimates of integral operators associated with compositional dot-product kernels. The estimates improve on previous ones established for power series kernels on spheres. This allows us to obtain the volumes of balls in the corresponding reproducing kernel Hilbert spaces. We discuss the consequences on statistical estimation with compositional dot product kernels and highlight interesting trade-offs between the approximation error and the statistical error depending on the number of compositions and the smoothness of the kernels.
\end{abstract}

\input{sections/section1-aistats.tex}
\input{sections/section2-aistats.tex}

\input{sections/section3-aistats.tex}
\input{sections/section4-aistats.tex}
\input{sections/section-appli-aistats}

\newpage
\bibliography{biblio}
\bibliographystyle{plainnat}
%
\clearpage
\appendix
\input{sections/supp_mat_aistats_place_holder}

\end{document}

%% file: preambule.tex
\usepackage{hyperref}
\usepackage{amsthm,amsmath}
\usepackage{amssymb}
\usepackage{bm}
\usepackage{bbm}
\usepackage{url}
\usepackage{comment}
\usepackage{multirow}
\usepackage{graphicx}
\usepackage{color}
\usepackage{tikz}
\usepackage{caption}
\usepackage{float}
\usepackage{mathtools}
\usepackage{pifont}
\usepackage[normalem]{ulem}

\newtheorem{assump}{Assumptions}[section]
\newtheorem{coro}{Corollary}[section]

\newtheorem{defn}{Definition}[section]
\newtheorem{prop}{Proposition}[section]
\newtheorem{thm}{Theorem}[section]

\newtheorem{rmq}{Remark}

\newtheorem{lemma}{Lemma}

\DeclareMathOperator*{\argminB}{argmin}   
 
\DeclareMathOperator*{\degreef}{\text{df}} 

%% file: sections/section1-aistats.tex
\section*{Introduction}
Dot product kernels are important tools to tackle signal or image data in machine learning, statistical estimation, and computational mathematics~\citep{steinwart2008support,MR1837879,MR2131724,zbMATH01602327}. Normalizing signal and image data to lie on a sphere is common in signal processing and computer vision~\citep{DBLP:journals/ftcgv/MairalBP14}. The shape and the volume of the reproducing kernel Hilbert space is reflected through the decay of the eigenvalues of the associated integral operator. 

The spectrum of eigenvalues of an integral operator associated with the Gaussian radial basis function kernel was first presented by~\citet{smola2001regularization}. The subject was further explored in several papers~\citep{NIPS2009_3628,zwicknagl2009power,azevedo2014sharp}. Recently, dot product kernels have been considered in relation to the theoretical analysis of deep networks~\citep{daniely2016toward, bach2017breaking, song2018mean} and in relation to the design of new kernel-based methods~\citep{mairal2014convolutional,NIPS2015_6602294b}. 

%
%
We present in this paper general estimates of eigenvalue decay of integral operators associated with dot product kernels of the form
\begin{align}
\label{eq:kernel-gene-form}
    K(x,x):= f(\langle x,x'\rangle_{\mathbb{R}^d})
\end{align}
when the function $f$ satisfies regularity conditions on $[-1,1]$. The eigenvalue decay estimates we obtain generalize previous fundamental results on Mercer decompositions and eigenvalue estimates of dot product kernels~\citep{zwicknagl2009power,azevedo2014sharp}. 

%
%
Spherical harmonic functions are central to our analysis. These special functions arise as the eigenfunctions of the integral operator associated with the simple dot product kernel. We highlight the \emph{relationship between $f$, the smoothness properties of the dot product kernel $K$ and the rate of decay of the eigenvalues of the associated integral operator}. The conditions we provide are concrete and verifiable, boiling down to conditions related to a Taylor series expansion of the kernel. This allows us to characterize the reproducing kernel Hilbert space, obtain estimates of the effective dimension in statistical estimation of eigendecay and show the learning rates of the regularized least-squares algorithm in all regimes.

%
%
The results we present here can potentially be used in a number of contexts. We illustrate them on three examples related to the theoretical analysis of deep networks.
These examples allow us to relate the nonlinear activation functions involved in the construction of a deep network to the spectrum of eigenvalues of an integral operator. 
In particular, we show that, as one iterates the composition of a nonlinear function, the effect on the spectrum is different if the nonlinearity is smooth, as in the case of the exponential or the Swish activation~\citep{DBLP:journals/corr/abs-1710-05941}, or non-smooth, as in the special case of the ReLU activation~\citep{dlbook2016}. 

Furthermore our results also establish sufficient conditions for this family of kernels in~(\ref{eq:kernel-gene-form}) to be universal. The universality of a kernel is a key property which guarantees Bayes-consistency~\citep{steinwart2008support}. We show that the universality can be related to smoothness properties of the function $f$. 
%
%


We start in Sec. 1 with a refresher on spherical harmonics and eigenspectra of integral operators associated with dot product kernels. In Sec. 2, we present our main results on eigendecay estimates for these integral operators. Table~\ref{table:2} summarizes our results. In Sec. 3, we explore the statistical implications for regularized least-squares. Finally, in Sec. 4, we discuss examples related to deep networks.  

\paragraph{Related work}
We give here a brief overview of the related works. The variety of the related works shows the versatility of dot product kernels in machine learning and related fields and the importance of general results on the eigendecay of integral operators.

\textit{Dot product kernels.}~\citet{smola2001regularization} provided in a seminal work estimates of eigenvalue decay for simple dot product kernels. Eigendecay estimates for power series kernels were obtained by~\citet{zwicknagl2009power,azevedo2014sharp} in a particular eigendecay regime. We obtain tight eigendecay estimates in a broad range of regimes.
The results we present can also be potentially applied to recent kernel-based alternatives to deep networks~\citep{shankar2020neural}. 

\textit{Regularized least-squares.}~\citet{caponnetto2007optimal} studied regularized least-square in a reproducing kernel Hilbert spaces in the polynomial regime of eigendecay of the spectrum of the integral operator. The polynomial regime is also common in asymptotic statistical results; see also~\citep{gu2013smoothing} for a review. We extend this line of work by delineating and studying the geometric regime and the super-geometric regime. The analysis requires a careful control of the eigendecay. The tools we develop for this purpose can be of independent interest.  

\textit{Deep networks in kernel regime.} Recent work has shown that a fully connected network, \textit{i.e.}, a multi-layer perceptron, trained with gradient descent may behave like a (tractable) kernel method in a certain over-parameterized regime. See~\citep{NEURIPS2020_9afe487d} for instance. The framework we develop here can be applied to such a tangent kernel and obtain the rate of decay of the eigenvalues of the integral operator associated with the kernel. While our theoretical results cover a broad class of activation functions,
very recent work has considered the special case of the ReLU activation and developed a tailored analysis for that case.~\citet{bietti2020deep} argue that in that case the RKHS remains unchanged regardless of the depth of the neural network; see also~\citep{chen2020deep}. In this paper, we relate the behavior of the coefficients in the Taylor expansion of $f$ to the decay of the eigenvalues of the integral operator. 

Moreover, we cover all regimes of eigendecay, including the regime corresponding to the ReLU activation.

\textit{Kernels on spheres and shallow networks.}~\citet{bach2017breaking} used reproducing kernel Hilbert spaces of dot product kernels to analyze single-hidden layer neural networks with input data normalized on the sphere.
We extend that work in that we analyze neural networks with more than one hidden layer in various eigenvalue decay regimes including the geometric and super-geometric ones which were not previously considered~\citep{bach2017breaking}. Moreover, in contrast to~\citep{bach2017breaking} in which the learning problem is assumed to be realizable,~\textit{i.e.}, the target function is assumed to live in the function space, we work with source conditions which allow us to obtain statistical convergence rates under more general assumptions. 

\textit{Hilbertian envelopes of deep networks.}~\citet{zhang2016l1} used reproducing kernel Hilbert spaces to analyze multi-layer neural networks with smooth activation functions. The family of kernels we consider in~Prop.~\ref{prop:RKHS-MLP} generalizes the one studied in that work
as our kernels are adaptive to the nonlinear functions involved in the construction of the network.~\citet{suzuki2018fast} obtained excess risk bounds for multi-layer perceptrons. 
The eigendecay estimates we obtain result in estimates of effective dimension or degrees-of-freedom of multi-layer perceptrons.

%

%% file: sections/section2-aistats.tex
\section{Dot Product Kernels and their Spectral Decompositions}


Kernels on spheres are ubiquitous in machine learning, statistical estimation, and computational mathematics~\citep{smola2001regularization,steinwart2008support,MR1837879,MR2131724,zbMATH01602327}.
Simple kernels on spheres date back to the seminal works on reproducing kernels~\citep{MR5922}. Examples of simple kernels on spheres include homogeneous polynomial kernels, inhomogeneous polynomial kernels, and Vovk's polynomial kernels~\citep{smola2001regularization,steinwart2008support}.
The analysis of a dot product kernel on the sphere hinges upon a Taylor-like expansion which gives, on the one hand, a spectral decomposition, and on the other hand, a Mercer decomposition. 

\begin{table*}[t!]
\centering
\begin{tabular}{lcccc}
Kernel & $b_m$ &  Space & $\mu$ & $\lambda_m$ \\ 
\hline
\vspace{0.25cm}
$\exp(-c\Vert x-y\Vert_2)$ & $b_m\in\mathcal{O}(m^{-3/2})$  & $S^{d-1}$ & $d\sigma_{d-1}$ &$m^{-d/2}$  \\
\vspace{0.25cm}
$\pi - \arccos(\langle x,x'\rangle)$ & $b_m\in\mathcal{O}(m^{-3/2})$  & $S^{d-1}$ & $d\sigma_{d-1}$ &$m^{-d/2}$  \\
\vspace{0.25cm}
$(2-\langle x,x'\rangle)^{-1}$ & $b_m\in\mathcal{O}(2^{-m})$ & $S^{d-1}$ & $d\sigma_{d-1}$ & $ 
2^{- m}$  \\
\vspace{0.25cm}
$\exp(-b\Vert x-x'\Vert_2^2 )$ & $|b_{m}/b_{m-1}|\in\mathcal{O}(m^{-1})$ & $S^{d-1}$ & $d\sigma_{d-1}$ & $(eb)^m m^{-m+(d-1)/2}$ \\
\hline\\
\vspace{0.25cm}
$\exp(-b(x-x')^2)$ & $|b_{m}/b_{m-1}|\in\mathcal{O}(m^{-1})$  & $[0,1]$ & $\propto\exp(-2ax^2)$ & $(b/(a+b))^{m}$\\
\vspace{0.25cm}
$1+\frac{(-1)^{s-1}(2\pi)^{2s}}{(2s)!}$$B_{2s}(\{x-y\})$ & $/$ & $[0,1]$ & $dx$ & $m^{-2s}$
\end{tabular}
\caption{Eigendecay rates for different kernels. The kernels above the horizontal line are dot-product kernels on the sphere.\label{table:1}}
\end{table*}

\paragraph{Dot product kernel on the sphere.} 
Let $d\geq 2$ and $S^{d-1}$ be the unit sphere of $\mathbb{R}^d$. A kernel of the form
\begin{align}
\label{eq:dot-product-kernel}
    K(x,y)=\sum_{m\geq 0}  b_m (\langle x,y\rangle)^{m}\text{,\quad} x,y\in S^{d-1}
\end{align}
where $(b_m)_{m\geq 0}$ is an absolutely summable sequence is called a \emph{dot product kernel on the sphere} $S^{d-1}$. 

Note that the construction we describe below could be extended to dot product kernels in Hilbert spaces~\citep{MR5922}. If $b_m\geq 0$ for every $m\geq 0$, then $K$ is a \emph{continuous positive semi-definite kernel} on the sphere $S^{d-1}$~\citep{pinkus2004strictly,zwicknagl2009power}. 

\paragraph{Integral operator.}
Let $L_2^{d\sigma_{d-1}}(S^{d-1})$ be the space of real square-integrable functions on the sphere $S^{d-1}$ endowed with its induced Lebesgue measure $d\sigma_{d-1}$ and $|S^{d-1}|$ the surface area of $S^{d-1}$.
Given a positive semi-definite dot product kernel $K$, we define the integral operator on $L_2^{d\sigma_{d-1}}(S^{d-1})$ associated 
$$\begin{array}{ccccc}
T_{K} &: L_2^{ d\sigma_{d-1}}(S^{d-1}) &\to&   L_2^{ d\sigma_{d-1}}(S^{d-1})\\
  	  & f &\to&  \int_{S^{d-1}} K(x,\cdot)f(x)d\sigma_{d-1}(x)
\end{array}$$
By continuity of $K$, $\int_{S^{d-1}} K(x,x)d\sigma_{d-1}(x)$ is finite and $T_K$ is well defined, self-adjoint, positive semi-definite and trace-class~\citep{smale2007learning,steinwart2008support}. 

Denote $H$ the Reproducing Kernel Hilbert Space (RKHS) associated to $K$. The spectral theorem for compact operators~\citep{MR1335452} tells us that for $M\in\mathbb{N}\cup\{+\infty\}$, we have a positive, non-increasing summable sequence $(\eta_m)_{0 \leq m\leq M}$ and a family $(e_m)_{0\leq m\leq M }\subset H$, such that $(\eta_m^{1/2} e_m)_{0\leq m\leq M }$ is an orthonormal system in $H$ while $(e_m)_{0\leq m\leq M }$ is an orthonormal system in $L_2^{d\sigma_{d-1}}(S^{d-1})$ with
\begin{align*}
T_{K} &=\sum_{m=0}^M\eta_m \langle ., e_m\rangle e_m\; .
\end{align*}
where $\langle \cdot,\cdot \rangle$ is in $L_2^{d\sigma_{d-1}}(S^{d-1})$. The system of eigenfunctions of $T_{K}$ is particularly interesting, yet often unknown analytically, except for special classes of kernels. Our class of kernels is one of them.
\medbreak

\paragraph{Spherical harmonics.}
Let $P_{m}(d)$ be the space of homogeneous polynomials of degree $m$ in $d$ variables with real coefficients and $\mathcal{H}_{m}(d)$ be the space of harmonics polynomials defined by 
 \begin{align*}\mathcal{H}_m(d):=\{P\in P_{m}(d)| \Delta P=0\}
  \end{align*}
   where 
   $\Delta \cdot = \sum\limits_{i=1}^d\frac{\partial^2\cdot}{\partial x_i^2} $ 
   is the Laplace operator on $\mathbb{R}^d$~\citep{MR2131724}. Define
 $H_{m}(S^{d-1})$ the space of real spherical harmonics of degree $m$ defined as the set of restrictions of harmonic polynomials in $\mathcal{H}_{m}(d)$ to $S^{d-1}$. Let also $L_2^{d\sigma_{d-1}}(S^{d-1})$ be the space of (real) square-integrable functions on the sphere $S^{d-1}$ endowed with its induced Lebesgue measure $d\sigma_{d-1}$ and $|S^{d-1}|$ the surface area of $S^{d-1}$. $L_2^{d\sigma_{d-1}}(S^{d-1})$ endowed with its natural inner product is a separable Hilbert space and the family of spaces $(H_m(S^{d-1}))_{m\geq 0}$, yields a direct sum decomposition~\citep{frye2012spherical} that reads as
 \begin{align}
 L_2^{d\sigma_{d-1}}(S^{d-1})=\bigoplus_{m\geq 0}H_m(S^{d-1})
 \end{align}
 which means that the summands are closed and pairwise orthogonal.
 Moreover, each $H_m(S^{d-1})$ has a finite dimension $\alpha_{m,d}$ with $\alpha_{0,d}=1$, $\alpha_{1,d}=d$ and for $m\geq 2$
 \begin{align*}
 \alpha_{m,d}=\dbinom{d-1+m}{m}-\dbinom{d-1+m-2}{m-2}
 \end{align*}
 Therefore for all $m\geq 0$, given any orthonormal basis of $H_m(S^{d-1})$, $(Y_{m}^1,...,Y_{m}^{\alpha_{m,d}})$, we can build an Hilbertian basis of $L_2^{d\sigma_{d-1}}(S^{d-1})$ by concatenating these orthonormal bases. Let us denote in the following $(Y_{m}^{l_m})_{m,l_m}$ such an Hilbertian basis of $L_2^{d\sigma_{d-1}}(S^{d-1})$.

\citet{azevedo2014sharp} give a \emph{Mercer decomposition} for a dot product kernel on the sphere of the form~(\ref{eq:dot-product-kernel}). Indeed each spherical harmonics of degree $m$, $Y_{m}\in H_{m}(S^{d-1})$, is an eigenfunction of $T_{K}$ with associated eigenvalue given by the formula
\begin{equation}
\label{eq:formula_eigen-notions}
\begin{aligned}
\lambda_{m}&=\frac{|S^{d-2}|\Gamma((d-1)/2)}{2^{m+1}}\\
&\sum_{s\geq 0}b_{2s+m}\frac{(2s+m)!}{(2s)!}\frac{\Gamma(s+1/2)}{\Gamma(s+m+d/2)}\; . 
\end{aligned}
\end{equation}
Mercer’s theorem then states that the RKHS $H$ associated to the kernel $K$ is the set of functions $f\in L_2(S^{d-1})$ satisfying
\begin{equation}
\begin{aligned}
\label{def:RKHS-Mercer}
     f = &\sum_{\substack{m \geq 0 \\ \lambda_m>0}}\sum_{l_m=1}^{\alpha_{m,d}} a_{m,l_m} Y_m^{l_m}
     \text{\quad s.t.\quad} \\
     &\sum_{\substack{m \geq 0 \\ \lambda_m>0}} \sum_{l_m=1}^{\alpha_{m,d}} \frac{a_{m,l_m}^2}{\lambda_m}<+\infty \; . 
\end{aligned}
\end{equation}
From this definition, we see immediately that as the eigenvalues of the integral operator decreases slower, the volume of the  RKHS becomes larger. More generally, the eigendecay of the integral operator is central to the understanding of a kernel. Note that, in general, the rate of convergence of a sub-sequence of positive $(\lambda_{m})_{m\geq 0}$, ranked in the non-increasing order, is different from the one of $(\eta_{m})_{0\leq m\leq M}$. Indeed we need to take into account the \emph{eigenvalue multiplicities} in order to control the eigendecay. A control of eigendecay is usually out of reach, except for specific kernels on specific domains; see~\citep{steinwart2008support} for a survey and an extended discussion. Indeed upper bounding or lower bounding can quickly result in such loose bounds that they are trivial bounds. A careful control of eigenvalues and their multiplicities is essential.

\paragraph{Eigendecay regimes.}
We distinguish three regimes of decay of eigenvalues: polynomial, geometric, and super-geometric. A polynomial decay corresponds to a rate proportional to $m^{-q}$ with $q>1$; geometric decay to one proportional to $\exp(-\alpha m^q)$ with $\alpha>0$ and $q>0$; super-geometric decay to one faster to geometric decay. We shall see that, depending on the behavior of the coefficients $(b_m)_{m\geq 0}$, dot product kernels relate to one of the above three regimes. In Table \ref{table:1}, we give an overview of dot product kernels on the sphere~\citep{blanchard2008finite,zhang2016l1,bach2017breaking,bach2017equivalence} and give the rates of the sequence $(\lambda_m)_{m\geq 0}$ defined in Eq.~(\ref{eq:formula_eigen-notions}). We also recall the eigendecay for classical kernels from the nonparametric statistics literature~\citep{gu2013smoothing}.

\begin{table*}[t!]
\centering
\begin{tabular}{cccc}
$b_m$ & $\mu_m^{\nu}$  & $\text{df}_{\nu}(\lambda)$ & Rates ($2 \geq \beta>1$)  \\ 
\hline
\vspace{0.25cm}
$b_m\in\mathcal{O}(m^{-\alpha})$, $\alpha>1$ & $m^{-\left(\frac{d/2+\alpha -3/2}{d-1}\right)}$ & $\lambda^{-\frac{d-1}{d/2+\alpha -3/2}} $&  $\ell^{-\frac{\beta}{\beta+q(\alpha,d)}}$, $q(\alpha,d):=\frac{d-1}{d/2+\alpha -3/2}$\\
\vspace{0.25cm}
$b_m\in\mathcal{O}(r^{- m})$, $1>r>0$ &  $e^{- \frac{(d-1)!}{Q_1}\log(1/r)m^{\frac{1}{d-1}}}$ & $\log(\lambda^{-1})^{d-1}$ & $\frac{\log(\ell)^{d-1}}{\ell}$   \\
\vspace{0.25cm}
$\left|b_{m}/b_{m-1}\right|\in O(m^{-\delta})$, $\delta>0$ &  $m^{-\frac{\delta}{s} m^{\frac{1}{d-1}}}$ & $\frac{\log(\lambda^{-1})^{d-1}}{\left(\log(\log(\lambda^{-1}))\right)^{d-1}}$& $\frac{\log(\ell)^{d-1}}{[\log(\log(\ell))]^{d-1}\ell}$ 
\end{tabular}
\caption{Comparison of the convergence rate of regularized least-squares with a dot product kernel on the sphere.\label{table:2}}
\end{table*}

%% file: sections/section3-aistats.tex
\section{Eigenvalue Decay of Dot Product Kernels on the Sphere}
\label{sec:eigen-dot-product}
We show now how to control the eigenvalue decay of an integral operator associated with a dot product kernel $K$ on the sphere introduced in Eq.~\ref{eq:dot-product-kernel}. We exhibit three regimes: polynomial, geometric and super-geometric. We can be in one or the other regime, depending on the coefficients $(b_m)_{m\geq 0}$ involved.
Recall that for such kernels we have an explicit formulation of the eigenvalues $(\lambda_m)_{m\geq 0}$ associated to the integral operator $T_{K}$ given by (\ref{eq:formula_eigen-notions}). In the following we denote $(\eta_m)_{0\leq m\leq M}$ the positive eigenvalues of the integral operator $T_{K}$ associated to the kernel $K$ ranked in a non-increasing order with their multiplicities, where $M\in\mathbb{N}\cup\{+\infty\}$.

\textbf{Super-Geometric Decay.} A first case of interest is the one studied by~\citet{azevedo2014sharp}. There tight estimates for eigenvalues $(\lambda_m)_{m\geq 0}$ are obtained, under the assumption that $|b_{m}/b_{m-1}|\in O(m^{-\delta})$ when $\delta$ is assumed to be strictly bigger than $1/2$. We present here a more general result, holding for any $\delta>0$. See Appendix~\ref{proof:menegato} for proof.
\begin{prop}
\label{prop:menegato}
If there exists $\delta>0$ such that
\begin{align}
\label{eq:b_n-final-f-2}
\left|\frac{b_{m}}{b_{m-1}}\right|\in O(m^{-\delta})
\end{align}
then, denoting $\alpha =1/(1-2\delta)$, we have
\begin{equation*}
\lambda_m \in \begin{cases}
              \mathcal{O}\left( \frac{b_m}{2^{m}m^{({d-2})/2}} \right) & \text{if $\delta\geq 1/2$} \\
              \\
              \mathcal{O}\left( \frac{m^{\frac{m\delta}{2\alpha}+\frac{1}{\alpha}} b_m}{2^{m+1}m^{({d-2})/2}}\right)  & \text{if $0<\delta < 1/2$}
       \end{cases} \quad
\end{equation*}
\end{prop}
To control the eigenvalue decay associated with such dot product kernels, one needs to take into account the eigenvalue multiplicities. From the above control, we  obtain a tight control of the eigenvalue decay of $T_{K}$ ranked a non-increasing order with their multiplicities. See Appendix~\ref{proof:spectrum_anal_gene} for proof.
\begin{prop}
\label{prop:spectrum_anal_gene}
Under the same assumption as Prop. \ref{prop:menegato},
$M=+\infty$ and there exists a universal constant $c>0$ such that
\begin{align*}
\eta_m&\in \mathcal{O}\left(m^{-\frac{\delta}{s} m^{\frac{1}{d-1}}}\right) \; 
\text{where $s=\frac{4c}{(d-2)!}$}\; .
\end{align*}
\end{prop}

    

\textbf{Geometric Decay.} Another case of interest is when the coefficients $(b_m)_{m\geq 0}$  decrease almost geometrically. Indeed we also obtain a tight control of the sequence $(\lambda_m)_{m\geq 0}$ associated and the eigenvalue decay with their multiplicities of the integral operator $T_{K}$. See Appendix~\ref{proof:eta-control} for proof.
\begin{prop}
\label{prop:eta-control}
If there exist 
$0<r<1$ and $0<c_2\leq c_1$ constants such that for all $m\geq 0$
\begin{align}
\label{eq:assump_geometric-spect}
c_2 r^{m} \leq b_m\leq c_1 r^m \;,
\end{align}
then there exists constants $C_1,C_2>0$ such that
\begin{align*}
C_2 \left(\frac{r}{4}\right)^m \leq \lambda_{m}\leq C_1 r^m\; .
\end{align*}
Moreover,  $M=+\infty$ and there exists universal constants $Q_1>Q_2>0$ such that for all $m\geq 0$
\begin{align*}
C_2 e^{-\frac{(d-1)!}{Q_2}\log(4/r)m^{\frac{1}{d-1}}} \leq \eta_{m}\leq C_1 e^{- \frac{(d-1)!}{Q_1}\log(1/r)m^{\frac{1}{d-1}}} \; .
\end{align*}
\end{prop}
\paragraph{Polynomial Decay.} When $(b_m)_{m\geq 0}$ admits a polynomial decay, we manage to control the rate of the sequence $(\lambda_m)_{m\geq 0}$ associated and the eigenvalue decay with their multiplicities of the integral operator $T_{K}$. 

\begin{prop}
\label{prop:poly-control}
If there exists $\alpha>1$ such that 
\begin{align}
\label{eq:assump_geometric-spect}
 b_m \in\mathcal{O}( m^{-\alpha}) \;,
\end{align}
then we have
\begin{align*}
\lambda_{m}\in\mathcal{O} (m^{-d/2 -\alpha + 3/2 })\; ,
\end{align*}
and
\begin{align*}
 \eta_{m} \in\mathcal{O} ( m^{-d/(2d-2) -\alpha/(d-1) + 3/(2d-2)})\; .
\end{align*}
\end{prop}

\paragraph{Approximation of the RKHS.} 
The eigenvalue decay of the integral operator gives here a concrete notion of the complexity of the function space considered. Roughly speaking, if the $(\eta_m)_{m\geq 0}$ decay rapidly, the kernel $K$ can be well approximated with a small number of terms in the Mercer decomposition. More formally, let $(S,d)$ a metric space, $M\subset S$ and $\epsilon>0$. The $\epsilon$-covering number of $M$ with respect to the metric $d$ denoted $\mathbf{N}(\epsilon,M,d)$ is the smallest number of elements of an $\epsilon$-cover for $M$ using the metric $d$. The $n$-th entropy number of a set $M$ for $n\in\mathbb{N}$ is defined as 
$$\varepsilon_n(M):=\inf\{\epsilon\text{:  } \mathbf{N}(\epsilon,M,d)\leq n\}\; .$$
Let $\mathcal{L}(E,F)$ be the set of all bounded linear operators $T$ between the normed spaces $(E,\Vert \cdot \Vert_E)$ and $(F,\Vert \cdot \Vert_F)$. The entropy numbers of an operator $T\in\mathcal{L}(E,F)$ are defined as 
\begin{align*}
    \varepsilon_n(T):=\varepsilon_n(T(B_E))
\end{align*}
where $B_E$ is the closed unit ball of $E$. Obtaining a control of $\varepsilon_n(T_K)$ leads to a control of the generalization error of the kernel-based method using the kernel $K$~\citep{smola2001regularization}. \citet{10.1007/3-540-49097-3_17} obtained a control of such quantities when the integral operator associated with the kernel has a polynomial or geometric eigendecay regime. Combining this with our results, we can obtain a control of the entropy numbers associated with dot product kernels. 

\begin{coro}
\label{coro-entropy-number}
Let $1>r>0$ and $\alpha>1$. We have
\begin{align*}
b_m\in\mathcal{O}(m^{-\alpha})&\implies \varepsilon_n(T_K)\in\mathcal{O}(\log^{-p(\alpha,d)/2}(n))\\
\text{where } p(\alpha,d)&=\frac{d/2+\alpha-3/2}{d-1}\; ,
\end{align*}
Furthermore we have
\begin{align}
\label{coro:entropy-geo}
b_m\in\mathcal{O}(r^m)\implies& |\log(\varepsilon_n(T_K))|\in\mathcal{O}(\log^{1/d}(n))
\end{align}
\end{coro}
Recall that for a compact set $M$ in finite dimensional space of dimension $d$ the entropy number is $\varepsilon_n(M) \in\mathcal{O}(n^{-1/d})$. What~(\ref{coro:entropy-geo}) tells us is that a nonparametric estimator with that function class basically behaves like an estimator defined on a finite-dimensional space. 
To obtain statistical bounds, all that is left is to substitute the above control into the classical uniform convergence results~\citep{boucheron2005theory,steinwart2008support}. 
In the next section, we focus on regularized least-squares (RLS) with dot product kernels, and, leveraging the eigendecay estimates we obtained in the previous section, we parameterize the statistical bounds in terms of the~\emph{effective dimension}. 


%% file: sections/section4-aistats.tex
\section{Statistical Bounds for RLS with Dot-product Kernels}
\label{sec:RLS-dot-product}

We present here general statistical bounds on the performance of regularized least-squares estimator of dot product kernels in all the regimes. These statistical bounds can be used to describe the statistical performance of a regularized least-squares estimator when this estimator can be computed exactly in practice. This applies for instance to the kernel-based deep networks developed by~\citet{shankar2020neural} and to kernel-based methods with kernels on spheres~\citep{steinwart2008support}.
We focus on the approximation error (our results do not assume realizability) and statistical prediction (our results match minimax rates) of regularized least-squares (RLS). 

\paragraph{Learning from data.} 
\medbreak
Given a dataset $\mathbf{z} = {(x_i,y_i)}_{i=1}^{\ell}$ independently sampled from an unknown distribution $\rho(x, y)$ on $Z:=\mathcal{X}\times \mathcal{Y}$ where $\mathcal{Y}\subset\mathbb{R}$, the goal of the least-squares regression is to estimate the conditional mean function $f_{\rho}:\mathcal{X}\rightarrow \mathbb{R}$ given by $f_{\rho}(x):= \mathbb{E}(Y |X = x)$. The joint distribution $\rho(x, y)$, the marginal distribution $\nu$, and the conditional distribution $\rho(.|x)$, are related through $\rho(x, y) =\nu(x)\rho(y|x)$.
Consider as hypothesis space a Hilbert space $H$ of functions $f:\mathcal{X}\rightarrow \mathcal{Y}$. For any regularization parameter $\lambda>0$ and training set $\mathbf{z}\in Z^{\ell}$, the regularized least-squares estimator $f_{H,\mathbf{z},\lambda}$ is the solution of
\begin{align}
\label{eq:rls-obj}
\argminB_{f\in H} \left\{ \frac{1}{\ell}\sum_{i=1}^{\ell}( f(x_i)-y_i)^2+\lambda \Vert f\Vert_{H}^2\right\} \; .
\end{align}
In the following, the input space $\mathcal{X}$ is the sphere $S^{d-1}$ and
the hypothesis space considered is the Hilbert space $H$ associated with the dot product kernel $K$ with coefficients $(b_m)_{m\geq 0}$. Define the integral operator on $L_2^{d{\nu}}(S^{d-1})$ associated as 
$T_{\nu}(f)(y) = \int_{S^{d-1}} K(x,y)f(x) d{\nu}(x)$ and denote $(\mu_m^{\nu})_{0\leq m\leq M}$ its positive eigenvalues ranked in a non-increasing order with their multiplicities, where $M\in\mathbb{N}\cup\{+\infty\}$.
The analysis of the convergence rates of RLS relies on the control of the effective dimension defined as
\begin{align*}
    \text{df}_{\nu}(\lambda):=\text{Tr}\left((T_{\nu}+\lambda)^{-1}T_{\nu}\right)=\sum_{m=0}^M \frac{\mu_m^{\nu}}{\mu_m^{\nu}+\lambda}\; .
\end{align*}
In the following, we manage to obtain tight estimates of the $\text{df}_{\nu}$ when $M=+\infty$ and $(\mu_m^{\nu})_{0\leq m\leq M}$ has a geometric decay or a super-geometric one. Note that~\citet{caponnetto2007optimal} previously obtained such a control in the polynomial decay regime. Applying these controls to the results obtained in the previous section allows us to deduce the convergence rates of RLS for dot product kernels in all the regimes. Table~\ref{table:2} summarizes the control of the quantities of interest as well as the convergence rates obtained for RLS associated to dot product kernels in the different regimes.

We work here under general assumptions on the set of probability measures $\rho$ on $S^{d-1}\times\mathcal{Y}$.
\begin{assump}\textbf{ [Probability measures on $S^{d-1}\times \mathcal{Y}$].} Let $\mathcal{P}$ a set of probability measures on $S^{d-1}$. Furthermore, let $B, B_{\infty}, L, \sigma> 0$ be some constants and $0< \beta \leq 2$ a parameter. Then we denote by $\mathcal{F}_{B, B_{\infty}, L, \sigma,\beta}(\mathcal{P})$ the set of all probability measures $\rho$ on $S^{d-1} \times\mathcal{Y}$ with the following properties. \\
(i) $\nu\in\mathcal{P}$\\
(ii)$\int_{S^{d-1}\times\mathcal{Y}}y^2d\rho(x,y) < \infty$, $\Vert f_{\rho}\Vert_{L_{\infty}^{d\nu}}^2 \leq  B_{\infty}$\\
(iii) There exists $g\in L_2^{d\nu}(S^{d-1})$ such that $f_\rho=T_{\nu}^{\beta/2}g$ and $\Vert g\Vert_{\rho}^2\leq B$\\
(iv)  there exist $\sigma>0$ and $L>0$ such that  $\int_{\mathcal{Y}} |y-f_{\rho}(x)|^m d\rho(y|x)\leq \frac{1}{2}m!\sigma^2 L^{m-2}$.
\end{assump}

For $\omega\geq 1$, we denote by $\mathcal{W}_{\omega}$ the set of all probability measures $\nu$ on $S^{d-1}$ which satisfying $d\nu/d\sigma_{d-1}< \omega$.
Furthermore, we introduce for a constant $\omega \geq 1 > h > 0$, $\mathcal{W}_{\omega,h}\subset\mathcal{W}_{\omega}$ the set of probability measures $\nu$ on $S^{d-1}$ which additionally satisfy 
$d\nu/d\sigma_{d-1}> h$.
In the following we denote $\mathcal{G}_{\omega, \beta}:=\mathcal{F}_{H_N,B, B_{\infty}, L, \sigma,\beta}(\mathcal{W}_{\omega})$
and $\mathcal{G}_{\omega,h,\beta}:=\mathcal{F}_{H_N,B, B_{\infty}, L, \sigma,\beta}(\mathcal{W}_{\omega,h})$. 


\paragraph{Geometric Case} We consider the case corresponding to a geometric eigendecay. Here the  coefficients $(b_m)_{m\geq 0}$ in the Taylor decomposition decrease almost geometrically. The first goal is to obtain a control the of the effective dimension associated with the integral operator $T_{\nu}$. 
\begin{prop}
Let $\omega>0$ and $\nu\in\mathcal{W}_{\omega}$. If there exists 
$0<r<1$ 
such that
\begin{align}
b_m\in\mathcal{O}(r^m) \;,
\end{align}
Then there exists a constant $Q>0$ such that all $0<\lambda\leq e^{-1}$ we have
\begin{align*}
\textsl{df}_{\nu}(\lambda)&\leq Q \log(\lambda^{-1})^{d-1}
\end{align*}
\end{prop}
From the above control, we are now able to show the convergence rates for nonparametric regression in the geometric regime. See Appendix~\ref{proof:rls-super-geo} for the proofs.
\begin{thm}
\label{thm:rls-geo}
Let us assume that there exists 
$0<r<1$ 
such that the sequence $(b_m)_{m\geq 0}$ satisfies:
\begin{align}
 b_m \in\mathcal{O}(r^m)
\end{align}
Let also $w\geq 1$ and $0<\beta\leq 2$.
Then there exists a constant $C>0$ independent of $\beta$ such that for any $\rho\in \mathcal{G}_{\omega, \beta}$ and $\tau\geq 1$ we have:
\begin{itemize}
\item  If $\beta> 1$, then there exists $\ell_{\tau,\beta}>0$ such that for all $\ell\geq \ell_\tau$ and $\lambda_{\ell}=\frac{1}{\ell^{1/\beta}}$, with a $\rho^{\ell}$-probability $\geq 1-e^{-4\tau}$ it holds
    \begin{align*}
\Vert f_{H_N,\mathbf{z},\lambda}-f_{\rho}\Vert_{\rho}^2 &\leq 3C\tau^2\frac{\log(\ell)^{d-1}}{\ell}
    \end{align*}
\item If $\beta = 1$, then there exists $\ell_\tau>0$ such that for all $\ell\geq \ell_\tau$ and $\lambda_{\ell}=\frac{\log(\ell)^{\mu}}{\ell}$, $\mu>d-1>0$, with a $\rho^{\ell}$-probability $\geq 1-e^{-4\tau}$ it holds
    \begin{align*}
    \Vert f_{H_N,\mathbf{z},\lambda_{\ell}}-f_{\rho}\Vert_{\rho}^2 &\leq 3C\tau^2\frac{\log(\ell)^{\mu}}{\ell}
    \end{align*}
\item If $\beta < 1$, then there exists $\ell_{\tau,\beta}>0$ such that for all $\ell\geq \ell_\tau$ and $\lambda_\ell=\frac{\log(\ell)^{\frac{d-1}{\beta}}}{\ell}$, with a $\rho^{\ell}$-probability $\geq 1-e^{-4\tau}$ it holds
    \begin{align*}
    \Vert f_{H_N,\mathbf{z},\lambda_{\ell}}-f_{\rho}\Vert_{\rho}^2 &\leq 3C\tau^2\frac{\log(\ell)^{d-1}}{\ell^{\beta}}
    \end{align*}
\end{itemize}
\end{thm}
Note that we have for any $0<\beta\leq 2$ an explicit formulation of $\ell_{\tau,\beta}$ which depend to the constants of the problem, $\tau$ and $\beta$ but we decide to hide them to simplify the exposition of the results. Moreover the rates obtained for RLS are optimal in the minimax sense and therefore no better rate can be obtained within this nonparametric learning framework. See Appendix~\ref{proof:rls-geo} for the proof.

\paragraph{Super-Geometric Case.} Let us now consider the case corresponding to a super-geometric eigendecay. As in the geometric case, we start by obtaining a control of $\text{df}_{\nu}(\lambda)$ associated with $T_{\nu}$ in this regime. See Appendix~\ref{proof:rls-super-geo} for the proof.
\begin{prop}
Let $\omega>0$ and $\nu\in\mathcal{W}_{\omega}$. If there exist 
$0<\delta<1$
such that
\begin{align*}
\left|\frac{b_{m}}{b_{m-1}}\right|\in \mathcal{O}(m^{-\delta})
\end{align*}
Then there exists a constant $Q>0$ such that all $0<\lambda\leq e^{-1}$ we have:
\begin{align*}
    \textsl{df}_{\nu}(\lambda)\leq Q \frac{\log(\lambda^{-1})^{d-1}}{\left(\log(\log(\lambda^{-1}))\right)^{d-1}}
\end{align*}
\end{prop}
From the above control, we also obtain the convergence rates for nonparametric regression in the super-geometric regime. Table~\ref{table:2} shows the rates obtained in that regime when 
$1<\beta \leq 2$. See Appendix~\ref{proof:rls-super-geo} for the full statement of the convergence rates obtained and the proofs.

\paragraph{Polynomial Case.}~\citet{caponnetto2007optimal} obtained the optimal convergence rates of RLS under the \emph{assumption} that the eigenvalue of the integral operator $T_{\nu}$ admits a polynomial decay for a given kernel $K$. Combining their results and the one obtained in Prop.~\ref{prop:poly-control} gives the convergence rates of RLS with dot product kernels in the polynomial regime. See Table~\ref{table:2} for rates obtained.
As expected, the convergence rate becomes faster as the complexity of the model shrinks,~\textit{i.e.}, the convergence rate of the super-geometric regime is faster than the one obtained in the geometric regime; the latter rate is therefore faster than the one in the polynomial regime.

\subsection{Numerical illustrations}

\begin{figure*}[!ht]
\centering
\begin{tabular}{ c  c  c}
\includegraphics[width=0.3\textwidth, height=0.15\textwidth]{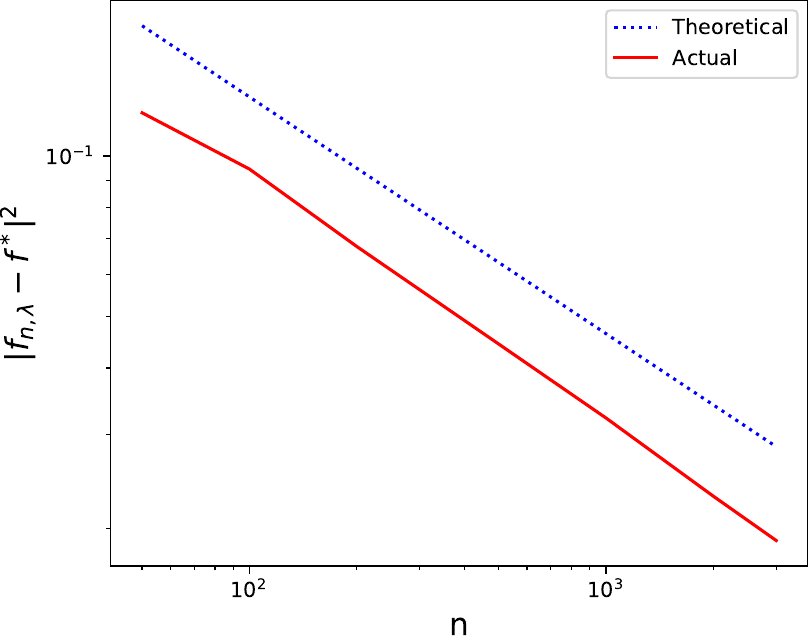} &
\includegraphics[width=0.3\textwidth, height=0.15\textwidth]{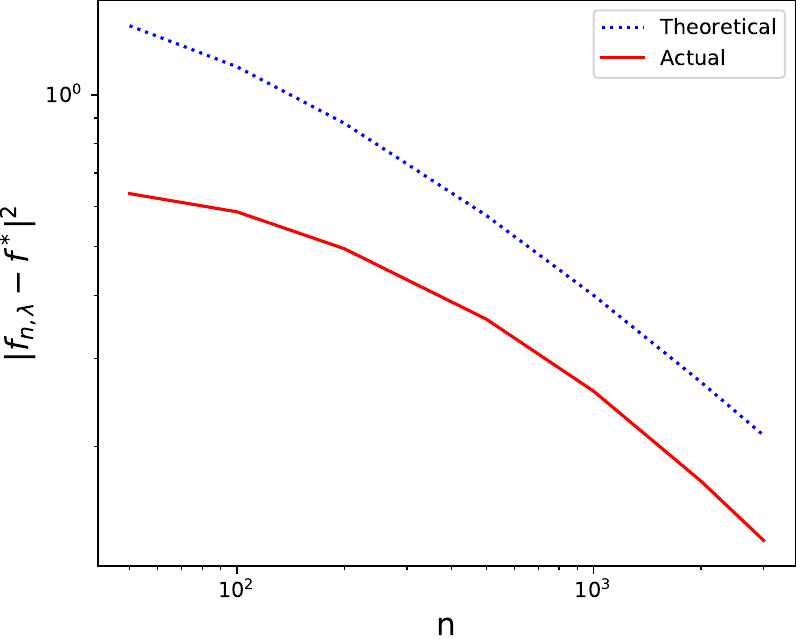} &
\includegraphics[width=0.3\textwidth, height=0.15\textwidth]{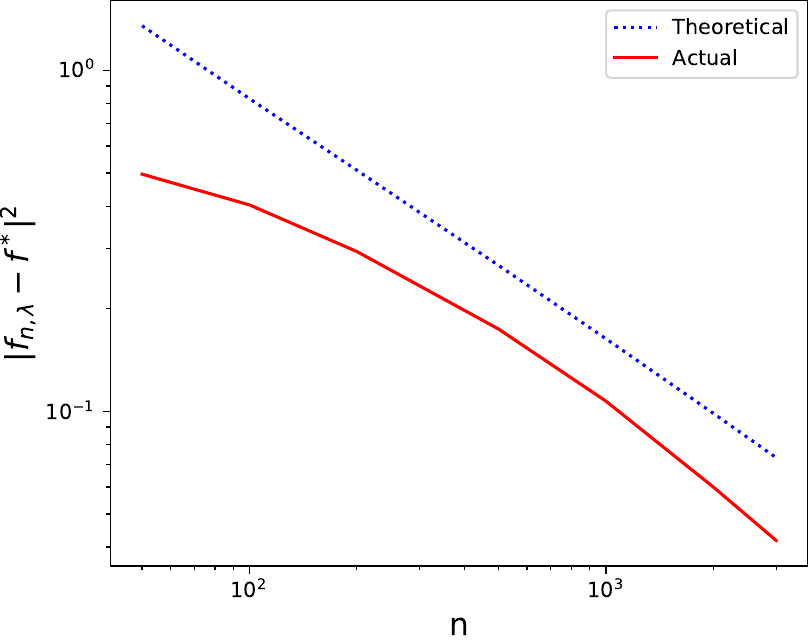} 
\end{tabular}
\caption{Comparison of the actual rates with the theoretical ones of the regularized least-squares estimator in the three different regimes. \emph{Left:} Polynomial case. \emph{Middle:} Geometric case. \emph{Right:} Super-geometric case. \label{fig-exp}}
\end{figure*}

In Figure~\ref{fig-exp}, we compare the theoretical rates of RLS estimator with the actual ones in the different regimes. We use here a similar setup to the one of~\citet[Sec. 4]{bietti2020deep}. In each regime, we consider a specific dot-product kernel. More precisely, for the polynomial, geometric and super-geometric regimes, the kernels considered are respectively, $k(x,y)=\exp(-c\Vert x-y\Vert)$,  $k(x,y)=(2-\langle x,y\rangle)^{-1}$ and $k(x,y)=\exp(-c\Vert x-y\Vert^2)$. To compare the rates, we consider randomly sampled inputs on the unit sphere $S^3$ in 4 dimensions, and generate outputs according to a target function living in the associated RKHS. The regularization parameter of RLS is chosen according the theoretical rules given in the paper.
The actual performance (red curve) is computed on $10,000$ test datapoints. The $x$-axis corresponds to the number of training datapoints. The blue curve corresponds to the theoretical upper rates obtained in our paper. We see that the theoretical rates we obtain match (up to a constant factor) the actual rates of RLS when the number of datapoints is sufficiently large.

%% file: sections/section-appli-aistats.tex
\section{Examples related to deep nets}

We give two other applications of the theoretical results from the previous sections related to multi-layer perceptrons (MLP). Before introducing the applications, Let us first recall the definition of an MLP.

\paragraph{Multi-layer perceptrons.} We refer to here as a multi-layer perceptron a fully-connected deep neural network~\citep{shalevshwartz2014}. Let $\mathcal{X}$ the input space be a subset of $\mathbb{R}^d$, $N$ the number of hidden layers, $\bm{\sigma}:=(\sigma_k)_{k=1}^N$ a sequence of nonlinear activation functions and $\mathbf{m}:=(m_{k})_{k=1}^{N}$ a sequence of integers corresponding to the width of the hidden layers. Let us also introduce the width $m_0$ of the input layer which is just the dimension of the input, and $m_{N+1}$ which is the width of the ouput layer supposed to be $1$ here. Then any function defined by a MLP is parameterized by weight matrices $\mathbf{W}:=(W^k)_{k=1}^{N+1}$ where $W^k\in\mathbb{R}^{m_{k-1}\times m_{k}}$ and can be recovered as follows. Let $x\in\mathcal{X}$, define $\mathcal{N}^0(x):=x$ and for $k\in\{1,\dots,N\}$,  denote
$W^k:=(w_1^k,...,w_{m_{k}}^k)$ where for all  $j\in\{1,\dots,m_k\}$ $w_j^k\in\mathbb{R}^{m_{k-1}}$. Then, for all $k\in\{1,...,N\}$, define the $k^{\text{th}}$ layer as
$$\mathcal{N}^{k}(x):=(\sigma_k(\langle \mathcal{N}^{k-1}(x),w_1^k \rangle), ...,\sigma_k(\langle \mathcal{N}^{k-1}(x),w_{m_k}^k \rangle))$$
Finally the function associated to the MLP with weights $\mathbf{W}$ is defined as $\mathcal{N}(x,\mathbf{W}):=\langle \mathcal{N}^N(x),W^{N+1}\rangle_{\mathbb{R}^{m_N}}$.
We shall denote $\mathcal{F}_{\mathcal{X},\bm{\sigma},\mathbf{m}}$ the function space defined by all functions $\mathcal{N}(\cdot,{\mathbf{W}})$ defined as above on $\mathcal{X}$ for any choice of $\mathbf{W}$. We shall also consider the union space
$$\mathcal{F}_{\mathcal{X},\bm{\sigma}}:=\bigcup_{\mathbf{m}\in\mathbb{N_{*}}^N} \mathcal{F}_{\mathcal{X},\bm{\sigma},\mathbf{m}} .$$

We assume in the following that the input data is on the unit sphere ($\mathcal{X}=S^{d-1}$) which is a common assumption in the literature~\citep{elad:2010}.

\paragraph{Neural Tangent Kernels.} Learning the weights of a network using gradient methods results in a non-convex problem. However, in a specific over-parameterized regime, it may be shown that gradient descent can reach a global minimum while keeping weights very close to random initialization. More precisely, for a network $\mathcal{N}(x,\mathbf{W})$ initialized with $\mathbf{W}_0$, learning in the infinitely width regime is then equivalent to a kernel method with a specific kernel referred to as a neural tangent kernel~\citep{NEURIPS2020_9afe487d}~and defined as 
\begin{align*}
    K_{\text{NTK}}(x,x'):=\lim_{\mathbf{m}\rightarrow +\infty}\langle \nabla_{\mathbf{W}}\mathcal{N}(x,\mathbf{W}_0), \nabla_{\mathbf{W}}\mathcal{N}(x',\mathbf{W}_0)\rangle.
\end{align*}
\citet{bietti2019inductive} show that, when the input space is the unit sphere, the neural tangent kernel associated to an MLP is a dot product kernel. More precisely,  consider the case where for all $i\neq j$, $\sigma_i=\sigma_j$ for simplicity and denote $\sigma$ the nonlinear activation considered. Moreover let $(a^{(1)}_i)_{i\geq 0}$ the coefficients in the decomposition of $\sigma$ in the basis of Hermite polynomials, $(a^{(0)}_i)_{i\geq 0}$ the coefficients in the decomposition of  the first-order derivative $\sigma^{\prime}$ of $\sigma$ (assuming that $\sigma$ is differentiable) in the basis of Hermite polynomials, and define
$f_1(x):=\sum_{i\geq 0} (a^{(1)}_i)^2 x^i$ and $f_2(x):=\sum_{i\geq 0} (a^{(0)}_i)^2 x^i$. Then by defining $K_1^{\text{NTK}}(x) = K_1(x)=x$ and for all $i=2,\dots,N$,
\begin{align*}
K_i(x)&=f_1(K_{i-1}(x))\\
K_i^{\text{NTK}}(x) &= K_{i-1}^{\text{NTK}}(x) f_0(K_{i-1}(x)) + K_i(x)\: ,
\end{align*}
we obtain that $K_{\text{NTK}}(x,x')=K_N^{\text{NTK}}(\langle x,x'\rangle)$. Therefore $K_{\text{NTK}}(x,x')$ is a dot product kernel and our results from Sec.~\ref{sec:eigen-dot-product}-\ref{sec:RLS-dot-product} can be applied. In particular, we can obtain estimates of the eigendecay of the integral operator associated with that kernel in all possible regimes of eigendecay. Such results can be applied for example to control the convergence the idealized gradient descent algorithm for a two-layer MLP. The convergence analysis of~\citet[Th. 4.2]{cao2020understanding} can be used.
The convergence results suggests that the magnitude of the projected residuals is driven by the magnitude of the $p_k$-th eigenvalue of the integral operator associated with the NTK. Therefore, during training by gradient descent, a two-layer MLP with a large enough width learns the target function along the eigenfunctions of the integral operator associated with the NTK corresponding to the larger eigenvalues. Moreover this convergence is faster in the polynomial regime than in the geometric regime; and faster in the geometric regime than in the super-geometric regime. 

\paragraph{Hilbertian Envelope of Smooth Multi-layer Perceptrons} The mapping defined by a multi-layer perceptron can be embedded into an appropriate reproducing kernel Hilbert space with respect to the nonlinear activations involved in the network architecture. Moreover the \emph{kernel induced by an MLP is a dot product kernel} of the form~(\ref{eq:dot-product-kernel})
where $(b_m)_{m\geq 0}$ is completely determined by the non linear activation functions $(\sigma_i)_{i=1}^N$ involved in the network. When the input space is the unit sphere $S^{d-1}$ of $\mathbb{R}^d$ with $d\geq 2$, our results from Sec.~\ref{sec:eigen-dot-product}-\ref{sec:RLS-dot-product} can be again applied now to the specific RHKS related to this network. Note that this RKHS is a different object than the one associated with a neural tangent kernel. 

We show that there exists an RKHS containing the function space $\mathcal{F}_{\mathcal{X},\bm{\sigma}}$ for any smooth activation functions $\bm{\sigma}:=(\sigma_i)_{i=1}^N$. Moreover, for well chosen activation maps, the kernel is a \emph{universal kernel} on $\mathcal{X}$ in the sense of~\citet{sriperumbudur2011universality}. 
 The universality property endows a kernel with interesting theoretical properties. Recall that a continuous positive semi-definite  kernel $k$ on a Hausdorff space $\mathcal{X}$ is said to be $cc$-universal if the RKHS, $H$ induced by $k$ is dense in $\mathcal{C}(\mathcal{X})$ endowed with the topology of compact convergence. Proofs are provided in Appendix~\ref{proof:lem-rkhs}. 

\begin{prop}
\label{prop:RKHS-MLP}
Let $\mathcal{X}$ be any subspace of $\mathbb{R}^d$, $N\geq 1$, $(\sigma_i)_{i=1}^N$ functions which admits a Taylor decomposition on $\mathbb{R}$. Moreover let $(f_i)_{i=1}^N$ be the sequence of functions such that for every $i\in\{1,...,N\}$:
\begin{align}
\label{eq:diff-sig-f}
f_i(x)=\sum_{n\geq 0}\frac{|\sigma_i^{(n)}(0)|}{n!}x^n
\end{align}
Then the RKHS $H_N$ of the kernel, $K_{N}$ defined on $\mathcal{X}\times\mathcal{X}$ by
\begin{align}
\label{eq:kernel-MLP}
K_{N}(x,x'):=f_N\circ...\circ f_1(\langle x,x'\rangle_{\mathbb{R}^d})
\end{align}
contains the function space $\mathcal{F}_{\mathcal{X},\bm{\sigma}}$. 
If we assume in addition that for every $i\in\{1,...,N\}$ and $n\in\mathbb{N}$, $\sigma_i^{(n)}(0)\neq 0$, then the kernel $K_N$ is $cc$-universal.
\end{prop}

The RKHS $H_N$ can be seen as an \emph{Hilbertian envelope} of the function space $\mathcal{F}_{\mathcal{X},\bm{\sigma}}$. Note that the RKHS we define above does \emph{not} require that networks are infinitely wide \textit{i.e.} that all layers of the network are infinitely large, as in some previous works~\citep{daniely2016toward, DBLP:journals/corr/abs-1712-00779}. Indeed, for any number of weights $\mathbf{m}:=(m_i)_{i=1}^N$, the function space $\mathcal{F}_{\mathcal{X},\bm{\sigma},\mathbf{m}}$ lies inside the RKHS we have just defined. This is an important difference with previous works where RKHS constructions were used to approach the function spaces related to deep networks.

There are several consequences to the Proposition above. A direct consequence fact is that $\inf_{f\in H_N} \mathbb{E}[(f(X)-Y)^2]  \leq  \inf_{f\in \mathcal{F}_{\mathcal{X},\bm{\sigma}}} \mathbb{E}[(f(X)-Y)^2]$. In other words, the minimum expected risk in $H_N$ is a straightforward lower bound on the minimum expected risk in $F$. A second consequence is that the kernel $K_N$ associated with $H_N$ defined above is universal. Therefore, Bayes-consistency holds for common loss functions and the Hilbert space embedding of probability distributions is injective under general assumptions~\citep{steinwart2008support}.
%
%
%
%

We would like to underscore that, contrary to a common misconception, many kinds of activations functions other than ReLU activation functions have been used with great success by practitioners in a number of applications; see~\citep{DBLP:conf/emnlp/EgerYG18} for a recent account.

\paragraph{Eigendecay and depth.} 
Thanks to our results, when the input lies on the unit sphere, obtaining the eigendecay of the integral operator associated with a kernel, hence the shape and the volume of the RKHS enveloping the MLP function space, boils down to finding the rate of decay of the coefficients in the Taylor decomposition of the kernel. However, as one overlays layers over layers, iterating compositions of nonlinear functions on top of the dot product, the rate of decay of the coefficients $(b_m)_{m\geq 0}$ changes. For example, as one performs the composition of the exponential function $f_1:=\exp(x)$ with the square function $f_2:=x^2$ (yielding a two-layer network), we get $b^{(2)}_m=2^m/m!$, while if we had considered only the exponential function (yielding a single-layer network) we would have got simply $b^{(1)}_m=1/m!$. Generally, as one performs compositions of functions, each  coefficient $b_m$ increases hence $\lambda_m$ increases, resulting in a growth of the RKHS. 

\paragraph{Convergence Rates and Network Depth.} In the geometric and super-geometric case, we can show that increasing the depth of the network does not affect the statistical rates as soon as the resulting kernel obeys the same regime. Indeed, in the geometric regime (resp. super-geometric), the statistical rates obtained in Sec.~\ref{sec:RLS-dot-product} do not depend on the parameter $0<r<1$ (resp. $\delta>0$). Therefore while the resulting composed kernel still obeys the same regime, the statistical rates remain the same. For example, in the previous example we obtain that $b^{(2)}_m=2^m/m!$, therefore $b^{(2)}_{m+1}/b^{(2)}_{m}=2m$. Moreover we also have that $b^{(1)}_{m+1}/b^{(1)}_{m}=m$, therefore both are still in the same regime and the statistical rates for both networks are the same. The two observations above suggest that, from this viewpoint, increasing the depth of a network can increase the size of the target space, \textit{i.e.}, the set of realizable functions, while the statistical rates appear to remain the same at least in the geometric and super-geometric regime. 

\paragraph{What about ReLUs?}
As shown in~\citep{daniely2016toward}, a ReLU network with $N$ layers can be approximated by the kernel $K_N$ introduced in Prop.~\ref{prop:RKHS-MLP} where for all $i=1,\dots,N$ $f_i(x)=g(x):=\frac{1}{\pi}(\pi - \arccos(x))$. To clearly make the distinction, note that here the function space generated by the ReLU network \emph{is not included} in the RKHS built associated to $K_N$, whereas the function space generated by the smooth MLP defined with instead the activation function $g$ at each layer is included. As $\arccos$ admits a Taylor decomposition and the coefficients admits a polynomial decay.

\citet{bietti2020deep} argue in a very recent work that, in that specific case of deep neural networks with ReLU activation functions, increasing the depth does not change the eigendecay of the associated integral operator. Our theoretical results encompass the polynomial regime of decay of eigenvalues that is characteristic of deep neural networks with ReLU activation functions. We can then obtain estimates of the effective dimension and statistical rates for regularized least-squares. Contrasting various viewpoints on ReLU networks is an interesting venue for future work~\citep{DBLP:conf/iclr/OngieWSS20}.


\paragraph{Conclusion.} We have analyzed the eigenvalue decay of integral operators associated with dot product kernels on Euclidean spheres. Depending on the behavior of the coefficients in the Taylor series expansion of the kernel, we have distinguished three regimes of decay of eigenvalues: polynomial, geometric, and super-geometric. In each eigendecay regime, we have provided tight effective dimension estimates as well as learning rates for regularized least-squares. We have further illustrated our results through examples inspired from recent theoretical analyses of deep neural networks.

\paragraph{Acknowledgments.} The authors would like to thank A. Rudi and L. Rosasco for pointing out relevant literature. We acknowledge support from NSF DMS 2023166, CCF 2019844, CIFAR ``Learning in Machines and Brains'' progam, and ``Chaire d’excellence de l’IDEX Paris Saclay''.

%% file: sections/supp_mat_aistats_place_holder.tex
\onecolumn
\section*{Appendix}



\paragraph{Outline.} 
We provide all the proofs of the mathematical statements in the main text. In Sec.~\ref{sec:RLS}
we provide a general statistical bound on the performance of regularized least-squares under the assumption of a geometric eigenvalue decay. This statistical bound is established in a regime which was not considered in~\citep{caponnetto2007optimal}. We show that the convergence rates we obtain are optimal in the minimax sense from a nonparametric learning viewpoint. Note that we do not consider the polynomial decay regime because it was already studied in~\citep{caponnetto2007optimal} and in~\citep{fischer2017sobolev}. In Sec.~\ref{section:eigendecay} we show how to control the eigendecay of an integral operator associated to a kernel in three different regimes: polynomial, geometric and super-geometric. Combining these two sets of results, in Sec.~\ref{sec:RLS-dot-product-kernel}, we obtain the statistical performance of regularized least-squares with a dot product kernels in each of the three different regimes. In Sec.~\ref{section:MLP-def}, we build a reproducing kernel Hilbert space (RKHS) that contains the function space generated by multi-layer perceptrons. We also provide sufficient conditions for the universal consistency of a dot product kernel. From Cor. 5.29~\citep{steinwart2008support}, Bayes-consistency for common loss functions then immediately follows. We explain how to control the RKHS norm associated with a multi-layer perceptron from our viewpoint. In Sec.~\ref{sec:background}-\ref{sec:useful-thm}-\ref{sec:tech-lem}, we collect useful technical results and basic notions. 

\input{sections/section_RLS_gene}

\section{Eigenvalue Decay of Dot Product Kernels on the Sphere}
\label{section:eigendecay}
\subsection{Basic Notions and Notations}
\textbf{Spherical Harmonics.} Let $P_{m}(d)$ be the space of homogeneous polynomials of degree $m$ in $d$ variables with real coefficients and $\mathcal{H}_{m}(d)$ be the space of harmonics polynomials defined by
\begin{align}
\mathcal{H}_m(d):=\{P\in P_{m}(d)| \Delta P=0\}
\end{align}
where $\Delta \cdot = \sum\limits_{i=1}^d\frac{\partial^2\cdot}{\partial x_i^2} $ is the Laplace operator on $\mathbb{R}^d$. Moreover let us define
$H_{m}(S^{d-1})$ the space of real spherical harmonics of degree $m$ defined as the set of restrictions of harmonic polynomials in $\mathcal{H}_{m}(d)$ to $S^{d-1}$. Let also $L_2^{d\sigma_{d-1}}(S^{d-1})$ be the space of (real) square-integrable functions on the sphere $S^{d-1}$ endowed with its induced Lebesgue measure $d\sigma_{d-1}$ and $|S^{d-1}|$ the surface area of $S^{d-1}$. $L_2^{d\sigma_{d-1}}(S^{d-1})$  endowed with its natural inner product is a separable Hilbert space and the family of spaces $(H_m(S^{d-1}))_{m\geq 0}$, yields a direct sum decomposition \cite{frye2012spherical}
\begin{align}
L_2^{d\sigma_{d-1}}(S^{d-1})=\bigoplus_{m\geq 0}H_m(S^{d-1})
\end{align}
which means that the summands are closed and pairwise orthogonal.
Moreover, each $H_m(S^{d-1})$ has a finite dimension $\alpha_{m,d}$ with $\alpha_{0,d}=1$, $\alpha_{1,d}=d$ and for $m\geq 2$
\begin{align*}
\alpha_{m,d}=\dbinom{d-1+m}{m}-\dbinom{d-1+m-2}{m-2}
\end{align*}
Therefore for all $m\geq 0$, given any orthonormal basis of $H_m(S^{d-1})$, $(Y_{m}^1,...,Y_{m}^{\alpha_{m,d}})$, we can build an Hilbertian basis of $L_2^{d\sigma_{d-1}}(S^{d-1})$ by concatenating these orthonormal basis. Let us denote in the following $(Y_{m}^{l_m})_{m,l_m}$ such an Hilbertian basis of $L_2^{d\sigma_{d-1}}(S^{d-1})$.

Before proving the results obtained on the eigenvalue decay in the different regimes, let us recall some useful notations. Let $d\geq 2$ and $f$ be a real valued function which admits a Taylor decomposition on $[-1,1]$, and we can denote its non negative coefficients $(b_m)_{m\geq 0}$ such that
\begin{align*}
f(x):=\sum_{m\geq 0} b_m x^m.
\end{align*}
Moreover, let $K$ be the kernel associated on $S^{d-1}$ defined by:
\begin{align*}
K(x,x'):=f(\langle x,x'\rangle_{\mathbb{R}^d}) ; .
\end{align*}
we denote the integral operator on $L_2^{d\sigma_{d-1}}(S^{d-1})$ associated 
$$\begin{array}{ccccc}
T_{K} & : &  L_2^{ d\sigma_{d-1}}(S^{d-1}) & \to &  L_2^{ d\sigma_{d-1}}(S^{d-1})\\
 & & f &\to & \int_{S^{d-1}} K(x,\cdot)f(x)d\sigma_{d-1}(x).
\end{array}$$
Thanks to theorem \ref{thm:eigenval} we have an explicit formula of the eigenvalues of $T_{K}$, the integral operator associated with the kernel $K$ defined on $L_{2}^{d\sigma_{d-1}}(S^{d-1})$. Indeed each spherical harmonics of degree $m$, $Y_{m}\in H_{m}(S^{d-1})$, is an eigenfunction of $T_{K}$ with associated eigenvalue given by the formula
\begin{align}
\label{eq:lamba-supp-mat}
\lambda_{m}=\frac{|S^{d-2}|\Gamma((d-1)/2)}{2^{m+1}}\sum_{s\geq 0}b_{2s+m}\frac{(2s+m)!}{(2s)!}\frac{\Gamma(s+1/2)}{\Gamma(s+m+d/2)}
\end{align}

\subsection{Proof of Proposition \ref{prop:menegato}}
\label{proof:menegato}
\begin{proof}
We consider only the case when $0<\delta \leq 1/2$ as \citep[Theorem 3.3]{azevedo2014sharp} show the result for $\delta> 1/2$.
Let us denote $\theta_{s,m}=b_{2s+m}\frac{(2s+m)!}{(2s)!}\frac{\Gamma(s+1/2)}{\Gamma(s+m+d/2)}$. Then we have
\begin{align*}
\bigg|\frac{\theta_{s+1,m}}{\theta_{s,m}}\bigg|&= \frac{b_{2s+2+m}}{b_{2s+m}}\frac{(2s+m+2)(2s+m+1)}{(2s+2)(2s+2m+d)}.
\end{align*}
Let $0<\delta<1/2$ such that Assumption (\ref{eq:b_n-final-f-2}) hold. There exists a $\gamma$ such that for all $m\geq 1$
\begin{align*}
\frac{|b_{m}|}{|b_{m-1}|}\leq \frac{\gamma}{m^{\delta}} \; .
\end{align*}
from which follows
\begin{align*}
\bigg|\frac{\theta_{s+1,m}}{\theta_{s,m}}\bigg|\leq \gamma^2 \frac{(2s+m+2)^{1-\delta}(2s+m+1)^{1-\delta}}{(2s+2)(2s+2m+d)} \; .
\end{align*}
If $\delta=1/2$, we obtain that
\begin{align*}
\bigg|\frac{\theta_{s+1,m}}{\theta_{s,m}}\bigg|&\leq\frac{\gamma^2}{2(s+1)}
\end{align*}
and for all $m\geq 1$ 
\begin{align*}
\sum_{s\geq 0}\bigg|\theta_{s,m}\bigg|\leq \bigg|\theta_{0,m}\bigg|\sum_{s\geq 0} \frac{\gamma^{2s}}{2^s s!}.
\end{align*}
Thanks to Theorem \ref{thm:eigenval}, we obtain that:
\begin{align*}
\bigg|\lambda_m\bigg| \leq \frac{\sigma_{d-2}\Gamma((d-1)/2)}{2^{m+1}}  \bigg|\theta_{0,m}\bigg|  \sum_{s\geq 0} \frac{\gamma^{2s}}{2^s s!}
\end{align*}
but as $\bigg|\theta_{0,m}\bigg|=\Gamma(1/2)\frac{b_m m!}{\Gamma(m+d/2)}$, we obtain
\begin{align*}
\bigg|\lambda_m\bigg| \leq \left[\sigma_{d-2}\Gamma((d-1)/2)  \Gamma(1/2)\sum_{s\geq 0} \frac{\gamma^{2s}}{2^s s!}\right] \frac{b_m m!}{2^{m+1}\Gamma(m+d/2)}  
\end{align*}
and as we have
\begin{align*}
\frac{m!}{\Gamma(m+d/2)}\in \mathcal{O}\left(\frac{1}{m^{(d-2)/2}}\right)
\end{align*}
we finally obtain that 
\begin{align*}
\lambda_m\in \mathcal{O}\left( \frac{m^{-\frac{(d-2)}{2}} b_m}{2^{m+1}}\right) \; .
\end{align*}

Let us now consider the case where $0<\delta<1/2$. Then we have
\begin{align*}
\bigg|\frac{\theta_{s+1,m}}{\theta_{s,m}}\bigg|&\leq \gamma^2 \frac{(2s+m+2)^{2-2\delta}}{(2s+2)(2s+2m+d)}\\
&\leq \gamma^2 \frac{(2s+m+2)^{2-2\delta}}{(2s+2)(2s+m+2)}\\
&\leq \gamma^2 \frac{(2s+m+2)^{1-2\delta}}{(2s+2)}
\end{align*}
Let $\alpha:=\frac{1}{1-2\delta}$, therefore we have
\begin{align*}
\bigg|\frac{\theta_{s+1,m}}{\theta_{s,m}}\bigg|^{\alpha}&\leq \gamma^{2\alpha} \frac{(2s+m+2)}{(2s+2)^{\alpha}}\\
&\leq \frac{\gamma^{2\alpha}}{(2s+2)^{\alpha-1}}+\frac{\gamma^{2\alpha}m}{(2s+2)^{\alpha}}
\end{align*}
and for all $m,s\geq 0$
\begin{align}
\label{eq-1}
\bigg|\theta_{s+1,m}\bigg|^{\alpha}\leq \bigg|\theta_{0,m}\bigg|^{\alpha}\prod_{i=0}^s\left(1+\frac{m}{2i+2}\right) \frac{\gamma^{2\alpha (s+1)}}{2^{(\alpha-1)(s+1)}((s+1)!)^{\alpha-1}}.
\end{align}
Moreover, as soon as $s+1\geq m\delta/2$, we have
\begin{align*}
1+\frac{m}{2s+2}\leq 1+ \frac{1}{\delta} \; .
\end{align*}
we obtain that for all $m,s\geq 0$
\begin{align}
\label{eq-2}
\prod_{i=0}^s\left(1+\frac{m}{2i+2}\right)\leq
\left(1+\frac{1}{\delta}\right)^{s+1}
\times \prod_{i=0}^{\lfloor \frac{m\delta}{2}\rfloor}
\left(1+\frac{m}{2i+2}\right).
 \end{align}
From Eq.~(\ref{eq-1}) and Eq.~(\ref{eq-2}), we deduce that for all $m\geq 2$, and $s\geq 0$, we have
\begin{align*}
\bigg|\theta_{s+1,m}\bigg|^{\alpha}\leq \bigg|\theta_{0,m}\bigg|^{\alpha} m^{\frac{m\delta}{2}+1}\times \left(\frac{\left(1+\frac{1}{\delta}\right)+\gamma^{2\alpha}}{2^{\alpha-1}}\right)^{s+1}\times \frac{1}{((s+1)!)^{\alpha-1}}.
\end{align*}
Let $C:=\frac{\left(1+\frac{1}{\delta}\right)+\gamma^{2\alpha}}{2^{\alpha-1}}$. We obtain that, for all $m\geq 2$
\begin{align*}
\bigg|\theta_{s,m}\bigg|\leq \bigg|\theta_{0,m}\bigg| m^{\frac{m\delta}{2\alpha}+\frac{1}{\alpha}}\times \frac{(C^{\frac{1}{\alpha}})^s}{(s!)^{\frac{\alpha-1}{\alpha}}}
\end{align*}
and for all $m\geq 2$
\begin{align}
\label{eq-3}
\sum_{s\geq 0} \bigg|\theta_{s,m}\bigg|\leq \bigg|\theta_{0,m}\bigg| m^{\frac{m\delta}{2\alpha}+\frac{1}{\alpha}}\times \sum_{s\geq 0 }\frac{(C^{\frac{1}{\alpha}})^s}{(s!)^{\frac{\alpha-1}{\alpha}}}.
\end{align}
Combining Eq.~(\ref{eq-3}) and Theorem \ref{thm:eigenval} leads to
\begin{align*}
\bigg|\lambda_m\bigg| \leq \frac{\sigma_{d-2}\Gamma((d-1)/2)}{2^{m+1}}  \bigg|\theta_{0,m}\bigg| m^{\frac{m\delta}{2\alpha}+\frac{1}{\alpha}}\times \sum_{s\geq 0 }\frac{(C^{\frac{1}{\alpha}})^s}{(s!)^{\frac{\alpha-1}{\alpha}}}
\end{align*}
but as $\bigg|\theta_{0,m}\bigg|=\Gamma(1/2)\frac{b_m m!}{\Gamma(m+d/2)}$, we obtain that
\begin{align*}
\bigg|\lambda_m\bigg| \leq \sigma_{d-2}\Gamma (1/2)\Gamma((d-1)/2)\times 
\sum_{s\geq 0 } \frac{(C^{\frac{1}{\alpha}})^s}{(s!)^{\frac{\alpha-1}{\alpha}}} \times m^{\frac{m\delta}{2\alpha}+\frac{1}{\alpha}} \frac{b_m m!}{2^{m+1} \Gamma(m+d/2)} 
\end{align*}
and as we have
\begin{align*}
\frac{m!}{\Gamma(m+d/2)}\in \mathcal{O}\left(\frac{1}{m^{(d-2)/2}}\right)
\end{align*}
we finally obtain the desired result:
\begin{align*}
\lambda_m\in \mathcal{O}\left( \frac{m^{\frac{m\delta}{2\alpha}+\frac{1}{\alpha}-\frac{(d-2)}{2}} b_m}{2^{m+1}}\right) \; .
\end{align*}
\end{proof}

\subsection{Proof of Proposition \ref{prop:spectrum_anal_gene}}
\label{proof:spectrum_anal_gene}
\begin{proof}
From Theorem \ref{thm:eigenval}, the $(Y_{k}^{l_k})_{k,l_k}$ is an orthonormal basis of eigenfunctions of $T_{K_N}$ associated with the eigenvalues $(\lambda_{k,l_k})_{k,l_k}$ such that for all $k\geq 0$ and $1\leq l_k\leq \alpha_{k,d}$, $\lambda_{k,l_k}:=\lambda_k\geq 0$ where $\lambda_k$ is given by the formula (\ref{eq:lamba-supp-mat}). Moreover Assumption (\ref{eq:b_n-final-f-2}) guarantees that $b_m>0$ for all $m\geq 0$, and thanks to the formula of (\ref{eq:lamba-supp-mat}), we deduce that $(\lambda_{k,l_k})_{k,l_k}$ are exactly the positive eigenvalues of $T_{K_N}$ with their multiplicities and that $M=+\infty$. Let us now define an indexation of the sequence of eigenvalues of $T_{K_N}$. For that purpose we define this following application.
$$\begin{array}{ccccc}
\phi & : &  \mathbb{N}\times\mathbb{N} & \to &  \mathbb{N}\\
 & & (k,q) &\to & \sum\limits_{i=0}^{k-1} \alpha_{i,d} + q - 1
\end{array}$$
with the convention that $\sum\limits_{i=0}^{-1} \alpha_{i,d}=0$. 
Therefore by defining $E:=\{(k,l_k): k\in \mathbb{N} \text{  and  } l_k\in[|1,\alpha_{k,d}|]\}$ we have first that $\phi(E)=\mathbb{N}$ and $\phi_{|E}$ is injective.
Therefore we can define $(\delta_m)_{m\geq 0}$ the sequence of eigenvalues of $T_{K_N}$ with their multiplicities such that for every $m\geq 0$  there exists a unique $(k,l_k)\in E$ such that
\begin{align}
\label{eq:def-order-gene}
m&:=\phi(k,l_k)=\sum\limits_{i=0}^{k-1} \alpha_{i,d} + l_k - 1\\
\delta_m&:=\lambda_{k,l_k}
\end{align}
Let us define also $ q(\delta,k) = \left\{
              \begin{array}{ll}
                 \frac{k\delta}{2\alpha}+\frac{1}{\alpha}\text{ if  } \frac{1}{2}>\delta>0 \\
                \ 0 \text{ otherwise.}
                \end{array}\right.$\\
Thanks to Proposition \ref{prop:menegato}, there exists a constant $C$ independent of $k$ such that:
\begin{align*}
\lambda_k\leq C \frac{b_k k^{q(\delta,k)}}{2^{k+1}k^{({d-2})/2}}
\end{align*}
therefore we have:
\begin{align}
\label{eq-4}
&\delta_m\leq C\frac{b_k k^{q(\delta,k)}}{2^{k+1}k^{({d-2})/2}}.
\end{align}
But there exist a constant $Q\geq 2$ such that, for all $k\geq 1$, we have
\begin{align}
\label{eq-5}
b_k&\leq \frac{Q}{k^{\delta}}b_{k-1}\\
&\leq \frac{Q^k}{(k!)^{\delta}}b_0.
\end{align}
Plugging Eq.~(\ref{eq-5}) into Eq.~(\ref{eq-4}) leads to
\begin{align*}
\delta_m \leq C b_0 \frac{Q^k k^{q(\delta,k)}}{2^{k+1}k^{({d-2})/2}(k!)^{\delta}}.
\end{align*}
Moreover, by applying the Stirling formula, there exist a constant $L>0$ such that
\begin{align*}
\delta_m\leq C L b_0 \frac{e^{\delta k} Q^k}{2^{k+1}k^{({d-2})/2+\delta/2}}\frac{ k^{q(\delta,k)}}{ k^{\delta k }}.
\end{align*}
We remarks that if $\delta >\frac{1}{2}$,
\begin{align*}
\delta k-q(\delta,k)= \delta k
\end{align*}
otherwise we have that
\begin{align*}
\delta k-q(\delta,k)= \delta k(1-\frac{1}{2\alpha})-\frac{1}{\alpha}
\end{align*}
but as soon as $0<\delta<1/2$, $\alpha>1$, we have that for all $\delta>0$
\begin{align*}
\delta_m\leq C L b_0 \frac{e^{\delta k} Q^k}{2^{k+1}k^{({d-2})/2+\delta/2}}\frac{1}{ k^{\frac{\delta k}{2}-1}}.
\end{align*}
Moreover there exist $Q_2\geq 1\geq Q_1>0$ constants only depending on $d$  such that
\begin{align*}
Q_1  m^{\frac{1}{d-1}} \leq k \leq Q_2 m^{\frac{1}{d-1}}
\end{align*}
therefore we obtain that
\begin{align*}
\delta_m \leq \frac{C L b_0}{2 Q_1^{(d-2)/2}} \frac{(\frac{e^{\delta}Q}{2})^{Q_2 m^{\frac{1}{d-1}}} Q_2 m^{\frac{1}{d-1}}}{ (Q_1m^{\frac{1}{d-1}})^{\frac{Q_1 \delta}{2}  m^{\frac{1}{d-1}}}}.
\end{align*}
Finally by denoting $s:= \frac{4(d-1)}{Q_1}$ which is independent of $\delta$ we obtain that
\begin{align*}
\delta_m\times m^{\frac{\delta}{s}m^{\frac{1}{d-1}}}\leq \frac{C L b_0}{2 Q_1^{(d-2)/2}} \frac{(\frac{e^{\delta}Q}{2})^{Q_2 m^{\frac{1}{d-1}}} Q_2 m^{\frac{1}{d-1}}}{ (Q_1)^{\frac{Q_1 \delta}{2}  m^{\frac{1}{d-1}}}  m^{\frac{\delta}{s}m^{\frac{1}{d-1}}}}
\end{align*}
and as the Right-Hand side of the inequality is bounded we obtain that:
\begin{align*}
\delta_m&\in \mathcal{O}\left(m^{-\frac{\delta}{s} m^{\frac{1}{d-1}}}\right).
\end{align*}
Finally Assumption (\ref{eq:b_n-final-f-2}) guarantees that eventually the sequence $(\delta_{m})_{m\geq 0}$ is exactly the sequence of eigenvalues ordered in the non-increasing order. Indeed we have the following Proposition.
\begin{lemma}
\label{decroissance_gene}
If $(b_m)_{m\geq 0}$ is a non-increasing sequence eventually, then $(\delta_m)_{m\geq 0}$ is a non-increasing sequence eventually. 
\end{lemma}
\begin{proof}
For a fixed $k$ we have that
\begin{align*}
\frac{2s+k+1}{2s+2k+d}<1 \text{  for all  }s\geq 0
\end{align*}
and as $(b_m)_{m\geq 0}$ is a non increasing sequence eventually, then there exists $k_0\geq 0$ such that for all $k\geq k_0$ and $s\geq 0$
\begin{align*}
    b_{2s+k+1}\leq b_{2s+k}
\end{align*}
Therefore for all $k\geq k_0$, we have $\lambda_k\geq \lambda_{k+1}>0$ and the result follows.
\end{proof}
Thanks the above Lemma, the sequence $(\delta_{m})_{m\geq 0}$ coincides eventually with the sequence of eigenvalues of the integral operator $T_{K_N}$ ordered in the non-increasing order and the result follows.
\end{proof}

\subsection{Proof of Proposition \ref{prop:eta-control}}
\label{proof:eta-control}
\begin{proof}
Let us first provides a control of $(\lambda_m)_{m\geq 0}$ defined in Eq. (\ref{eq:lamba-supp-mat}).
\begin{prop} 
\label{prop:control_lambda_k}
If there exist $1>r>0$ and $0<c_2\leq c_1$ constants  such that for all $m\geq 0$
\begin{align}
\label{eq:assump_geometric-2}
c_2 r^{m} \leq b_m\leq c_1 r^m
\end{align}
then by denoting $C_1= \frac{|S^{d-2}|\Gamma((d-1)/2)}{2}   \frac{2^{\frac{d-1}{2}}c_1}{1-r^2} $ and $C_2= |S^{d-2}|\Gamma((d-1)/2)\Gamma(1/2) L \frac{c_2}{2}$ where $L$ is a constant only depending on $d$, we have that
all $m\geq 0$
\begin{align*}
C_2 \left(\frac{r}{4}\right)^m \leq \lambda_{m}\leq C_1 r^m
\end{align*}
\end{prop}
\begin{proof}
Let us denote $\theta_{s,m}=b_{2s+m}\frac{(2s+m)!}{(2s)!}\frac{\Gamma(s+1/2)}{\Gamma(s+m+d/2)}$. We have:
\begin{align*}
\theta_{s,m}&= b_{2s+m} \frac{(2s+m)...(2s+1)}{(s+m+\frac{d-2}{2})...(s+\frac{1}{2})}\\
&=b_{2s+m} \frac{(2s+m)...(2s+1)}{(2s+2m+d-2)...(2s+1)}\times 2^{m+\frac{d-1}{2}}\\
&\leq 2^{\frac{d-1}{2}}c_1 (2r)^m r^{2s} 
\end{align*}
The last inequality comes from the fact that $ \frac{(2s+m)...(2s+1)}{(2s+2m+d-2)...(2s+1)}\leq 1$ and the upper bound given in Eq. (\ref{eq:assump_geometric-2}). Taking the sum, we obtain that
\begin{align*}
\sum_{s\geq 0}\theta_{s,m}\leq (2r)^m \frac{2^{\frac{d-1}{2}}c_1}{1-r^2} 
\end{align*}
and we have
\begin{align*}
\lambda_{m}=\frac{|S^{d-2}|\Gamma((d-1)/2)}{2^{m+1}}\sum_{s\geq 0}\theta_{s,m}\leq r^m \frac{|S^{d-2}|\Gamma((d-1)/2)}{2}   \frac{2^{\frac{d-1}{2}}c_1}{1-r^2}.
\end{align*}
Moreover we have
\begin{align*}
\lambda_{m}=\frac{|S^{d-2}|\Gamma((d-1)/2)}{2^{m+1}}\sum_{s\geq 0}\theta_{s,m}&\geq  \frac{|S^{d-2}|\Gamma((d-1)/2)}{2^{m+1}} \theta_{0,m}\\
&\geq b_m \frac{|S^{d-2}|\Gamma((d-1)/2)}{2^{m+1}} \frac{m!\Gamma(1/2)}{\Gamma(m+d/2)}\\
&\geq   |S^{d-2}|\Gamma((d-1)/2)\Gamma(1/2) \frac{c_2}{2} \left(\frac{r}{2}\right)^m \frac{m!}{\Gamma(m+d/2)}
\end{align*}
The last inequality comes from the lower bound given in Eq. (\ref{eq:assump_geometric-2}). Thanks to the Stirling’s approximation formula we obtain that
\begin{align*}
\Gamma(x)\sim   \sqrt{2\pi} x^{x-1/2}e^{-x}
\end{align*}
which leads to 
\begin{align*}
\frac{m!}{\Gamma(m+d/2)}\sim e^{d/2}\left(1-\frac{d/2}{m+d/2} \right)^m \frac{m^{1/2}}{(m+d/2)^{d-1/2}}.
\end{align*}
Finally we obtain
\begin{align*}
 \frac{m!}{\Gamma(m+d/2)}\sim\frac{1}{m^{d-1/2}}   
\end{align*}
therefore there exists a constant $L>0$ only depending on $d$ such that for all $m\geq 0$ we have
\begin{align*}
\frac{m!}{\Gamma(m+d/2)}\geq L \frac{1}{2^m}
\end{align*}
from which follows our lower bound:
\begin{align*}
\lambda_{m}\geq |S^{d-2}|\Gamma((d-1)/2)\Gamma(1/2) L \frac{c_2}{2} \left(\frac{r}{4}\right)^m 
\end{align*}
\end{proof}
We are now able to prove the result. From Theorem \ref{thm:eigenval}, the $(Y_{k}^{l_k})_{k,l_k}$ is an orthonormal basis of eigenfunctions of $T_{K}$ associated with the eigenvalues $(\lambda_{k,l_k})_{k,l_k}$ such that for all $k\geq 0$ and $1\leq l_k\leq \alpha_{k,d}$, $\lambda_{k,l_k}:=\lambda_k\geq 0$ where $\lambda_k$ is given by the formula (\ref{def:formula_eigen}). Moreover the assumption (\ref{eq:assump_geometric-2}) guarantees that $b_m>0$ for all $m\geq 0$, and thanks to the formula of (\ref{eq:lamba-supp-mat}), we deduce that $(\lambda_{k,l_k})_{k,l_k}$ are exactly the positive eigenvalues of $T_{K}$ with their multiplicities and that $M=+\infty$.
Let us now define an indexation of the sequence of eigenvalues of $T_{K}$. Therefore we can define $(\delta_m)_{m\geq 0}$ the sequence of eigenvalues of $T_{K}$ with their multiplicities such that for every $m\geq 0$  there exists a unique $(k,l_k)\in  E:=\{(k,l_k): k\in \mathbb{N} \text{  and  } l_k\in[|1,\alpha_{k,d}|]\}$ such that
\begin{align*}
m&:=\phi(k,l_k)=\sum\limits_{i=0}^{k-1} \alpha_{i,d} + l_k - 1\\
\delta_m&:=\lambda_{k,l_k}
\end{align*}
Thanks to Proposition \ref{prop:control_lambda_k} we have that for all $k\geq 0$
\begin{align*}
C_2 \left(\frac{r}{4}\right)^k \leq \lambda_{k}\leq C_1 r^k.
\end{align*}
Let $m\geq 0$ and $k,l_k$ defined as in Eq. (\ref{eq:def-order-gene}).
First we have that for all $k\geq 2$
\begin{align*}
\alpha_{k,d}=\dbinom{d-1+k}{k}-\dbinom{d-1+k-2}{k-2}
\end{align*}
which gives that
\begin{align*}
\sum\limits_{i=0}^{k} \alpha_{i,d}\sim \frac{2k^{d-1}}{(d-1)!}.
\end{align*}
Then there exist $Q_2>Q_1>0$ constants only depending on $d$  such that
\begin{align*}
Q_1  m^{\frac{1}{d-1}} \leq k \leq Q_2 m^{\frac{1}{d-1}}
\end{align*}
and we have
\begin{align*}
   C_2 \left(\frac{r}{2}\right)^{Q_2 m^{\frac{1}{d-1}}}\leq \delta_m\leq C_1 r^{Q_1 m^{\frac{1}{d-1}}}. 
\end{align*}
Finally by definition of the $(\eta_m)_{m\geq 0}$ and as the upper and lower bounds of $\delta_m$ are decreasing functions, we obtain that for all $m\geq 0$
\begin{align*}
C_2 e^{-Q_2\log(4/r)m^{\frac{1}{d-1}}} \leq \eta_{m}\leq C_1 e^{-Q_1\log(1/r)m^{\frac{1}{d-1}}} \; .
\end{align*}
\end{proof}

\subsection{Proof of Proposition~\ref{prop:poly-control}}
Let $\alpha>1$ such that 
\begin{align*}
 b_m \in\mathcal{O}( m^{-\alpha}) \;,
\end{align*}
Therefore there exists a universal constant $C$ such that for all $m\geq0$,
\begin{align*}
     b_m\leq \frac{C}{m^{\alpha}}
\end{align*}
and we have
\begin{align*}
   \theta_{s,m}&:= b_{2s+m}\frac{(2s+m)!}{(2s)!}\frac{\Gamma(s+1/2)}{\Gamma(s+m+d/2)}\\
   & \leq \frac{C}{(2s+m)^{\alpha}} \frac{(2s+m)!}{(2s)!} \frac{(s+1/2-1)\dots (1/2)}{(s+m+d/2-1)\dots(d/2)} \frac{\Gamma(1/2)}{\Gamma(d/2)}.
\end{align*}
By developing the terms we obtain that
\begin{align*}
 \frac{\theta_{s,m}}{2^m}&\leq \frac{C}{(2s+m)^{\alpha}} \frac{(2s+m)\dots(2s+1)}{([2(s+m+d/2-1)]\dots [2(s+d/2)]} \frac{(s+1/2-1)\dots (1/2)}{(s+m+d/2-1)\dots(d/2)} \frac{\Gamma(1/2)}{\Gamma(d/2)}
\end{align*}
Let $k\geq 1$ such that $(k-1)m\leq s\leq km$. Let $c_1,c_2>0$, as $f_{c_1,c_2}:~x\rightarrow \frac{x+c_1}{x+c_1 +c_2}$ is an increasing function, we deduce that for all $(k-1)m\leq s\leq km$, we have 
\begin{align*}
 \frac{\theta_{s,m}}{2^m}&\leq \frac{C}{((2k-1)m))^{\alpha}} \frac{((2k+1)m)!}{(2km)!}\frac{\Gamma(km + 1/2)}{\Gamma((k+1)m + d/2)}\frac{1}{2^m}
\end{align*}
moreover thanks to the Stirling’s approximation formula we have
\begin{align*}
\Gamma(x)\sim   \sqrt{2\pi} x^{x-1/2}e^{-x}.
\end{align*}
Thus there exits a universal constant $L$ (independent of $k$ and $m$) such that
\begin{align*}
    \frac{\theta_{s,m}}{2^m}\leq \frac{C L}{((2k-1)m))^{\alpha}} \frac{\left(\frac{(2k+1)m}{e}\right)^{(2k+1)m}}{\left(\frac{2km}{e}\right)^{2km}} \frac{ (km + 1/2)^{km}}{((k+1)m + d/2)^{(k+1)m + d/2-1/2}}\frac{e^{-km-1/2}}{ e^{-(k+1)m - d/2}} \frac{\sqrt{2k+1}}{\sqrt{2k}}\frac{1}{2^m}
\end{align*}
Moreover for all $p,q>0$, $(x+q)^{x+p}\in \mathcal{O}(x^{x+p})$ and there exists a universal constants $\tilde{L}$ such that
\begin{align*}
    \frac{\theta_{s,m}}{2^m}\leq \frac{C \tilde{L}L}{((2k-1)m))^{\alpha}} \frac{\left(\frac{(2k+1)m}{e}\right)^{(2k+1)m}}{\left(\frac{2km}{e}\right)^{2km}} \frac{ (km)^{km}}{((k+1)m)^{(k+1)m + d/2-1/2}}\frac{e^{-km-1/2}}{ e^{-(k+1)m - d/2}} \frac{1}{2^m}
\end{align*}
from which follows
\begin{align*}
    \frac{\theta_{s,m}}{2^m}\leq e^{d/2-1/2} C \tilde{L}L
    \left[\frac{1}{2}\times\frac{(2k+1)^{2k+1}k^k}{(2k)^{2k}(k+1)^{k+1}}\right]^m \frac{1}{(2k-1)^{\alpha}(k+1)^{d/2-1/2}} \frac{1}{m^{d/2+\alpha-1/2}}
\end{align*}
Now remark that for all $k\geq 1$, $\frac{(2k+1)^{2k+1}k^k}{(2k)^{2k}(k+1)^{k+1}}\leq 2$. Indeed the function $x\rightarrow \frac{(2x+1)^{2x+1}x^x}{(2x)^{2x}(x+1)^{x+1}}$ is an increasing function which converges towards 2 as $x$ goes to $+\infty$. Therefore we obtain that

\begin{align*}
    \frac{\theta_{s,m}}{2^m}\leq e^{d/2-1/2} C \tilde{L}L
     \frac{1}{(2k-1)^{\alpha}(k+1)^{d/2-1/2}} \frac{1}{m^{d/2+\alpha-1/2}}
\end{align*}
and we deduce that 
\begin{align*}
  \lambda_{m}=\frac{|S^{d-2}|\Gamma((d-1)/2)}{2^{m+1}}\sum_{s\geq 0}\theta_{s,m}&\leq
    \left[e^{d/2-1/2} C \tilde{L}L |S^{d-2}|\Gamma((d-1)/2) \sum_{k\geq 1} \frac{m}{(2k-1)^{\alpha}(k+1)^{d/2-1/2}}\right] \frac{1}{m^{d/2+\alpha-1/2}}\\
    &\leq     \left[e^{d/2-1/2} C \tilde{L}L |S^{d-2}|\Gamma((d-1)/2) \sum_{k\geq 1} \frac{1}{(2k-1)^{\alpha}(k+1)^{d/2-1/2}}\right] \frac{1}{m^{d/2+\alpha-3/2}}.
\end{align*}
As $\alpha>1$ and $d\geq 2$, it follows that the serie converges and we obtain that 
\begin{align*}
  \lambda_{m}\in\mathcal{O}\left(\frac{1}{m^{d/2+\alpha-3/2}}\right)
\end{align*}
Moreover by the exact same reasoning as in the other regimes, by considering the multiplicities of the eigenvalues, we finally obtain that 
\begin{align*}
  \eta_{m}\in\mathcal{O}\left(m^{-\frac{d/2+\alpha-3/2}{d-1}}\right)
\end{align*}

\subsection{Approximation of the RKHS}
Here we recall ~\citep[Lemma 13]{945262}, from which we derive Corollary~\ref{coro-entropy-number}. 
\begin{thm}
Let $K$ be a Mercer kernel with eigenvalues $\eta_m\in\mathcal{O}(m^{-(p+1)})$ for some $p>0$. Then we have
\begin{align*}
\varepsilon_n(T_K)\in\mathcal{O}(\log^{-(p+1)/2}(n))\; .
\end{align*}
Moreover, if the eigenvalues are such that $\eta_m\in\mathcal{O}(e^{-q m^p})$ for some $p,q>0$. Then we have
\begin{align*}
 |\log(\varepsilon_n(T_K))|\in\mathcal{O}(\log^{p/(p+1)}(n))\; .   
\end{align*}
\end{thm}

\section{Statistical Bounds for RLS with Dot-Product Kernels}
\label{sec:RLS-dot-product-kernel}

\subsection{Proof of Theorem \ref{thm:rls-geo}}
\label{proof:rls-geo}
\begin{proof}
Here the main goal is to control the rate of decay of the eigenvalues associated with the integral operator $T_{\rho}$. Indeed to show the upper rate, we show that there exists $\alpha, C_0,\gamma>0$ such that for any $\rho\in\mathcal{G}_{\omega,\beta}$, the eigenvalues, $(\mu_i)_{i\in I}$ where $I$ is at most countable, of the integral operator $T_{\rho}$ associated with $K_N$ fulfill the following upper bound for all $i \in I$:
\begin{align*}
\mu_i \leq C_0 e^{-\gamma i^{1/\alpha}}
\end{align*}
therefore $\mathcal{G}_{\omega,\beta}\subset $
$\mathcal{F}_{H_N,\alpha,\beta}$ and the result will follow from Theorem \ref{thm:upper-rate}.
Moreover to show the minimax-rate, we show that $I=\mathbb{N}$ and that there exists $c>0$ and $q\geq \gamma>0$ such that for any $\rho\in\mathcal{G}_{\omega,h,\beta}$, the eigenvalues, $(\mu_i)_{i\in I}$ of the integral operator $T_{\rho}$ associated with $K_N$ fulfill the following lower bound for all $i \in I$
\begin{align*}
\mu_i \geq c e^{-q i^{1/\alpha}}
\end{align*}
and the result will follow from Theorem \ref{thm:lower-rate}.
By definition of $(b_m)_{m\geq 0}$, we have
\begin{align}
&K_N(x,x')=\sum_{m\geq 0}b_m (\langle x,x'\rangle_{\mathbb{R}^d})^{m}
\end{align}
Let us now derive a tight control of the eigenvalues of $T_{\rho}$, denoted $(\mu_m)_{m\in\mathcal{I}}$, which will conclude the proof. Let us first show that $I=\mathbb{N}$. Indeed as $S^{d-1}$ is compact and $K_N$ continuous,  the Mercer theorem guarantees that $H_N$ and $\mathcal{L}_2^{d\nu}(S^{d-1})$ are isomorphic. Recall now the two key assumptions to obtain a control on the eigenvalues of $T_{\rho}$.
Indeed we have assumed that
\begin{align}
\label{assump:low-upp}
\frac{d\nu}{d\sigma_{d-1}}< \omega\quad\text{ and }\quad \frac{d\nu}{d\sigma_{d-1}}> h
\end{align}
Let us now define 
$$\begin{array}{ccccc}
T_{\omega} & : &  L_2^{ d\sigma_{d-1}}(S^{d-1}) & \to &  L_2^{ d\sigma_{d-1}}(S^{d-1})\\
 & & f &\to & \omega \int_{S^{d-1}} K_N(x,.)f(x)d\sigma_{d-1}(x) - \int_{S^{d-1}} K_N(x,.)f(x)d\nu(x) 
\end{array}$$
and let us denote $E^{k}$, the span of the greatest $k$ eigenvalues strictly positive of $T_{\rho}$ with their multiplicities.\\
Thanks to the Courant–Fischer–Weyl theorem~\cite{HirschF.Francis1999Eofa} we have that
\begin{align*}
\mu_k=\max_{V\subset G_k}\min_{\substack{x\in V\setminus\{0\} \\ \Vert x \Vert = 1}} \langle T_{\rho} x,x\rangle_{L_2^{ d\sigma_{d-1}}(S^{d-1})}
\end{align*}
where $G_k$ is the set of all s.e.v of dimension $k$ in $L_2^{ d\sigma_{d-1}}(S^{d-1})$.
Therefore we have
\begin{align*}
 \eta_k&\geq \frac{1}{\omega}\min_{\substack{x\in E^{k}\setminus\{0\} \\ \Vert x \Vert = 1}} \langle \omega \times T_{K_{N}} x,x\rangle_{L_2^{ d\sigma_{d-1}}(S^{d-1})}\\
&=\frac{1}{\omega}\min_{\substack{x\in E^{k}\setminus\{0\} \\ \Vert x \Vert = 1}} \{\langle T_{\rho} x,x\rangle_{L_2^{ d\sigma_{d-1}}(S^{d-1})}+\langle T_{\omega} x,x\rangle_{L_2^{ d\sigma_{d-1}}(S^{d-1})}\}\\
&\geq \frac{1}{\omega}\min_{\substack{x\in E^{k}\setminus\{0\} \\ \Vert x \Vert = 1}} \langle T_{\rho} x,x\rangle_{L_2^{ d\sigma_{d-1}}(S^{d-1})}+ \frac{1}{\omega}\min_{\substack{x\in E^{k}\setminus\{0\} \\ \Vert x \Vert = 1}} \langle T_{\omega} x,x\rangle_{L_2^{d\sigma_{d-1}}(S^{d-1})}\\
\end{align*}
Then if $T_{\omega}$ is positive we obtain that
\begin{align*}
\eta_k  \geq \frac{1}{\omega} \mu_k  
\end{align*}
Let us now show the positivity of $T_{\omega}$. Thanks to the assumption (\ref{assump:low-upp}), we have that for all $f\in L_2^{d\sigma_{d-1}}(S^{d-1})$ 
\begin{align*}
T_{\omega}(f)=\int_{S^{d-1}}\left[\omega-\frac{d\nu}{d\sigma_{d-1}}\right]K_N(x,.)f(x)d\sigma_{d-1}(x)
\end{align*}
Therefore $v:=\omega-\frac{d\nu}{d\sigma_{d-1}(x)}$ is positive and by denoting $M=\int_{S^{d-1}}v(x)d\sigma_{d-1}(x)$ and by re-scaling the above equality by $\frac{1}{M}$, we have that $V:x\rightarrow \frac{v(x)}{M}$ is a density function and by denoting $d\Gamma=V d\sigma_{d-1}$ we have 
\begin{align*}
 \frac{1}{M}\times T_{\omega}(f)=\int_{S^{d-1}}K_N(x,.)f(x) d\Gamma(x)   
\end{align*}
Therefore $T_{\omega}$ is positive and thanks to Proposition \ref{prop:eta-control}, we have
\begin{align*}
 \mu_m \leq  \omega \eta_m \leq \omega C_1 e^{-Q_1\log(1/r)m^{\frac{1}{d-1}}}  
\end{align*}

Moreover if we assume in addition that the assumption (\ref{assump:low-upp}), we obtain by a similar reasoning that for all $k\geq 0$
\begin{align*}
\eta_k  \leq  \frac{1}{h} \mu_k  \; .
\end{align*}
And we have that for all $m\geq 0$
\begin{align}
\label{eq:geo-control-eigen-rho}
h C_2 e^{-Q_2\log(4/r)m^{\frac{1}{d-1}}}  \leq h \eta_m  \leq \mu_m \leq  \omega \eta_m \leq \omega C_1 e^{-Q_1\log(1/r)m^{\frac{1}{d-1}}} \; .
\end{align}
\textbf{Upper Rate}. Let us now prove Theorem \ref{thm:rls-geo}. Let  $w\geq 1$ and $0<\beta\leq 2$ and
let us denote $\alpha = d-1$ and $\gamma = Q_1\log(\frac{1}{r})$. Thanks to the RHS of Eq. (\ref{eq:geo-control-eigen-rho}) we have that for any $\rho\in\mathcal{G}_{\omega,\beta}$, the eigenvalues, $(\mu_i)_{i\geq 0}$, of the integral operator $T_{\rho}$ associated with $K_N$ fulfill the following upper bound for all $i$:
\begin{align*}
\mu_i \leq \omega C_1 e^{-\gamma i^{1/\alpha}}
\end{align*}
where $C_1 = \frac{|S^{d-2}|\Gamma((d-1)/2)}{2}   \frac{2^{\frac{d-1}{2}}c_1}{1-r^2}$.
Therefore $\mathcal{G}_{\omega,\beta}\subset $
$\mathcal{F}_{H_N,\alpha,\beta}$ and by applying theorem \ref{thm:upper-rate}, we obtain the results for  $C=2*\text{max}(B,128*V\text{max}(5*Q,K))$ and $A= \text{max}(256KQ,16K,1)$ where $Q=\gamma^{-\alpha} \left[1+\omega C_1\int_{1}^{\infty} \frac{(\log(u)+1)^{\alpha-1}}{\omega C_1 u+u^2}du\right]$,  $V=\text{max}(L^2,\sigma^2,2BK+ 2B_{\infty})$, $K=  \sup_{x\in\mathcal{X}}k(x,x)$.
\end{proof}

The following theorem investigates the optimality of our learning rates, from a nonparametric learning viewpoint~\cite{steinwart2008support}.
\begin{thm}
\label{thm:lower-geo}
Let us assume that there exists $0<c_2<c_1$ such that for all $m\geq 0$:
\begin{align*}
  c_2 r^m \leq b_m \leq c_1 r^m
\end{align*}
Then for any $\omega\geq 1 > h >0$ such that $\mathcal{W}_{\omega,h}$ is not empty it holds 
\begin{align*}
 \lim_{\tau\rightarrow 0^{+}}\lim\inf_{\ell\rightarrow\infty}\inf_{f_{\mathbf{z}}}\sup_{\rho\in\mathcal{G}_{\omega,h,\beta}} \rho^{\ell}\left(\mathbf{z}:\Vert f_{\mathbf{z}}-f_{\rho}\Vert_{\rho}^2>\tau b_{\ell}\right)=1
\end{align*}
where $b_{\ell}=\frac{\log(\ell)^{d-1}}{\ell}$. The infimum is taken over all measurable learning methods with respect to $\mathcal{G}_{\omega,h,\beta}$.
\end{thm}
Therefore for such networks, when $\beta>1$ the learning rates of the regularized least-squares estimator stated in theorem  \ref{thm:rls-geo} coincide with the minimax lower
bounds from theorem \ref{thm:lower-geo} and therefore are optimal in the minimax sense.

\begin{proof}
Let $0<h<1\leq \omega$ and let us denote $q=Q_2\log(\frac{4}{r})$. Thanks to the left-hand-side of Eq. (\ref{eq:geo-control-eigen-rho}) we have that for any $\rho\in\mathcal{G}_{\omega,h,\beta}$, the eigenvalues, $(\mu_i)_{i\geq 0}$ of the integral operator $T_{\rho}$ associated with $K_N$ fulfill the following lower bound for all $i$
\begin{align*}
\mu_i \geq hC_2 e^{-q i^{1/\alpha}}
\end{align*}
where $C_2=|S^{d-2}|\Gamma((d-1)/2)\Gamma(1/2) L \frac{c_2}{2}$. Therefore  $\mathcal{G}_{\omega,h,\beta}\subset$ 
$\mathcal{F}_{H_N,\alpha,q,\beta}$ and we can apply Theorem \ref{thm:lower-rate}.
\end{proof}


\subsection{Super-geometric Case: Learning rates}
\label{proof:rls-super-geo}
Another regime of interest is the one where the composition of nonlinear functions involved in the network admits a Taylor decomposition with coefficients decreasing super-geometrically. As in the geometric case, the goal is to obtain a control the of the eigenvalues associated with the integral operator $T_{K_N}$. See Appendix~\ref{proof:rls-super-geo}.
\begin{thm}
Assume that there exists $1 > \delta>0$ such that
\begin{align}
\label{eq:b_n-final}
\left|\frac{b_{m}}{b_{m-1}}\right|\in O(m^{-\delta}) \; .
\end{align}
Let also $0<\beta\leq 2$ and $\omega>0$.
Then there exist a constant $C$ independent of $\beta$ such that
for any $\rho\in \mathcal{G}_{\omega,\beta}$ and $\tau\geq 1$ we have:
\begin{itemize}
\item  If $\beta> 1$, then there exists $\ell_{\tau}>0$ such that for all $\ell\geq \ell_{\tau}$ and $\lambda_{\ell}=\frac{1}{\ell^{1/\beta}}$, with a $\rho^{\ell}$-probability $\geq 1-e^{-4\tau}$ it holds
    \begin{align*}
    \Vert f_{H_N,\mathbf{z},\lambda_{\ell}}-f_{\rho}\Vert_{\rho}^2 &\leq 3C\tau^2\frac{\log(\ell)^{d-1}}{[\log(\log(\ell))]^{d-1}\ell}
    \end{align*}
    \item  If $\beta = 1$, then there exists $\ell_{\tau}>0$ such that for all $\ell\geq \ell_{\tau}$, $\lambda_{\ell}=\frac{\log(\ell)^{\mu}}{[\log(\log(\ell))]^{\mu}\ell}$ and $\mu\geq d-1$, with a $\rho^{\ell}$-probability $\geq 1-e^{-4\tau}$ it holds
    \begin{align*}
    \Vert f_{H_N,\mathbf{z},\lambda_{\ell}}-f_{\rho}\Vert_{\rho}^2 &\leq 3C\tau^2\frac{\log(\ell)^{\mu}}{[\log(\log(\ell))]^{\mu}\ell}
    \end{align*}
    \item If $\beta < 1$, then there exists $\ell_{\tau}>0$ such that for all $\ell\geq \ell_{\tau}$ and $\lambda_{\ell}=\frac{\log(\ell)^{\frac{d-1}{\beta}}}{[\log(\log(\ell))]^{\frac{d-1}{\beta}}\ell}$ with a $\rho^{\ell}$-probability $\geq 1-e^{-4\tau}$ it holds
    \begin{align*}
    \Vert f_{H_N,\mathbf{z},\lambda_{\ell}}-f_{\rho}\Vert_{\rho}^2 &\leq 3C\tau^2\frac{\log(\ell)^{d-1}}{[\log(\log(\ell))]^{d-1}\ell^{\beta}}
    \end{align*}
\end{itemize}
\end{thm}
As expected, in that regime, we obtain faster rates compared to the geometric case as the eigenvalues decrease faster.

\begin{proof}
Our first goal is to obtain a control of $\text{df}(\lambda)$ associated with $T_{\rho}$ in the super-geometric regime. 
\begin{prop}
\label{prop:df-super-exp}
Let $\omega>0$ and $\nu\in\mathcal{W}_{\omega}$. Let also $(\mu_i)_{i\in I}$, where $I$ is at most countable, be the sequence of eigenvalues with their multiplicities associated with $T_{\rho}$ ranked in the non decreasing order. If there exist $1>\delta>0$ such that
\begin{align*}
\left|\frac{b_{m}}{b_{m-1}}\right|\in O(m^{-\delta})
\end{align*}
Then we have:
\begin{align*}
    \degreef(\lambda)\in\mathcal{O}\left(\frac{\log(\lambda^{-1})^{d-1}}{\left(\log(\log(\lambda^{-1}))\right)^{d-1}} \right)
\end{align*}
\end{prop}

\begin{proof}
From Proposition \ref{prop:spectrum_anal_gene}, there exists $0<G\leq 1$ such that
\begin{align*}
\eta_m&\in O\left(m^{-\frac{\delta}{s} m^{\frac{1}{d-1}}}\right) \; ,
\end{align*}
where $s:=\frac{4(d-1)}{G}$. 
Moreover as $S^{d-1}$ is compact and $K_N$ continuous, the Mercer theorem guarantees that $H_N$ and $\mathcal{L}_2^{d\nu}(S^{d-1})$ are isomorphic, and $I=\mathbb{N}$. Therefore, as in the proof of Proposition \ref{thm:rls-geo} the Courant–Fischer–Weyl theorem~\cite{HirschF.Francis1999Eofa} allows us to obtain that for all $m\geq 0$:
\begin{align*}
\mu_m \leq \omega \eta_m \in  O\left(m^{-\frac{\delta}{s} m^{\frac{1}{d-1}}}\right)
\end{align*}

Let $\lambda>0$ and let us now compute $\text{df}(\lambda)$. There exist $\gamma>0$ such that
\begin{align*}
\text{df}(\lambda)=\sum_{i\in I}\frac{\mu_i}{\mu_i+\lambda}\leq \sum_{m\geq 1}\frac{\gamma}{\gamma+\lambda e^{\frac{\delta}{s}\log(m)  m^{\frac{1}{d-1}}}}
\end{align*}
Moreover as the application $x\in[1,+\infty[\rightarrow \frac{\gamma}{\gamma+\lambda e^{\frac{\delta}{s} \log(x) x^{\frac{1}{d-1}}}}$ is positive and decreasing, therefore we have
\begin{align*}
\textsl{df}(\lambda)\leq \int_{1}^{+\infty}\frac{\gamma}{\gamma+\lambda e^{\frac{\delta}{s} \log(x) x^{\frac{1}{d-1}}}} dx
\end{align*}
Let us denote $\alpha= \left(\frac{s}{\delta}\right)^{d-1}$, $g:x\in]1,+\infty[\rightarrow \log(x)^{d-1} x$ and let us consider the following substitution
\begin{align*}
v=\lambda e^{\frac{\delta}{s} \log(x) x^{\frac{1}{d-1}}}
\end{align*}
Therefore we have
\begin{align*}
\int_{1}^{+\infty}\frac{\gamma}{\gamma+\lambda e^{\frac{\delta}{s} \log(x) x^{\frac{1}{d-1}}}} dx = \int_{\lambda}^{+\infty}\frac{\gamma\left[\alpha \frac{d-1}{v}\log(v\lambda^{-1}) \right](g^{-1})'(\alpha \log(v\lambda^{-1})^{d-1})
}{\gamma+v} dv
\end{align*}
Moreover we have for any $\lambda\leq 1$ and $v\geq e^{\alpha^{-\frac{1}{d-1}}}$
\begin{align*}
g\left(\alpha^{\frac{1}{d}}\log(v\lambda^{-1})^{\frac{d-1}{d}}\right)&=\log(\alpha^{\frac{1}{d}}\log(v\lambda^{-1})^{\frac{d-1}{d}})^{d-1} \alpha^{\frac{1}{d}}\log(v\lambda^{-1})^{\frac{d-1}{d}}\\
&\leq \alpha \log(v\lambda^{-1})^{d-1}
\end{align*}
Indeed the latter inequality comes from the fact that $0\leq \log(x)\leq x$ for $x\geq 1$. Therefore by monotonicity of $g$ we obtain that
\begin{align*}
g^{-1}(\alpha \log(v\lambda^{-1})^{d-1})\geq \alpha^{\frac{1}{d}}\log(v\lambda^{-1})^{\frac{d-1}{d}}
\end{align*}
Moreover as $g'(x)=\log(x)^{d-1}+ (d-1)\log(x)^{d-2}$ and it is an increasing function, we obtain that
\begin{align*}
g'\circ g^{-1}(\alpha \log(v\lambda^{-1})^{d-1})&\geq g'(\alpha^{\frac{1}{d}}\log(v\lambda^{-1})^{\frac{d-1}{d}})\\
&\geq \log(\alpha^{\frac{1}{d}}\log(v\lambda^{-1})^{\frac{d-1}{d}})^{d-1}
\end{align*}
where the last inequality comes from
\begin{align*}
\alpha^{\frac{1}{d}}\log(v\lambda^{-1})^{\frac{d-1}{d}}\geq 1 \; .
\end{align*}
Finally we obtain that
\begin{align*}
(g^{-1})^{'}\left(\alpha \log(v\lambda^{-1})^{d-1} \right)\leq \frac{1}{\log(\alpha^{\frac{1}{d}}\log(v\lambda^{-1})^{\frac{d-1}{d}})^{d-1}} \; .
\end{align*}
Therefore we have
\begin{align*}
\text{df}(\lambda)&\leq  \int_{\lambda}^{+\infty}\frac{\gamma\left[\alpha \frac{d-1}{v}\log(v\lambda^{-1}) \right](g^{-1})'(\alpha \log(v\lambda^{-1})^{d-1})
}{\gamma+v} dv   \\
& \leq \int_{\lambda}^{e^{\alpha^{\frac{-1}{d-1}}}}\left[\alpha \frac{d-1}{v}\log(v\lambda^{-1}) \right](g^{-1})'(\alpha \log(v\lambda^{-1})^{d-1})dv  + \int_{e^{\alpha^{\frac{-1}{d-1}}}}^{+\infty}\frac{\gamma\left[\alpha \frac{d-1}{v}\log(v\lambda^{-1}) \right]
}{\left(\gamma+v\right)\left[\log(\alpha^{\frac{1}{d}}\log(v\lambda^{-1})^{\frac{d-1}{d}})^{d-1}  \right]} dv \\
&\leq g^{-1}\left(\alpha \log(e^{\frac{1}{d-1}}\lambda^{-1})^{d-1}\right) + \int_{e^{\alpha^{\frac{-1}{d-1}}}}^{+\infty}\frac{\gamma\left[\alpha \frac{d-1}{v}\log(v\lambda^{-1}) \right]
}{\left(\gamma+v\right)\left[\log(\alpha^{\frac{1}{d}}\log(v\lambda^{-1})^{\frac{d-1}{d}})^{d-1}  \right]} dv
\end{align*}
But by considering $\lambda\leq e^{-1}$, we obtain that
\begin{align*}
\log(v\lambda^{-1})^{d-1}\leq (\log(v)+1)^{d-1}\log(\lambda^{-1})^{d-1}
\end{align*}
Moreover we have for all $v\geq e^{\alpha^{\frac{-1}{d-1}}}\geq 1$
\begin{align*}
\log(\alpha^{\frac{1}{d}}\log(v\lambda^{-1})^{\frac{d-1}{d}})^{d-1} &= \left[\frac{1}{d}\log(\alpha)+\frac{d-1}{d}\log(\log(v\lambda^{-1}))  \right]^{d-1}\\
&\geq  \left[\frac{1}{d}\log(\alpha)+\frac{d-1}{d}\log(\log(\lambda^{-1}))  \right]^{d-1}
\end{align*}
And as soon as 
\begin{align*}
\lambda \leq \text{min}\left(e^{-1}, \exp(\exp(\frac{d}{d-1}\log(\alpha^{\frac{1}{d}})))\right)
\end{align*}
We have that
\begin{align*}
  \log(\alpha^{\frac{1}{d}}\log(v\lambda^{-1})^{\frac{d-1}{d}})^{d-1}\geq 
 \left(\frac{d-1}{d}\right)^{d-1} \left(\log(\log(\lambda^{-1}))\right)^{d-1}
\end{align*}
Finally we obtain that
\begin{align*}
\int_{e^{\alpha^{\frac{-1}{d-1}}}}^{+\infty}\frac{\gamma\left[\alpha \frac{d-1}{v}\log(v\lambda^{-1}) \right]
}{\left(\gamma+v\right)\left[\log(\alpha^{\frac{1}{d}}\log(v\lambda^{-1})^{\frac{d-1}{d}})^{d-1}  \right]} dv \leq \frac{\log(\lambda^{-1})^{d-1}}{\left(\log(\log(\lambda^{-1}))\right)^{d-1}}\frac{\gamma\alpha(d-1)d^{d-1}}{(d-1)^{d-1}}\int_{e^{\alpha^{\frac{-1}{d-1}}}}^{+\infty}
\frac{(\log(v)+1)^{d-1}}{(\gamma v+v^2)}dv
\end{align*}
Moreover we have
\begin{align*}
g\left(\frac{\log(\lambda^{-1})^{d-1}}{\left(\log(\log(\lambda^{-1}))\right)^{d-1}} \right)=\frac{\log(\lambda^{-1})^{d-1}}{\left(\log(\log(\lambda^{-1}))\right)^{d-1}} \log\left(\frac{\log(\lambda^{-1})^{d-1}}{\left(\log(\log(\lambda^{-1}))\right)^{d-1}} \right)^{d-1}
\end{align*}
And we have:
\begin{align*}
\log\left(\frac{\log(\lambda^{-1})^{d-1}}{\left(\log(\log(\lambda^{-1}))\right)^{d-1}} \right)^{d-1}=\left[(d-1)\log(\log(\lambda^{-1}))-(d-1)\log(\log(\log(\lambda^{-1}))) \right]^{d-1}
\end{align*}
But as $y:\rightarrow \frac{d-2}{d-1}y-\log(y)$ is positive on $\mathbb{R}_{+}$, we obtain that:
\begin{align*}
  \log\left(\frac{\log(\lambda^{-1})^{d-1}}{\left(\log(\log(\lambda^{-1}))\right)^{d-1}} \right)^{d-1}\geq  \left(\log(\log(\lambda^{-1}))\right)^{d-1} 
\end{align*}
And finally we have that:
\begin{align}
\label{eq:g^-1}
\frac{\log(\lambda^{-1})^{d-1}}{\left(\log(\log(\lambda^{-1}))\right)^{d-1}}\geq g^{-1}(\log(\lambda^{-1})^{d-1})
\end{align}
Therefore we obtain that:
\begin{align*}
\text{df}({\lambda})&\leq g^{-1}\left(\alpha \log(e^{\frac{1}{d-1}}\lambda^{-1})^{d-1}\right) + \frac{\log(\lambda^{-1})^{d-1}}{\left(\log(\log(\lambda^{-1}))\right)^{d-1}}\frac{\gamma\alpha(d-1)d^{d-1}}{(d-1)^{d-1}}\int_{e^{\alpha^{\frac{-1}{d-1}}}}^{+\infty}
\frac{(\log(v)+1)^{d-1}}{(\gamma v+v^2)}dv\\
\end{align*}
And thanks to the Eq. (\ref{eq:g^-1}), we obtain that:
\begin{align*}
    \text{df}(\lambda)\in\mathcal{O}\left(\frac{\log(\lambda^{-1})^{d-1}}{\left(\log(\log(\lambda^{-1}))\right)^{d-1}} \right)
\end{align*}
\end{proof}
Let us now prove the theorem. Thanks to Proposition \ref{prop:df-super-exp}, we have an explicit control of $\text{df}(\lambda)$, which allows us to derive a proof analogous to the one of Theorem \ref{thm:upper-rate}. Indeed there exists a constant $Q>0$ such that for any $\lambda>0$ we have:
\begin{align*}
\text{df}(\lambda)\leq Q \frac{\log(\lambda^{-1})^{d-1}}{\lambda \left(\log(\log(\lambda^{-1}))\right)^{d-1}}
\end{align*}

The only changes required are the ones related to the new formulation of the the $\text{df}(\lambda)$, that is to say,  $N_{\lambda,\tau}$ and the RHS term of the inequality in Eq.\ref{eq:main-upper-inequality}. Indeed in that case by the exact same arguments of the proof of Lemma \ref{lem:control-const}, we have for $\tau\geq 1$ and $\lambda>0$:
\begin{align*}
N_{\lambda,\tau}&\leq N\tau^2
\frac{\log(\lambda^{-1})^{d-1}}{\lambda \left(\log(\log(\lambda^{-1}))\right)^{d-1}}\\
\text{where\quad} N&=\text{max}(256KQ,16K,1)
\end{align*}
Moreover if $\ell\geq N_{\lambda,\tau}$, a similar proof as in Theorem \ref{thm:upper-rate} gives
\begin{align*}
\Vert f_{H_N,\mathbf{z},\lambda}-f_{\rho,\lambda}\Vert_{\rho}^2&\leq
C_1\left[\frac{\tau^2}{\ell \lambda^{\text{max}(0,1-\beta)}}\left(\frac{\log(\lambda^{-1})^{d-1}}{\left(\log(\log(\lambda^{-1}))\right)^{d-1}}+\frac{1}{\ell \lambda}\right)\right]\\
\text{where\quad } C_1&=128*V\text{max}(5*Q,K)
\end{align*}
with a $\rho^{\ell}$-probability $\geq 1-e^{-4\tau}$. Moreover from Lemma \ref{lem:approx-error} we finally have:
\begin{align*}
\Vert f_{H_N,\mathbf{z},\lambda}-f_{\rho}\Vert_{\rho}^2&\leq
C\left[\lambda^{\beta}+\frac{\tau^2}{\ell \lambda^{\text{max}(0,1-\beta)}}\left(\frac{\log(\lambda^{-1})^{d-1}}{\left(\log(\log(\lambda^{-1}))\right)^{d-1}}+\frac{1}{\ell \lambda}\right)\right]\\
\text{where\quad } C&=2*\text{max}(B,128*V\text{max}(5*Q,K))
\end{align*}
with a $\rho^{\ell}$-probability $\geq 1-e^{-4\tau}$. Now if we assume that $\beta>1$ and $\lambda_{\ell} = \frac{1}{\ell^{1/\beta}}$, we obtain that:
\begin{align*}
\Vert f_{H_N,\mathbf{z},\lambda}-f_{\rho}\Vert_{\rho}^2&\leq
C\left[\frac{1}{\ell}+\frac{\tau^2}{\ell}\left(\frac{1}{\beta^{d-1}}\frac{\log(\ell)^{d-1}}{[\log(\log(\ell))]^{d-1}}+\frac{1}{\ell^{1-\frac{1}{\beta}}}\right)\right]\\
&\leq 3C\tau^2\frac{\log(\ell)^{d-1}}{[\log(\log(\ell))]^{d-1}\ell}
\end{align*}
with a $\rho^{\ell}$-probability $\geq 1-e^{-4\tau}$ provided that 
\begin{align*}
\ell\geq  \left(\frac{ N}{\beta^{\frac{d-1}{\delta}}}\right)^{{\frac{\beta}{\beta-1}}} \tau^{{\frac{2\beta}{\beta-1}}}
\left(\frac{\log(\ell)}{[\log(\log(\ell))]}\right)^{\frac{(d-1)\beta}{\beta-1}}:=\ell_{\tau}
\end{align*}

A similar reasoning leads to the results in the cases where $\beta=1$ and $0<\beta<1$.
\end{proof}

\begin{coro}
\label{coro:exp-rate}
Under the exact same assumptions as before
\begin{equation*}
\lim_{\tau\rightarrow +\infty}\lim \sup_{\ell\rightarrow\infty}\sup_{\rho\in \mathcal{F}_{H,\alpha,\beta}} \rho^{\ell}\left(\mathbf{z}:\Vert f_{\mathbf{z},\lambda_{\ell}}-f_{\rho}\Vert_{\rho}^2>\tau a_{\ell}\right)=0
\end{equation*}
if one of the following conditions hold:
\begin{itemize}
    \item $\beta> 1$, $\lambda_{\ell}=\frac{1}{\ell^{1/\beta}}$ and $a_{\ell} = \frac{\log(\ell)^{d-1}}{[\log(\log(\ell))]^{d-1}\ell}$
    \item $\beta = 1$, $\lambda_{\ell}=\frac{\log(\ell)^{\mu}}{[\log(\log(\ell))]^{\mu}\ell}$ and $a_{\ell} = \frac{\log(\ell)^{\mu}}{[\log(\log(\ell))]^{\mu}\ell}$ for $\mu>d-1>0$
    \item $\beta < 1$, $\lambda_{\ell}=\frac{\log(\ell)^{\frac{d-1}{\beta}}}{[\log(\log(\ell))]^{\frac{d-1}{\beta}}\ell}$ and $a_{\ell} = \frac{\log(\ell)^{d-1}}{[\log(\log(\ell))]^{d-1}\ell^{\beta}}$
\end{itemize}
\end{coro}


\subsection{Polynomial Case: Learning rates}
In~\citep[Theorem 3.3]{fischer2017sobolev} the optimal learning rates of the RLS have been obtained under the assumption that the eigenvalue decay is polynomial. Recall we assume that there exists $g\in L_2^{d\nu}(S^{d-1})$ such that $f_\rho=T_{\nu}^{\beta/2}g$ where $0 < \beta\leq 2$. Note that the rates obtained in~\citep[Theorem 1]{caponnetto2007optimal} correspond to the case where $1\leq \beta\leq 2$ and that in~\citep{fischer2017sobolev}, the authors extend the learning rates of the RLS in this setting when $0<\beta\leq 2$. As we obtain a polynomial control of the eigendecay in the polynomial regime for dot product kernels, we do not need to assume that the eigendecay is polynomial and can obtained the learning rates of the RLS in this setting for an explicit class of distributions.

\section{Illustration related to deep nets}
\label{section:MLP-def}
\paragraph{Multi-layer Perceptrons.} We refer to here as a multi-layer perceptron a fully-connected deep neural network~\citep{shalevshwartz2014}. Let $\mathcal{X}$ the input space be a subset of $\mathbb{R}^d$, $N$ the number of hidden layers, $\bm{\sigma}:=(\sigma_k)_{k=1}^N$ a sequence of nonlinear activation functions and $\mathbf{m}:=(m_{k})_{k=1}^{N}$ a sequence of integers corresponding to the width of the hidden layers. Let us also introduce the width $m_0$ of the input layer which is just the dimension of the input, and $m_{N+1}$ which is the width of the ouput layer supposed to be $1$ here. Then any function defined by a MLP is parameterized by weight matrices $\mathbf{W}:=(W^k)_{k=1}^{N+1}$ where $W^k\in\mathbb{R}^{m_{k-1}\times m_{k}}$ and can be recovered as follows. Let $x\in\mathcal{X}$, define $\mathcal{N}^0(x):=x$ and for $k\in\{1,\dots,N\}$,  denote
$W^k:=(w_1^k,...,w_{m_{k}}^k)$ where for all  $j\in\{1,\dots,m_k\}$ $w_j^k\in\mathbb{R}^{m_{k-1}}$. Then, for all $k\in\{1,...,N\}$, define the $k^{\text{th}}$ layer as
$$\mathcal{N}^{k}(x):=(\sigma_k(\langle \mathcal{N}^{k-1}(x),w_1^k \rangle), ...,\sigma_k(\langle \mathcal{N}^{k-1}(x),w_{m_k}^k \rangle))$$
Finally the function associated to the MLP with weights $\mathbf{W}$ is defined as $\mathcal{N}(x,\mathbf{W}):=\langle \mathcal{N}^N(x),W^{N+1}\rangle_{\mathbb{R}^{m_N}}$.
We shall denote $\mathcal{F}_{\mathcal{X},\bm{\sigma},\mathbf{m}}$ the function space defined by all functions $\mathcal{N}(\cdot,{\mathbf{W}})$ defined as above on $\mathcal{X}$ for any choice of $\mathbf{W}$. We shall also consider the union space
$$\mathcal{F}_{\mathcal{X},\bm{\sigma}}:=\bigcup_{\mathbf{m}\in\mathbb{N_{*}}^N} \mathcal{F}_{\mathcal{X},\bm{\sigma},\mathbf{m}} .$$

\paragraph{Neural tangent kernels.} Consider a normalized nonlinear activation $\sigma$ and a number of layer $N$, the neural tangent kernel $K_{\text{NTK}}$ associated to the fully-connected neural network with $N$ layers and activation function $\sigma$ can be defined as follows. Denote $(a^{(1)}_i)_{i\geq 0}$ the coefficients in the decomposition of $\sigma$ in the basis of Hermite polynomials, $(a^{(0)}_i)_{i\geq 0}$ the coefficients in the decomposition of  the first-order derivative $\sigma^{\prime}$ of $\sigma$ (assuming that $\sigma$ is differentiable) in the basis of Hermite polynomials, and define
$f_1(x):=\sum_{i\geq 0} (a^{(1)}_i)^2 x^i$ and $f_2(x):=\sum_{i\geq 0} (a^{(0)}_i)^2 x^i$. Then by defining $K_1^{\text{NT}}(x) = K_1(x)=x$ and for all $i=2,\dots,N$,
\begin{align*}
K_i(x)&=f_1(K_{i-1}(x))\\
K_i^{\text{NT}}(x) &= K_{i-1}^{\text{NTK}}(x) f_0(K_{i-1}(x)) + K_i(x)\: ,
\end{align*}
we obtain that $K_{\text{NT}}(x,x')=K_N^{\text{NT}}(\langle x,x'\rangle)$. Therefore $K_{\text{NT}}(x,x')$ is a dot product kernel. We can apply Sec.~\ref{sec:eigen-dot-product}-\ref{sec:RLS-dot-product}. 

In particular, we can obtain estimates of the eigendecay of the integral operator associated with a neural tangent kernel (NTK). We consider Theorem 4.2 in Section 4.1 in~\citep{cao2020understanding}. Consider a two-layer MLP with a first layer of width $m$ and with an output prediction scaled by a factor $\theta$ which is a constant to account for the effect of initialization~\cite[Sec. 3.2]{cao2020understanding}. We assume here that the input data $x_1, \ldots, x_\ell$ is uniformly distributed in the sphere and that the output data is bounded. Denote $y = (y_1, \ldots, y_\ell)$ and denote for $t\geq 0$, $\hat{y}^{(t)} = (\hat{y}_1^{(t)}, \ldots, \hat{y}_\ell^{(t)})$ the predictions obtained on the training examples after $t$ updates of the weights in the idealized gradient descent algorithm in~\cite[Sec. 3.2]{cao2020understanding}. Let $(\mu_m)_{m\geq 0}$ be the strictly positive eigenvalues ranked in the non-increasing order of the integral operator $T_{K_{\text{NTK}}}$ and $(e_m)_{m\geq 0}$ the corresponding orthonormal eigenfunctions. Define $v_j = \ell^{-1/2}(e_j(x_1),\ldots, e_j(x_\ell))$ and $V_{{p_k}} = (v_1,\ldots,v_{{p_k}})$ for $j = 1,2,\ldots$. With abuse of notation we shall denote by $p_k$ the sum of the multiplicities of the first $k$ distinct eigenvalues of the integral operator associated with the kernel. Denote by $M$ a uniform upper bound for eigenfunctions of the integral operator. For any $\epsilon,\delta >0$ and $k>1$, if $\ell \geq \tilde{\Omega}( \epsilon^{-2} \max\{ ( \mu_{p_k } - \mu_{p_k + 1}  )^{-2} , M^4 p_k^2 \}   ) $, $m \geq \tilde{\Omega}( \text{poly}(t, \mu_{p_k}^{-1},\epsilon^{-1}) )$, then with probability at least $1 - \delta$, the idealized gradient descent algorithm in~\cite[Sec. 3.2]{cao2020understanding} with $\eta = \tilde{O} ( m^{-1} \theta^{-2} )$ and $\theta = \tilde{O} (\epsilon)$ satisfies 
\begin{align*}
\ell^{-1/2}  \Vert V_{p_k}^T (y - \hat{y}^{(t)})) \Vert_2 \leq 2 \ell^{-1/2}  ( 1  - \mu_{p_k})^t  \Vert V_{p_k}^T y \Vert_2  + \epsilon\; .
\end{align*}

The convergence results suggests that the magnitude of the projected residuals is driven by the magnitude of the $p_k$-th eigenvalue of the integral operator associated with the NT kernel. Therefore, during training by gradient descent, a two-layer MLP with a large enough width learns the target function along the eigenfunctions of the integral operator associated with the NTK corresponding to the larger eigenvalues. Moreover this convergence is faster in the polynomial regime than in the geometric regime; and faster in the geometric regime than in the super-geometric regime. 



\subsection{Proof of Proposition \ref{prop:RKHS-MLP}}
\label{proof:lem-rkhs}

Before proving the result, let us first recall the following definition.
\begin{defn}[\textbf{$cc$-universal kernel}]
\label{def:cc-universal}
A continuous positive semi-definite  kernel k on a Hausdorff space $\mathcal{X}$ is said to be $cc$-universal if the RKHS, H induced by k is dense in $\mathcal{C}(\mathcal{X})$ endowed with the topology of compact convergence.
\end{defn}

\begin{proof}
Let $N\geq 0$ be the number of layers, $(m_1,...,m_N)\in\mathbb{N_{*}}^N$ and let $(\sigma_i)_{i=1}^N$ be a sequence of N functions which admits a Taylor decomposition in 0 on $\mathbb{R}$ such that for every $i\in[|1,N|]$ and $x\in\mathbb{R}$
\begin{align*}
\sigma_i(x)=\sum_{t\geq 0} a_{i,t} x^t
\end{align*}
We can now define the sequence $(f_i)_{i=1}^N$ such that  for every $i\in[|1,N|]$ and $x\in\mathbb{R}$
\begin{align*}
f_i(x):=\sum_{t\geq 0} |a_{i,t}| x^t
\end{align*}
Let us now introduce two sequence of functions  $(\phi_i)_{i=1}^N$ and $(\psi_i)_{i=1}^N$  such that for all $ i\in[|1,N|]$ and $x\in \ell_2$
\begin{align*}
\phi_i(x):=\left(\sqrt{|a_{i,t}|}x_{k_1}...x_{k_t}\right)_{\underset{k_1,...,k_t\in\mathbb{N}}{t\in\mathbb{N}}}\\
\psi_i(x):=\left(\frac{a_{i,t}}{\sqrt{|a_{i,t}|}}x_{k_1}...x_{k_t}\right)_{\underset{k_1,...,k_t\in\mathbb{N}}{t\in\mathbb{N}}}\\
\end{align*}
with the convention that $\frac{0}{0}=0$. Moreover as a countable union of countable sets is countable and $(\sigma_i)_{i=1}^N$ are defined on $\mathbb{R}$, we have that for all $x\in\ell_2$ and $i\in[|1,N|]$,  $\phi_i(x),\psi_i(x)\in\ell_2$. Indeed there exists a bijection $\mu:\mathbb{N}\rightarrow\cup_{t\geq 0} \mathbb{N}^t$, therefore we can denote for all $i\in[|1,N|]$ and $x\in\ell_2$, 
$\phi_i(x)=(\phi_i(x)_{\mu(j)})_{j\in\mathbb{N}}$ and $\psi_i(x)=(\psi_i(x)_{\mu(j)})_{j\in\mathbb{N}}$. We have then
\begin{align*}
\langle \phi_i(x), \phi_i(x')\rangle_{\ell_2}&=\sum_{j\in\mathbb{N}} \phi_i(x)_{\mu(j)}\phi_i(x')_{\mu(j)}\\
&=\sum_{t\geq 0} |a_{i,t}| \sum_{k_1,...,k_t} x_{k_1}...x_{k_t} x_{k_1}'..x_{k_t}'\\
&=\sum_{t\geq 0} |a_{i,t}| \langle x,x'\rangle_{\ell_2}^t\\
&=f_i( \langle x,x'\rangle_{\ell_2})
\end{align*}
Moreover the same calculations lead also to the fact that
\begin{align*}
\langle \psi_i(x), \psi_i(x')\rangle_{\ell_2}=f_i( \langle x,x'\rangle_{\ell_2})
\end{align*}
Therefore $\phi_i$ and $\psi_i$ are a feature maps of the positive semi-definite kernel $k_i:x,x'\in\ell_2\times\ell_2\rightarrow f_i( \langle x,x'\rangle_{\ell_2})$ and we have
\begin{align*}
\langle \phi_i(x), \phi_i(x)\rangle_{\ell_2}=\langle \psi_i(x), \psi_i(x)\rangle_{\ell_2}=f_i(\Vert x\Vert^2_{\ell_2})<\infty
\end{align*}  
Finally let us define the sequence of kernels $(K_i)_{i=1}^N$ defined on $\mathcal{X}\times\mathcal{X}$ such that, for all $x,x'\in\mathcal{X}$
\begin{align*}
K_i(x,x'):=f_i\circ...\circ f_1(\langle x,x'\rangle_{\mathbb{R}^d})
\end{align*}
and let us denote $(H_i)_{i=1}^N$ the sequence of RKHS associated. Moreover in the following, we consider $\mathbb{R}^d$ as a subset of $\ell_2$. One can easily show by induction on $i\in[|1,N|]$ that for all $x,x'\in\mathcal{X}$
\begin{align*}
K_i(x,x')=\langle \phi_i\circ...\phi_1(x), \phi_i\circ...\phi_1(x')\rangle_{\ell_2}
\end{align*}
Let us denote $\mathcal{F}_{\mathcal{X},(\sigma_i)_{i=1}^N}$ the function space defined by a neural network where the activations are the $(\sigma_i)_{i=1}^N$. Moreover let $(W^k)_{k=1}^{N+1}$ be any sequence such that for all $k\in [|1,N+1|]$, $W^k:=(w_j^k)_{j=1}^{m_k}\in\mathcal{M}_{m_{k-1},m_{k}}(\mathbb{R})$  and $\mathcal{N}$ the function in $\mathcal{F}_{\mathcal{X},(\sigma_i)_{i=1}^N}$  associated. Let us now show by induction that at each layer $i\in[|1,N|]$ of the neural network, for $k\in [|1,m_{i}|]$, the coordinate $N^i_{k}$ is a function which lives in $H_i$ such that for all $x\in\mathcal{X}$
\begin{align*}
N_{k}^i(x)=\left\langle \psi_i\left(\sum_{j_{i-1}=1}^{m_{i-1}} W^{i}_{j_{i-1},k}\psi_{i-1}\left(...\psi_2\left(\sum_{j_1=1}^{m_1}W^{2}_{j_1,j_2}\psi_1(w^{1}_{j_1})\right)...\right)\right),\phi_i\circ...\circ\phi_1(x)\right\rangle_{\ell_2}
\end{align*} 
For $i=1$, we have for all $k\in [|1,m_{1}|]$
\begin{align*}
N_{k}^1(x)=\sigma_1(\langle x,w_{k}^{1}\rangle)
\end{align*}
where $w_{k}^{1}\in\mathbb{R}^d\subset\ell_2$. In fact we can show that for every $w\in\mathbb{R}^d$, $x\in\mathcal{X}\rightarrow \sigma_1(\langle x,w\rangle)$ lives in $H_1$. Indeed we have
\begin{align*}
\sigma_1(\langle x,w\rangle)&=\sum_{t\geq 0} a_{1,t} \langle x,w\rangle^t\\
&=\sum_{t\geq 0} a_{1,t} \sum_{k_1,...,k_t} x_{k_1}...x_{k_t}w_{k_1}...w_{k_t}\\
&=\sum_{t\geq 0} \sum_{k_1,...,k_t}\sqrt{a_{1,t}} x_{k_1}...x_{k_t}\frac{a_{1,t}}{\sqrt{a_{1,t}} }w_{k_1}...w_{k_t}\\
&=\langle \phi_1(x),\psi_1(w)\rangle_{\ell_2}
\end{align*}
Thanks to Theorem \ref{thm:RKHS-Inclusion} the following application $x\in\mathcal{X}\rightarrow \sigma_1(\langle x,w\rangle)$ lives in $H_1$ and finally we have for all $k\in[|1,m_1|]$: 
\begin{align*}
N_{k}^1(x)=\langle \phi_1(x),\psi_1(w^1_k)\rangle\in H_1
\end{align*}
\end{proof}
Let us assume that the result hold for $i\in[|1,N-1|]$ and let $k\in[|1,m_{i+1}|]$, then we have by definition of the neural network that for all $x\in\mathcal{X}$
\begin{align*}
N_{k}^{i+1}(x)&=\sigma_{i+1}\left(\langle N^{i}(x),w^{i+1}_k\rangle_{\mathbb{R}^{m_i}}\right)\\
&=\sigma_{i+1}\left(\sum_{j_{i}=1}^{m_{i}} W^{i+1}_{j_i,k} N^{i}_{j_i}(x)\right)
\end{align*}
Then by induction we have that
\begin{align*}
N_{k}^{i+1}(x)&=\sigma_{i+1}\left(\left\langle \sum_{j_{i}=1}^{m_{i}} W^{i+1}_{j_i,k} V^i_{j_i},\phi_i\circ...\circ\phi_1(x) \right\rangle_{\ell_2}\right)\\
\text{where \quad} V^i_{j_i}&:= \psi_i\left(\sum_{j_{i-1}=1}^{m_{i-1}} W^{i}_{j_{i-1},j_i}\psi_{i-1}\left(...\psi_2\left(\sum_{j_1=1}^{m_1}W^{2}_{j_1,j_2}\psi_1(w^{1}_{j_1})\right)...\right)\right)\in \ell_2
\end{align*}
Therefore we obtain
\begin{align*}
N_{k}^{i+1}(x)=\left\langle \psi_{i+1}\left(\sum_{j_{i}=1}^{m_{i}} W^{i+1}_{j_i,k} V^i_{j_i}\right),\phi_{i+1}\circ...\circ\phi_1(x) \right\rangle_{\ell_2}
\end{align*}
And finally, thanks to Theorem \ref{thm:RKHS-Inclusion}, we have that
$x\in\mathcal{X}\rightarrow N_{k}^{i+1}(x)$ lives in $H_{i+1}$ and  we have for all $k\in[|1,m_{i+1}|]$
\begin{align*}
N_{k}^{i+1}(x)=\left\langle \psi_{i+1}\left(\sum_{j_{i}=1}^{m_{i}} W^{i+1}_{j_{i},k}\psi_{i}\left(...\psi_2\left(\sum_{j_1=1}^{m_1}W^{2}_{j_1,j_2}\psi_1(w^{1}_{j_1})\right)...\right)\right),\phi_{i+1}\circ...\circ\phi_1(x)\right\rangle_{\ell_2}
\end{align*} 

Now let us show that $\mathcal{N}$ lives in $H_N$ and that we can bound its RKHS norm. Indeed by definition of $\mathcal{N}$ we have that for all 
$x\in\mathcal{X}$
\begin{align*}
\mathcal{N}(x)&=\langle \mathcal{N}^N(x), W^{N+1}\rangle_{\mathbb{R}^{m_N}}\\
&=\sum_{j_{N}=1}^{m_N} W^{N+1}_{j_{N}} \mathcal{N}^N_{j_N}(x)
\end{align*}
And thanks to what precedes, we have that for all $j_N\in[|1,m_N|]$, $\mathcal{N}^{N}_{j_N}\in H_N$, then as a linear combination of the $(\mathcal{N}^{N}_{j_N})_{j=1}^{m_N}$, we finally have that $\mathcal{N}\in H_N$. Moreover thanks to what precedes we have
\begin{align*}
\mathcal{N}(x)=\left\langle \sum_{j_{N}=1}^{m_N} W^{N+1}_{j_{N}} \psi_{N}\left(\sum_{j_{N-1}=1}^{m_{N-1}} W^{N}_{j_{N-1},j_N}\psi_{N-1}\left(...\psi_2\left(\sum_{j_1=1}^{m_1}W^{2}_{j_1,j_2}\psi_1(w^{1}_{j_1})\right)...\right)\right),\phi_{N}\circ...\circ\phi_1(x)\right\rangle_{\ell_2}
\end{align*}
Therefore thanks to the Theorem \ref{thm:RKHS-Inclusion} we have that
\begin{align*}
\Vert \mathcal{N} \Vert_{H_N}^2\leq \left\Vert \sum_{j_{N}=1}^{m_N} W^{N+1}_{j_{N}} \psi_{N}\left(\sum_{j_{N-1}=1}^{m_{N-1}} W^{N}_{j_{N-1},j_N}\psi_{N-1}\left(...\psi_2\left(\sum_{j_1=1}^{m_1}W^{2}_{j_1,j_2}\psi_1(w^{1}_{j_1})\right)...\right)\right) \right\Vert_{\ell_2}^2
\end{align*}
Moreover, as  the $(\phi_i)_{i=1}^N$ are respectively  feature maps of the kernels $(k_i)_{i=1}^N$,  one can show by  induction that
\begin{align*}
&\left\Vert \sum_{j_{N}=1}^{m_N} W^{N+1}_{j_{N}} \psi_{N}\left(\sum_{j_{N-1}=1}^{m_{N-1}} W^{N}_{j_{N-1},j_N}\psi_{N-1}\left(...\psi_2\left(\sum_{j_1=1}^{m_1}W^{2}_{j_1,j_2}\psi_1(w^{1}_{j_1})\right)...\right)\right) \right\Vert_{\ell_2}^2\\
&=(W^{N+1})^Tf_N\left((W^{N})^{T} f_{N-1}\left((W^{N-1})^{T}...f_2\left((W^{2})^{T}f_1\left((W^1)^{T}W^1\right)W^2\right)...W^{N-1}\right)W^{N}\right)W^{N+1}
\end{align*}
where the $(f_i)_{i=1}^N$ are functions acting coordinate-wise. Thanks to the Schur Inequality~\citep{horn2012matrix}, we obtain that for all $k\in\mathbb{N}$ and any $A\in\mathcal{M}_n(\mathbb{R})$
\begin{align*}
\Vert A^{\circ k}\Vert\leq \Vert A \Vert^k
\end{align*}
And as for all $i \in [|1,N|]$, $(|a_{i,t}|)_{t\geq 0}$ are non negative sequences, we obtain that for all $i\in [|1,N|]$ and $A\in\mathcal{M}_n(\mathbb{R})$
\begin{align*}
\Vert f_i(A) \Vert \leq f_i(\Vert A\Vert)
\end{align*}
Finally by a simple induction we obtain that
\begin{align*}
&(W^{N+1})^Tf_N\left((W^{N})^{T} f_{N-1}\left((W^{N-1})^{T}...f_2\left((W^{2})^{T}f_1\left((W^1)^{T}W^1\right)W^2\right)...W^{N-1}\right)W^{N}\right)W^{N+1}\\
&\leq \Vert W^{N+1} \Vert^2 f_N(\Vert W^{N} \Vert^{2} f_{N-1} (...f_1(\Vert W^1\Vert^2)...))
\end{align*}
And we obtain our upper bound of the RKHS norm of $\mathcal{N}$.

Finally let us assume that for every $i\in[|1,N|]$ and $n\in\mathbb{N}$ we have $\sigma_i^{(n)}(0)\neq 0$. Then by denoting by $(b_m)_{m\geq 0}$ the coefficients of the Taylor decomposition of $f_N\circ ...\circ f_1$, the following Lemma ensures that for every $m\geq 0$, $b_m>0$. 
See proof in Sec.~\ref{lem-proof-explicitcoeffcomp}).
\begin{lemma}
\label{lem:explicitcoeffcomp}
Let $(f_i)_{i=1}^N$ a family of functions that can be expanded in Taylor series in 0 on $\mathbb{R}$ such that for all $k\in[|1,N|]$, $(f_{k}^{(n)}(0))_{n\geq 0}$ are positive. Let us define also $\phi_{1},...,\phi_{N-1}:\mathbb{N}^2\rightarrow \mathbb{R}_{+}$ such that for every $k\in[|1,N-1|]$ and $l,m\geq 0$:
\begin{align*}
\phi_{k}(l,m):=\frac{d^m}{dt^m}|_{t=0}\frac{f_{k}^{l}(t)}{m!}
\end{align*}
Then $g:=f_N\circ ...\circ f_1$ can be expanded in its Taylor series in 0 on $\mathbb{R}$ such that for all $t\in\mathbb{R}$:
\begin{align*}
g(t)=\sum_{l_1,...,l_{N}\geq 0}\frac{f_N^{(l_N)}(0)}{l_{N}!}\times 
\phi_{N-1}(l_N,l_{N-1}) \text{...}\times \phi_{1}(l_2,l_1) t^{l_1}
\end{align*}
Moreover $(g^{(n)}(0))_{n\geq 0}$ is a positive sequence.
\end{lemma}
Therefore thanks to the positivity of the coefficient $(b_m)_{m\geq 0}$ the $cc$-universality of $K_N$ follows directly from Theorem \ref{thm:univ_classic}.

\textbf{Examples.} Consider the case where the input space $\mathcal{X}$ is the unit sphere. Therefore if we assume that for all $i\in\{1,...,N\}$, $f_i(1)=1$ then the norm is preserved across the layers as we obtain that
$K_i(x,x) = \Vert x\Vert_2^2$. In that case, we can consider a larger class of functions as we only need the $(f_i)_{i=1}^N$ to admit a Taylor expansion about 0 on $[-1,1]$. We list some classical functions which satisfy these assumptions~\citep{bietti2017group}.
\begin{table}[ht]
\centering
\begin{tabular}{l|l}
polynomial & $f(t)=\frac{1}{p^p}(p-1+t)^p$\\
Inverse polynomial & $f(t)=\frac{1}{2-t}$\\
arc-cosine & $f(t)=\frac{1}{\pi}(\sin(\text{arcos}(t))+(\pi-\text{arcos}(t))\cos(\text{arcos}(t)))$ \\
Vovk's & $f(t)=\frac{1}{3}(1+t+t^2)$
\end{tabular}
\end{table}

Note that the inverse polynomial kernel was used by \citep{zhang2016l1,zhang2017convexified} to build convex models of fully connected networks and two-layer convolutional neural networks, while the arc-cosine kernel appears in early deep kernel machines \citep{NIPS2009_3628}.

Thanks to Proposition~\ref{prop:RKHS-MLP}, we can have a direct control on the RKHS norm of the network through the spectral norm of its weights.
\begin{coro}
\label{coro:norm-mlp}
For $(W^k)_{k=1}^{N+1}$ a sequence such that for all $k\in [|1,N+1|]$, $W^k\in\mathbb{R}^{m_{k-1}m_{k}}$ and $\mathcal{N}$ the function in $\mathcal{F}_{\mathcal{X},(\sigma_i)_{i=1}^N}$ associated, we have
\begin{align*}
\Vert \mathcal{N} \Vert_{H_N}^2 \leq \Vert W^{N+1} \Vert^2 f_N(\Vert W^{N} \Vert^{2} f_{N-1} (...f_1(\Vert W^1\Vert^2)...))
\end{align*}
where $\Vert.\Vert$ is the spectral norm.
\end{coro}
We can then obtain statistical bounds for deep networks such as as the ones in~\citep{zhang2016l1,zhang2017convexified,
bartlett2017spectrally,bietti2017group}. Indeed the RKHS norm of the network allows one to get for instance generalization bounds involving Rademacher complexities as in~\citep{boucheron2005theory}. To be more precise, if $\mathcal{X}$ is compact and $\mathcal{Y}\subset [-M,M]$, then for any $\delta>0$ with a probability $\rho^n$ not less than $1-\delta$
\begin{align*}
R(\mathcal{N})& \leq \frac{1}{\ell}\sum_{i=1}^{\ell} (\mathcal{N}(x_i)- y_i)^2+\mathcal{O}\left( \frac {\Vert \mathcal{N} \Vert_{H_N}}{\sqrt{n}}+ \sqrt{\frac{ \log(\frac{2}{\delta})}{2n}}\right) \; ,
\end{align*}
There are two major issue with these bounds. First they do not consider the approximation bias of the class of functions considered which is the major issue in high dimension. Moreover these generalization bounds may blow up as the number of parameters increases. We will see in the following that we can obtain obtain learning rates for RLS with the specific RKHSs associated to MLPs we have built. In our framework, the generalization bounds do not suffer from the above issues.

\section{Background}
\label{sec:background}
\begin{defn}(\textbf{Positive semi-definite kernel})
A positive semi-definite kernel on $\mathcal{X}$ is an application $k:\mathcal{X}\times\mathcal{X}\rightarrow\mathbb{R}$ symmetric, such that for all finite families of points in $\mathcal{X}$, the matrix of pairwise kernel evaluations is positive semi-definite.
\end{defn}
\begin{defn}(\textbf{Positive definite kernel} \cite{pinkus2004strictly})
A positive definite kernel on $\mathcal{X}$ is an application $k:\mathcal{X}\times\mathcal{X}\rightarrow\mathbb{R}$ symmetric, such that for all finite families of \textbf{distinct} points in $\mathcal{X}$, the matrix of pairwise kernel evaluations is positive definite.
\end{defn}
\begin{defn}(\textbf{$c$-universal} \cite{sriperumbudur2011universality})
A continuous positive semi-definite kernel k on a compact Hausdorff space $\mathcal{X}$ is called $c$-universal if the RKHS, H induced by k is dense in $\mathcal{C}(\mathcal{X})$ w.r.t. the uniform norm, i.e., for every function $g\in\mathcal{C}(\mathcal{X})$ and all $\epsilon > 0$, there exists an $f \in H$ such that $\Vert f-g\Vert_{u}\leq \epsilon$.
\end{defn}
\begin{defn}(\textbf{$cc$-universal} \cite{sriperumbudur2011universality})
A continuous positive semi-definite  kernel k on a Hausdorff space $\mathcal{X}$ is said to be $cc$-universal if the RKHS, H induced by k is dense in $\mathcal{C}(\mathcal{X})$ endowed with the topology of compact convergence, i.e., for any compact set $Z \subset \mathcal{X}$, for any $g\in \mathcal{C}(Z)$ and all $\epsilon > 0$, there exists an $f\in H|_{Z}$ such that $\Vert f- g \Vert_u\leq \epsilon$.
\end{defn}
\begin{rmq}
It is important to notice that a $c$- or $cc$- universal kernel is necessarily a positive definite kernel.
\end{rmq}
\begin{defn}(\textbf{Dot product kernels on the sphere}  \cite{smola2001regularization})
Let  consider a function $f:[-1,1]\rightarrow\mathbb{R}$ that can be expended into its Taylor series in 0, i.e. for all $t\in[-1,1]$:
\begin{align*}
\label{def:dev}
f(t)=\sum_{n\geq 0}\frac{f^{(n)}(0)}{n!} t^{n}
\end{align*}
Let us now define the general dot product kernel $k$ associated with $f$ on the unit sphere $S^{d-1}\subset\mathbb{R}^d$ by:
\begin{align*}
k(x,y):=f(\langle x,y\rangle_{\mathbb{R}^d}).
\end{align*}
\end{defn}
\begin{rmq} 
If $f^{(n)}(0)\geq 0$  for every $n\geq 0$ then k is a continuous positive semi-definite kernel on $S^{d-1}$.
\end{rmq}

\section{Useful Theorems}
\label{sec:useful-thm}
\begin{thm}{\citep[Theroem 5.3]{fischer2017sobolev}}
\label{thm:stein-oracle-1}
Let $0<\beta\leq 2$ and $H$ a separable RKHS on $\mathcal{X}$ with respect to a bounded and measurable kernel $k$ and $\rho$ a probability measure on $\mathcal{X}\times\mathcal{Y}$ with $\int_{\mathcal{X}\times\mathcal{Y}}y^2d\rho(x,y) < \infty$. We assume that $\Vert f_{\rho}\Vert_{L_{\infty}^{d\rho_{\mathcal{X}}}} < \infty$ and that there exist $g\in L_2^{d\rho_{\mathcal{X}}}(\mathcal{X})$ such that $f_\rho=T_{\rho}^{\beta/2}g$.
Furthermore, we assume that there exist $\sigma>0$ and $L>0$ such that
\begin{equation}
 \int_{\mathcal{Y}} |y-f_{\rho}(x)|^m d\rho(y|x)\leq \frac{1}{2}m!L^{m-2} 
\end{equation}
for $\rho_{X}$-almost all $x\in\mathcal{X}$ and all $m\geq 2$.
Then for $\tau \geq 1$, $\lambda>0$ and $\ell\geq N_{\lambda,\tau}$
 we have with $\rho^{\ell}$-probability $\geq 1-4e^{-\tau}$,
\begin{equation}
\Vert f_{\mathbf{z},\lambda}-f_{\rho,\lambda}\Vert_{\rho}^2\leq 128\frac{\tau^2}{\ell}\left(5\text{df}(\lambda)\sigma_{\lambda}^2+ K \frac{L_{\lambda}}{\ell\lambda}   \right)
\end{equation}
where $K = sup_{x\in\mathcal{X}}k(x,x)$
\begin{align*}
N_{\lambda,\tau}&= \max\left(\frac{256\tau^2 K\text{df}(\lambda)}{\lambda},\frac{16\tau K}{\lambda}, \tau \right)\\
\sigma_{\lambda}&=\max(\sigma,\Vert f_{\rho}- f_{\rho,\lambda}\Vert_{L_\infty^{d\rho_{\mathcal{X}}}})\\
L_{\lambda}&=\max(L,\Vert f_{\rho}-f_{\rho,\lambda}\Vert_{L_\infty^{d\rho_{\mathcal{X}}}})
\end{align*}
\end{thm}

\begin{thm}{\citep[Proposition 6]{caponnetto2007optimal}}
\label{prop:number_of_func}
For every $m > 16$ there exist $N\in\mathbb{N}$ and $\omega^1,...,\omega^N\in \{-1,+1\}^m$
such that
\begin{align*}
 \sum_{i=1}^m(\omega^k_i-\omega^{j}_i)^2 &\geq m \text{\quad} k\neq j=1,...,N\\
 N&\geq e^{m/24}
\end{align*}
\end{thm}

\begin{thm}\citep[§2.1]{saitoh1997integral}
\label{thm:RKHS-Inclusion}
Let $\phi:\mathcal{X}\rightarrow H$ be a feature map to a Hilbert space $H$, and let $K(z,z'):=\langle \phi(z),\phi(z')\rangle_{H}$ a positive semi-definite kernel on $\mathcal{X}$.
Then $\mathcal{H}:=\{f_{\alpha}:z\in\mathcal{X}\rightarrow\langle \alpha,\phi(z)\rangle_{H}\text{,\quad} \alpha\in H\}$ endowed with the following norm:
\begin{align*}
\Vert f_{\alpha}\Vert^2:=\inf_{\alpha'\in\mathcal{H}}\{ \Vert \alpha' \Vert_{H}^2\text{\quad s.t \quad } f_{\alpha'}=f_{\alpha}\}
\end{align*}
is the RKHS associated to K.
\end{thm}

\begin{thm}{\citep[Corollary 10]{steinwart2001influence}}
\label{thm:univ_classic}
Let $0<r\leq +\infty$ and $f: (-r,r)\rightarrow\mathbb{R}$ be a $C^{\infty}$-function that can be expanded
into its Taylor series in 0, i.e.
\begin{align*}
\label{function}
f(x)=\sum_{m=0}^{\infty}a_m x^{m}\text{.}
\end{align*}
Let $X:=\{x\in \mathbb{R}^d: \Vert x\Vert_2 <\sqrt{r}\}$. If we have $a_n>0$ for all $n\geq 0$ then $k(x,y):=f(\langle x,y\rangle)$ defines a $c$-universal kernel on every compact subset of X.
\end{thm}

\begin{thm}\citep[Theorem 2.1]{azevedo2014sharp}
\label{thm:eigenval}
Each spherical harmonics of degree m, $Y_{m}\in H_{m}(S^{d-1})$, is an eigenfunction of $T_{K_N}$ with associated eigenvalue given by the formula:
\begin{equation*}
\label{def:formula_eigen}
\lambda_{m}=\frac{|S^{d-2}|\Gamma((d-1)/2)}{2^{m+1}}\sum_{s\geq 0}b_{2s+m}\frac{(2s+m)!}{(2s)!}\frac{\Gamma(s+1/2)}{\Gamma(s+m+d/2)}
\end{equation*}
\end{thm}

\section{Technical Lemmas}
\label{sec:tech-lem}
\begin{lemma}{\citep[Lemma 5.12]{fischer2017sobolev}}
\label{lem:KL-div}
For $f,f'\in L_2^{d\nu}(\mathcal{X})$ and $\ell \geq  1$ it holds $\rho_f^{\ell} \ll \rho_{f'}^{\ell}$ and $\rho_f^{\ell} \gg \rho_{f'}^{\ell}$ . Furthermore, the KL
divergence fulfills
\begin{align*}
\text{KL}(\rho_f^{\ell},\rho_{f'}^{\ell})=\frac{\ell}{2\hat{\sigma}^2}\Vert f-f'\Vert_{L_2^{d\nu}(\mathcal{X})}^2
\end{align*}
\end{lemma}

\begin{lemma}\citep[Lemma 3.3]{devore2006approximation}
\label{lem:devore-approx} 
 Let $\mathcal{A}$ be a sigma algebra on the space $\Omega$. Let $A_i\in \mathcal{A}$, $i \in \{0,1,...,n\}$ such that $\forall i\neq j$, $A_i\bigcap A_j =\emptyset$. Let $P_i$, $i\in\{0,1,...,n\}$ be $n+1$ probability measures on
$(\Omega, \mathcal{A})$. If
\begin{align*}
p:=\sup_{i=1,...,n} P_i(\Omega \setminus A_i)
\end{align*}
then either $p>\frac{n}{n+1}$ or 
\begin{align*}
\min_{j=1,...,n} \frac{1}{n}\sum_{i\neq j}\text{KL}(P_i,P_j)\geq \psi_n(p)
\end{align*}
where 
\begin{align*}
\psi_n(p):=\log(n) + (1-p) \log \left(\frac{( 1-p )}{p}\right)- p \log \left(\frac{ n-p}{p}\right)
\end{align*}
\end{lemma}
\subsection{Proof of Lemma \ref{lem:explicitcoeffcomp}}
\label{lem-proof-explicitcoeffcomp}
\begin{proof}
Let us show the result by induction on $N$.
For $N=1$ the result is clear as $f_1$  can be expand in its Taylor series in 0 on $\mathbb{R}$ with positive coefficients. Let $N\geq 2$, therefore we have:
\begin{align*}
g(t)=(f_N\circ\text{...}\circ f_2)\circ ( f_{1}(t))
\end{align*}
By induction, we have that for all $t\in\mathbb{R}$:
\begin{align*}
f_N\circ\text{...}\circ f_2(t)=\sum_{l_2,...,l_{N}\geq 0}\frac{f_N^{(l_N)}(0)}{l_{N}!}\times 
\phi_{N-1}(l_N,l_{N-1}) \text{...}\times \phi_{2}(l_3,l_2) t^{l_2}
\end{align*}
Therefore we have that:
\begin{align*}
g(t)&=\sum_{l_2,...,l_{N}\geq 0}\frac{f_N^{(l_N)}(0)}{l_{N}!}\times 
\phi_{N-1}(l_N,l_{N-1}) \text{...}\times \phi_{2}(l_3,l_2) (f_1(t))^{l_2}\\
\end{align*}
Moreover, for all $l_2\geq 0$, $f_{1}^{l_2}$ can be expand in its Taylor series in 0 on $\mathbb{R}$ with non negative coefficients, and we have that for all $l_2\geq 0$ and  $t\in\mathbb{R}$:
\begin{align*}
(f_1(t))^{l_2}=\sum_{l_1\geq 0} \phi_{1}(l_2,l_1) t^{l_1}
\end{align*}
And we obtain that:
\begin{align*}
g(t)=\sum_{l_1,...,l_{N}\geq 0}\frac{f_N^{(l_N)}(0)}{l_{N}!}\times 
\phi_{N-1}(l_N,l_{N-1}) \text{...}\times \phi_{1}(l_2,l_1)t^{l_1}
\end{align*}
Finally we have by unicity of the decomposition that for all $l_1\geq 0$:
\begin{align*}
\frac{g^{(l_1)}(0)}{l_1!}=\sum_{l_2,...,l_{N}\geq 0}\frac{f_N^{(l_N)}(0)}{l_{N}!}\times 
\phi_{N-1}(l_N,l_{N-1}) \text{...}\times \phi_{1}(l_2,l_1)
\end{align*}
Moreover let $k\in[|1,N-1|]$, $l\geq 1$ and let us denote $(a^k_i)_{i\geq 0}$ the coefficients in the Taylor decomposition of $f_k$. Then we have:
\begin{align*}
f_k^{l}(t)&=\sum_{n_1,...,n_l\geq 0}\prod_{i=1}^l \left(a^{k}_{n_i}\right) x^{n_1+...+n_l}\\
&=\sum_{q\geq 0} \left[\sum_{\substack{n_1,...,n_l\geq 0\\ \sum\limits_{i=1}^n n_i=q}} \prod_{i=1}^l \left(a^{k}_{n_i}\right)\right] x^{q}  
\end{align*}
But as $a^k_i>0$ for all $i\geq 0$, we obtain that by unicity of the decompoisition that for all $m\geq 0$:
\begin{align*}
\phi_{k}(l,m)=\frac{d^m}{dt^m}|_{t=0}\frac{f_{k}^{l}(t)}{m!}>0
\end{align*}
and the result follows.
\end{proof}

%% file: sections/section_RLS_gene.tex
\section{Regularized Least-Squares: General Setting}
\label{sec:RLS}


We consider the standard nonparametric learning framework, where the goal is to learn, from independent and identically distributed examples $\mathbf{z} = \{(x_1, y_1), \dots , (x_\ell, y_\ell)\}$ from an unknown distribution $\rho$, a functional dependency $f_\mathbf{z} : \mathcal{X} \rightarrow  \mathcal{Y} $ between input $x\in \mathcal{X}$ and output $y \in \mathcal{Y}$~\cite{gkkw:2002,steinwart2008support}. The joint distribution $\rho(x, y)$, the marginal distribution $\rho_\mathcal{X}$, and the conditional distribution $\rho(.|x)$, are related through $\rho(x, y) =\rho_{\mathcal{X}}(x)\rho(y|x)$. We call the $f_\mathbf{z}$ the learning method or the estimator and the learning algorithm is the procedure that, for any sample size $\ell \in \mathbb{N}$ and training set $\mathbf{z} \in Z^{\ell}$ yields the learned function or estimator $f_\mathbf{z}$. If the output space $\mathcal{Y}\subset \mathbb{R}$, given a function $f : \mathcal{X} \rightarrow \mathcal{Y}$ , the ability of $f$ to describe the distribution $\rho$ is measured by its expected risk
\begin{align}
R(f):=\int_{\mathcal{X}\times \mathcal{Y}}(f(x)-y)^{2}\,d\rho(x,y) \; .
\end{align}
The minimizer over the space of measurable  $\mathcal{Y}$-valued functions on $\mathcal{X}$ is the function
\begin{align}
f_{\rho}(x):=\int_{\mathcal{Y}} y d\rho(y|x) \; .
\end{align}
The final aim of learning theory is to find an algorithm such that $R(f_\mathbf{z})$ is close to $R(f_\rho)$ with high probability.
Let us now introduce the regularized least-squares algorithm.
Consider as hypothesis space a Hilbert space $H$ of functions $f:\mathcal{X}\rightarrow \mathcal{Y}$. For any regularization parameter $\lambda>0$ and training set $\mathbf{z}\in Z^{\ell}$, the Regularized Least-Square (RLS) estimator $f_{H,\mathbf{z},\lambda}$ is the solution of
\begin{align}
\min_{f\in H} \left\{ \frac{1}{\ell}\sum_{i=1}^{\ell}( f(x_i)-y_i)^2+\lambda \Vert f\Vert_{H}^2\right\} \; .
\end{align}

Let us recall basic definitions yet important for the exposition of the convergence rates. The goal is to establish bounds, either in expectation or in probability, on
$\Vert f_{H,\mathbf{z},\lambda}-f_{\rho} \Vert_{\rho}$ where the norm  $\Vert .\Vert_{\rho}$ is the $L_2^{\rho_{\mathcal{X}}}$-norm and to obtain the convergence rate. Let $\mathcal{F}$ a class of Borel probability distributions on $\mathcal{X}\times\mathcal{Y}$ satisfying general assumptions.
\begin{defn}(Upper rate of convergence)
\label{defn:upper-rate}
A sequence $(a_{\ell})_{\ell\geq 1}$ of
positive numbers is called upper rate of convergence in $L_2^{d\rho_{\mathcal{X}}}$ norm over the model $\mathcal{F}$, for the sequence of estimated solutions $(f_{\mathbf{z},\lambda_{\ell}})_{\ell\geq 1}$ using regularization parameters $(\lambda_{\ell})_{\ell\geq 0}$ if
\begin{equation*}
\lim_{\tau\rightarrow +\infty}\lim \sup_{\ell\rightarrow\infty}\sup_{\rho\in \mathcal{F}} \rho^{\ell}\left(\mathbf{z}:\Vert f_{\mathbf{z},\lambda_{\ell}}-f_{\rho}\Vert_{\rho}^2>\tau a_{\ell}\right)=0 \, .
\end{equation*}
\end{defn}
\begin{defn}(Minimax Lower Rate of Convergence)
A sequence $(w_{\ell})_{\ell\geq 1}$ of
positive numbers is called minimax lower rate of convergence in $L_2^{d\rho_{\mathcal{X}}}$ norm over the model $\mathcal{F}$ if
\begin{align*}
 \lim_{\tau\rightarrow 0^{+}}\lim\inf_{\ell\rightarrow\infty}\inf_{f_{\mathbf{z}}}\sup_{\rho\in\mathcal{F}} \rho^{\ell}\left(\mathbf{z}:\Vert f_{\mathbf{z}}-f_{\rho}\Vert_{\rho}^2>\tau w_{\ell}\right)=1
\end{align*}
where the infimum is taken over all measurable learning methods with respect to $\mathcal{F}$.
\end{defn}
Such sequences $(w_{\ell})_{\ell\geq 1}$ are called minimax lower rates. Every sequence $( \hat{w}_{\ell})_{\ell\geq 1}$ decreasing at least with the same rate as $(w_{\ell})_{\ell\geq 1}$ is a lower rate for this set of probability measures and at least the same lower rate holds on any larger set of probability measures. When the rate of the sequence of learned functions coincides with the minimax lower rates, it is then said to be optimal in the minimax sense.

\medbreak
\textbf{Setting.} 
Let $(\mathcal{X},\mathcal{B})$ a measurable space, $\mathcal{Y}=\mathbb{R}$ and $\rho(x,y)=\rho_{\mathcal{X}}(x)\rho(y|x)$ an unknown distribution on $Z:=\mathcal{X}\times \mathcal{Y}$. We assume that $(\mathcal{X},\mathcal{B})$ is $\rho_{\mathcal{X}}$-complete. Let $H$ be separable reproducing kernel Hilbert space on $\mathcal{X}$ with respect to a measurable and bounded kernel $k$. Define the integral operator on $L_2^{d{\rho_{\mathcal{X}}}}(\mathcal{X})$ associated
$$\begin{array}{ccccc}
T_{\rho} & : &  L_2^{d{\rho_{\mathcal{X}}}}(\mathcal{X}) & \to &  L_2^{d{\rho_{\mathcal{X}}}}(\mathcal{X})\\
 & & f &\to & \int_{\mathcal{X}} k(x,.)f(x){d\rho_{\mathcal{X}}}(x)
\end{array}$$
Since $k$ is bounded, $T_{\rho}$ is self-adjoint, positive semi-definite and trace-class~\cite{blanchard2008finite}. Let us denote $(\mu_i)_{i\in I} \in \ell_1(I)$ its eigenvalues ranked in a non decreasing order. We shall assume in the following that $I=\mathbb{N}$. When $I$ is finite, results can be found for example in~\cite{caponnetto2007optimal}. We work here under general assumptions on the set of probability measures $\rho$ on $\mathcal{X}\times\mathcal{Y}$.
\begin{assump}\textbf{[Probability measures on $\mathcal{X}$].}
Let H be an infinite dimensional separable reproducing kernel Hilbert space (RKHS) on $\mathcal{X}$ with respect to a bounded and measurable kernel $k$. Furthermore, let $C_0,\gamma>0$ be some constants and $\alpha>0$ be a parameter. By $\mathcal{P}_{H,C_0,\gamma,\alpha}$ we denote the set of all probability measures $\nu$ on $\mathcal{X}$ with the following.
\begin{itemize}
    \item The measurable space $(\mathcal{X}, \mathcal{B})$ is $\nu$-complete.
    \item  The eigenvalues fulfill the following upper bound $\mu_i \leq C_0 e^{-\gamma i^{1/\alpha}}$ for all $i \in I$.
\end{itemize}
Let us introduce for a constant $c > 0$ and a parameter $q\geq \gamma>0$ the subset $\mathcal{P}_{H,C_0,\gamma,\alpha,c,q}\subset\mathcal{P}_{H,C_0,\gamma,\alpha}$ of probability measures $\mu$ on $\mathcal{X}$ which additionally have the following property.
\begin{itemize}
    \item The eigenvalues fulfill the following lower bound $\mu_i \geq c e^{-q i^{1/\alpha}}$ for all $i \in I$.
\end{itemize}
We shall denote $\mathcal{P}_{H,\alpha}:= \mathcal{P}_{H,C_0,\gamma,\alpha}$ and  $\mathcal{P}_{H,\alpha,q}:=\mathcal{P}_{H,C_0,\gamma,\alpha,c,q}$.\\
\end{assump}

\begin{assump}\textbf{ [Probability measures on $\mathcal{X} \times \mathcal{Y}$].} Let $H$ be a separable RKHS on $\mathcal{X}$ with respect to a bounded and measurable kernel $k$ and $\mathcal{P}$ a set of probability measures on $\mathcal{X}$. Furthermore, let $B, B_{\infty}, L, \sigma> 0$ be some constants and $0< \beta \leq 2$ a parameter. Then we denote by $\mathcal{F}_{H,B, B_{\infty}, L, \sigma,\beta}(\mathcal{P})$ the set of all probability measures $\rho$ on $\mathcal{X}\times\mathcal{Y}$ with the following properties:
\begin{itemize}
    \item $\rho_{\mathcal{X}}\in\mathcal{P}$, $\int_{\mathcal{X}\times\mathcal{Y}}y^2d\rho(x,y) < \infty$, $\Vert f_{\rho}\Vert_{L_{\infty}^{d\rho_{\mathcal{X}}}}^2 \leq  B_{\infty}$
    \item There exists $g\in L_2^{d\rho_{\mathcal{X}}}(\mathcal{X})$ such that $f_\rho=T_{\rho}^{\beta/2}g$ and $\Vert g\Vert_{\rho}^2\leq B$
    \item there exist $\sigma>0$ and $L>0$ such that
 $\int_{\mathcal{Y}} |y-f_{\rho}(x)|^m d\rho(y|x)\leq \frac{1}{2}m!\sigma^2 L^{m-2}$
\end{itemize}
\end{assump}
A sufficient condition  for the last assumption is that $\rho$ is concentrated on $\mathcal{X} \times [-M,M]$ for some constant $M>0$. 
We shall denote $\mathcal{F}_{H,\alpha,\beta}:=\mathcal{F}_{H,B, B_{\infty}, L, \sigma,\beta}(\mathcal{P}_{H,\alpha})$
and $\mathcal{F}_{H,\alpha,q,\beta}:=\mathcal{F}_{H,B, B_{\infty}, L, \sigma,\beta}(\mathcal{P}_{H,\alpha,q})$.

\subsection{Upper rate of convergence}
We state a theorem establishing $L_2^{d\rho_{\mathcal{X}}}$-convergence rates. 

\begin{thm}
\label{thm:upper-rate}
Let $H$ be a separable RKHS on $\mathcal{X}$ with respect to a bounded and measurable kernel $k$, $\alpha>0$ and $0<\beta\leq 2$. Then for any $\rho\in \mathcal{F}_{H,\alpha,\beta}$ and $\tau\geq 1$ we have:
\begin{itemize}
\item  If $\beta> 1$, then for $\lambda_{\ell}=\frac{1}{\ell^{1/\beta}}$ and $\ell\geq \text{max}\left(e^{\beta},\left(\frac{N}{\beta^{\alpha}}\right)^{\frac{\beta}{\beta-1}}\tau^{\frac{2\beta}{\beta-1}} \log(\ell)^{\frac{\alpha\beta}{\beta-1}}\right)$, with a $\rho^{\ell}$-probability $\geq 1-e^{-4\tau}$ it holds
    \begin{align*}
\Vert f_{H,\mathbf{z},\lambda_{\ell}}-f_{\rho}\Vert_{\rho}^2 &\leq 3C\tau^2\frac{\log(\ell)^{\alpha}}{\ell}
    \end{align*}
    \item  If $\beta = 1$, then for $\lambda_{\ell}=\frac{\log(\ell)^{\mu}}{\ell}$, $\mu>\alpha>0$ and  $\ell\geq \text{max}\left(\exp\left((N\tau)^{\frac{1}{\mu-\alpha}}\right),e^{1}\log(\ell)^{\mu}\right)$, with a $\rho^{\ell}$-probability $\geq 1-e^{-4\tau}$ it holds
    \begin{align*}
    \Vert f_{H,\mathbf{z},\lambda_{\ell}}-f_{\rho}\Vert_{\rho}^2 &\leq 3C\tau^2\frac{\log(\ell)^{\mu}}{\ell^{\beta}}
    \end{align*}
    \item  If $\beta < 1$, then for $\lambda_\ell=\frac{\log(\ell)^{\frac{\alpha}{\beta}}}{\ell}$ and $\ell\geq \text{max}\left(\exp\left((N\tau)^{\frac{\beta}{\alpha(1-\beta)}}\right),e^{1}\log(\ell)^{\frac{\alpha}{\beta}}\right)$, with a $\rho^{\ell}$-probability $\geq 1-e^{-4\tau}$ it holds
    \begin{align*}
    \Vert f_{H,\mathbf{z},\lambda_{\ell}}-f_{\rho}\Vert_{\rho}^2 &\leq 3C\tau^2\frac{\log(\ell)^{\alpha}}{\ell^{\beta}}
    \end{align*}
\end{itemize}
where $N$ and $C$ are constants independent of $\alpha$ and $\beta$. 
\end{thm}

\begin{proof}
To show the result let us first define the degrees-of-freedom $\text{df}(\lambda):=\text{Tr}\left((C_{\rho}+\lambda)^{-1}C_{\rho}\right)=\sum_{i\in I}\frac{\mu_i}{\mu_i+\lambda}$.
Moreover for $\lambda > 0$ the minimization problem
\begin{align*}
\inf_{f\in H}\left\{ R(f)+\lambda \Vert f\Vert_H^2 \right\}
\end{align*}
has a unique solution defined as $f_{H,\rho,\lambda}=(C_{\rho}+\lambda)^{-1}T_{\rho} f_{\rho}\in H$.
To obtain an upper rate of convergence we first control the degrees-of-freedom with respect to the regularization parameter.
\begin{prop}
\label{lem:effective-dim-geo}
Let $\alpha >0$. If there exist $C,\gamma>0$ such that $\mu_i \leq C_0 e^{-\gamma i^{1/\alpha}}$ for all $i \in I$. Then  all $0<\lambda\leq e^{-1}$ we have:
\begin{align*}
\textsl{df}(\lambda)&\leq Q \log(\lambda^{-1})^{\alpha}\\
\text{where \quad} Q&=\gamma^{-\alpha} \left[1+C_0\int_{1}^{\infty} \frac{(\log(u)+1)^{\alpha-1}}{C_0 u+u^2}du\right]
\end{align*}
\end{prop}

\begin{proof}
By definition of $\textsl{df}(\lambda)$ we have
\begin{align*}
\textsl{df}(\lambda)\leq \sum_{m\geq 1}\frac{C_0}{C_0+\lambda e^{\gamma m^{\frac{1}{\alpha}}}} \; .
\end{align*}
Moreover as the $x\rightarrow \frac{C_0}{C_0+\lambda e^{\gamma x^{\frac{1}{\alpha}}}}$ is positive and decreasing, therefore we have
\begin{align*}
\textsl{df}(\lambda)\leq \int_{0}^{+\infty}\frac{C_0}{C_0+\lambda e^{\gamma x^{\frac{1}{\alpha}}}} dx \; .
\end{align*}
Let us consider the following substitution
\begin{align*}
u=\lambda e^{\gamma x^\frac{1}{\alpha}} \; .
\end{align*}
Therefore we have
\begin{align*}
\textsl{df}(\lambda)&\leq \int_{\lambda}^{\infty}C_0\gamma^{-\alpha}\alpha\frac{(\log(u\lambda^{-1}))^{\alpha-1}}{C_0 u+u^2}du\\
&\leq  \int_{\lambda}^{1}C_0\gamma^{-\alpha}\alpha\frac{(\log(u\lambda^{-1}))^{\alpha-1}}{C_0 u+u^2}du+ \int_{1}^{\infty}C_0\gamma^{-\alpha}\alpha\frac{(\log(u\lambda^{-1}))^{\alpha-1}}{C_0 u+u^2}du
\end{align*}
Therefore for all $\lambda\leq e^{-1}$, we obtain that
\begin{align*}
\textsl{df}(\lambda)\leq  \int_{\lambda}^{1}\gamma^{-\alpha}(\alpha)\frac{(\log(u\lambda^{-1}))^{\alpha-1}}{u}du+C_0\gamma^{-\alpha}\alpha\log(\lambda^{-1})^{\alpha-1}\int_{1}^{\infty} \frac{(\log(u)+1)^{\alpha-1}}{C_0 u+u^2}du
\end{align*}
Finally we obtain that
\begin{align*}
\textsl{df}(\lambda)\leq \gamma^{-\alpha}[\log(\lambda^{-1})^{\alpha}] +\log(\lambda^{-1})^{\alpha-1}[C_0\gamma^{-\alpha}\alpha \int_{1}^{\infty} \frac{(\log(u)+1)^{\alpha-1}}{C_0 u+u^2}du]
\end{align*}
Therefore we have
\begin{align*}
\textsl{df}(\lambda)&\leq Q \log(\lambda^{-1})^{\alpha}\\
\text{where \quad} Q&=\gamma^{-\alpha} \left[1+C_0\int_{1}^{\infty} \frac{(\log(u)+1)^{\alpha-1}}{C_0 u+u^2}du\right]
\end{align*}
\end{proof}

Let us now split the error $\Vert f_{H,\mathbf{z},\lambda}-f_{\rho}\Vert_{\rho}^2$ into two parts.
\begin{equation*}
\Vert f_{H,\mathbf{z},\lambda}-f_{\rho}\Vert_{\rho}^2\leq 2 \Vert f_{H,\mathbf{z},\lambda}-f_{H,\rho,\lambda}\Vert_{\rho}^2 + 2\Vert f_{H,\rho,\lambda}-f_{\rho}\Vert_{\rho}^2
\end{equation*}
The following Lemma provides a control of the approximation error.
\begin{lemma}
\label{lem:approx-error}
Let $0 < \beta \leq 2$, $\rho$ a probability measure on $\mathcal{X} \times\mathcal{Y}$ and H a separable RKHS on $\mathcal{X}$ with respect to a bounded and measurable kernel k. If there exist $g\in L_2^{d\rho_{\mathcal{X}}}(\mathcal{X})$ such that $f_\rho=T_{\rho}^{\beta/2}g$, then for all $\lambda > 0$ it holds
$$\Vert f_{H,\rho,\lambda}-f_{\rho}\Vert_{\rho}^2 \leq  \lambda^{\beta} \Vert g\Vert^{2}_{\rho}$$ 
\end{lemma}
\begin{proof}
By denoting $a_i:=\langle f_{\rho},e_i\rangle_{L_2^{d\rho_{\mathcal{X}}}}$ we have
\begin{align*}
\Big\Vert f_{H,\rho,\lambda}-f_{\rho}\Big\Vert_{\rho}^2&=\Big\Vert \sum_{i\in\mathcal{I}}\frac{\mu_i}{\mu_i+\lambda}a_i e_i - \sum_{i\in\mathcal{I}} a_i e_i\Big\Vert_{\rho}^2\\
&=\lambda^2 \sum_{i\in \mathcal{I}}  \left(\frac{\mu_i^{\frac{\beta}{2}}}{\mu_i+\lambda}\right)^{2}\mu_{i}^{-\beta}a_i^2
\end{align*}
Let us consider the following function for $\lambda>0$ and $0\leq \gamma\leq 1$ $f_{\lambda,\gamma}:t\in\mathbb{R}_{+}\rightarrow \frac{t^{\gamma}}{t+\lambda}$. By considering the derivative of $f_{\lambda,\gamma}$, we obtain that $\text{sup}_{t\in\mathbb{R}_{+}}f_{\lambda,\gamma}\leq \lambda^{\gamma-1}$, therefore we have
\begin{align*}
\Vert f_{H,\rho,\lambda}-f_{\rho}\Vert_{\rho}^2&\leq \lambda^{\beta}\sum_{i\in\mathcal{I}}\mu_{i}^{-\beta}a_i^2\\
&\leq \lambda^{\beta} \Vert g\Vert^{2}_{\rho}
\end{align*}
\end{proof}
Let us now focus on the estimation error. Let $\tau \geq 1$, $\lambda>0$. Thanks to theorem \ref{thm:stein-oracle-1} we have with $\rho^{\ell}$-probability $\geq 1-4e^{-\tau}$,
\begin{equation}
\label{eq:main-upper-inequality}
\Vert f_{\mathbf{z},\lambda}-f_{\rho,\lambda}\Vert_{\rho}^2\leq 128\frac{\tau^2}{\ell}\left(5\text{df}(\lambda)\sigma_{\lambda}^2+ K \frac{L_{\lambda}}{\ell\lambda}   \right)
\end{equation}
as soon as  $\ell\geq N_{\lambda,\tau}$ where $K = \sup_{x\in\mathcal{X}}k(x,x)$ and 
\begin{align*}
N_{\lambda,\tau}&= \max\left(\frac{256\tau^2 K\text{df}(\lambda)}{\lambda},\frac{16\tau K}{\lambda}, \tau \right)\\
\sigma_{\lambda}&=\max(\sigma,\Vert f_{\rho}- f_{\rho,\lambda}\Vert_{L_\infty^{d\rho_{\mathcal{X}}}})\\
L_{\lambda}&=\max(L,\Vert f_{\rho}-f_{\rho,\lambda}\Vert_{L_\infty^{d\rho_{\mathcal{X}}}})
\end{align*}
In the next Lemma, we exhibit a control $N_{\lambda,\tau}$, $L_{\lambda}^2$ and $\sigma_{\lambda}^2$.
\begin{lemma}
\label{lem:control-const}
Let $\rho\in  \mathcal{F}_{H,\alpha,\beta}$ be a probability measure. Then there are constants $N, V >0$ depending only on $\mathcal{F}_{H,\alpha,\beta}$ such that $N_{\lambda,\tau}\leq N \frac{\tau^2\log(\lambda^{-1})^{\alpha}}{\lambda}$  and $L_{\lambda}^2,\sigma_{\lambda}^2\leq \frac{V}{\lambda^{\text{max}(1-\beta,0)}}$ for all $0<\lambda \leq e^{-1}$ and $\tau\geq 1$.
\end{lemma}
\begin{proof}
Thanks to Lemma \ref{lem:effective-dim-geo}, we have that
\begin{align*}
\textsl{df}(\lambda)&\leq Q \log(\lambda^{-1})^{\alpha}\\
\text{where \quad} Q&=\gamma^{-\alpha} \left[1+C_0\int_{1}^{\infty} \frac{(\log(u)+1)^{\alpha-1}}{C_0 u+u^2}du\right]
\end{align*}
And by definition we have
\begin{align*}
N_{\lambda,\tau}&\leq \text{max}\left(\frac{256\tau^2 K\text{df}(\lambda)}{\lambda},\frac{16\tau K}{\lambda}, \tau \right)\\
&\leq N \frac{\tau^2 \log(\lambda^{-1})^{\alpha}}{\lambda}\\
\text{where \quad} N&=\text{max}(256KQ,16K,1)
\end{align*}
Moreover we have
\begin{align*}
\Vert f_{\rho}- f_{H,\rho,\lambda}\Vert_{L_\infty^{d\rho_{\mathcal{X}}}}^2&\leq 2(\Vert f_{\rho}\Vert_{L_\infty^{d\rho_{\mathcal{X}}}}^2+\Vert f_{\rho,\lambda}\Vert_{L_\infty^{d\rho_{\mathcal{X}}}}^2\\
&\leq 2B_{\infty}+2K\Vert f_{H,\rho,\lambda}\Vert_{H}^2)\\
\end{align*}
But by denoting $a_i:=\langle f_{\rho},e_i\rangle_{L_2^{d\rho_{\mathcal{X}}}(\mathcal{X})}$ we remark that
\begin{align*}
\Vert f_{H,\rho,\lambda}\Vert_{H}^2=\Big\Vert \sum_{i\in I}\frac{\mu_i}{\mu_i+\lambda}a_i e_i\Big\Vert_H^2= \sum_{i\in I}\left(\frac{\mu_i}{\mu_i+\lambda}\right)^2 a_i^2 \mu_i^{-1}=\sum_{i\in I}\left(\frac{\mu_{i}^{\frac{\beta+1}{2}}}{\mu_i+\lambda}\right)^2 a_i^2 \mu_i^{-\beta}
\end{align*}
Indeed the first equality is due to the fact that we have assumed the existence of $\beta>0$ and $g\in L_2^{d\rho_{\mathcal{X}}}(\mathcal{X})$ such that $f_\rho=T_{\rho}^{\beta/2} g$.
Let us now consider the following function for $\lambda>0$ and $0\leq \gamma\leq 1$ $f_{\lambda,\gamma}:t\in\mathbb{R}_{+}\rightarrow \frac{t^{\gamma}}{t+\lambda}$. By considering the derivative of $f_{\lambda,\gamma}$, we obtain that $\sup\limits_{t\in\mathbb{R}_{+}}f_{\lambda,\gamma}\leq \lambda^{\gamma-1}$
Therefore if $0<\beta \leq 1$ we have that
\begin{align*}
\Vert f_{H,\rho,\lambda}\Vert_{H}^2 \leq \lambda^{\beta-1} \Vert g\Vert_{\rho}^{2}\leq \lambda^{\beta-1} B  
\end{align*}
Finally if $\beta>1$ then $f_{\rho}\in H$ and we have
\begin{align*}
\Vert f_{\rho}- f_{H,\rho,\lambda}\Vert_{L_\infty^{d\rho_{\mathcal{X}}}}^2\leq K\Vert f_{\rho}- f_{H,\rho,\lambda}\Vert_{H}^2
\end{align*}
But we have
\begin{align*}
\Vert f_{\rho}- f_{H,\rho,\lambda}\Vert_{H}^2=\Big\Vert \sum_{i\in I}\frac{\lambda}{\mu_i+\lambda} a_i e_i\Big\Vert_{H}^2=\sum_{i\in I} \left(\frac{\lambda}{\mu_i+\lambda}\right)^2 a_i^2\mu_i^{-1}=\lambda^2\sum_{i\in I} \left(\frac{\mu_i^{\frac{\beta-1}{2}}}{\mu_i+\lambda}\right)^2 a_i^2\mu_i^{-\beta}
\end{align*}
And as $\beta > 1$, we obtain that
\begin{align*}
\Vert f_{\rho}- f_{H,\rho,\lambda}\Vert_{H}^2\leq \lambda^{\beta-1}\Vert g\Vert_{\rho}\leq B 
\end{align*}
Finally by choosing $V=\text{max}(L^2,\sigma^2,2BK+ 2B_{\infty})$ we obtain that
$L_{\lambda}^2,\sigma_{\lambda}^2\leq \frac{V}{\lambda^{\text{max}(1-\beta,0)}}$.
\end{proof}
We can now prove the Theorem. If $\ell\geq N\frac{\tau^2 \log(\lambda^{-1})^{\alpha}}{\lambda}$ with $N$ a constant from Lemma \ref{lem:control-const} we obtain that
\begin{align*}
\Vert f_{H,\mathbf{z},\lambda}-f_{H,\rho,\lambda}\Vert_{\rho}^2&\leq
C_1\left[\frac{\tau^2}{\ell \lambda^{\text{max}(0,1-\beta)}}\left(\log(\lambda^{-1})^{\alpha}+\frac{1}{\ell \lambda}\right)\right]\\
\text{where\quad } C_1&=128*V\text{max}(5*Q,K)
\end{align*}
with a $\rho^{\ell}$-probability $\geq 1-e^{-4\tau}$. We finally have
\begin{align*}
\Vert f_{H,\mathbf{z},\lambda}-f_{\rho}\Vert_{\rho}^2&\leq
C\left[\lambda^{\beta}+\frac{\tau^2}{\ell \lambda^{\text{max}(0,1-\beta)}}\left(\log(\lambda^{-1})^{\alpha}+\frac{1}{\ell \lambda}\right)\right]\\
\text{where\quad } C&=2*\text{max}(B,128*V\text{max}(5*Q,K))
\end{align*}
with a $\rho^{\ell}$-probability $\geq 1-e^{-4\tau}$.
Now if we assume that $\beta>1$ and $\lambda_{\ell} = \frac{1}{\ell^{1/\beta}}$, we obtain that
\begin{align*}
\Vert f_{H,\mathbf{z},\lambda}-f_{\rho}\Vert_{\rho}^2&\leq
C\left[\frac{1}{\ell}+\frac{\tau^2}{\ell}\left(\frac{1}{\beta^{\alpha}}\log(\ell)^{\alpha}+\frac{1}{\ell^{1-\frac{1}{\beta}}}\right)\right]\\
&\leq 3C\tau^2\frac{\log(\ell)^{\alpha}}{\ell}
\end{align*}
with a $\rho^{\ell}$-probability $\geq 1-e^{-4\tau}$ provided that
\begin{align}
\label{condition-ell}
 \ell\geq  \text{max}\left(e^{\beta}, N\frac{\tau^2 \log(\lambda_{\ell}^{-1})^{\alpha}}{\lambda_{\ell}}\right)
\end{align}
Moreover as $\beta>1$, $\frac{\ell \lambda_{\ell}}{\log(\lambda_{\ell}^{-1})^{\alpha}}=\frac{\ell^{1-\frac{1}{\beta}}}{\frac{1}{\beta^{\alpha}}\log(\ell)^{\alpha}}$ goes to infinity as $\ell$ goes to infinity, we conclude that there exist $\ell_{\tau}$ such that for all $\ell\geq \ell_{\tau}$, the condition (\ref{condition-ell}) is satisfied and we finally have
\begin{equation*}
\lim_{\tau\rightarrow +\infty}\lim \sup_{\ell\rightarrow\infty}\sup_{\rho\in \mathcal{F}_{H,\alpha,\beta}} \rho^{\ell}\left(\mathbf{z}:\Vert f_{H,\mathbf{z},\lambda_{\ell}}-f_{\rho}\Vert_{\rho}^2>\tau \frac{\log(\ell)^{\alpha}}{\ell}\right)=0
\end{equation*}
Now if we consider the case where $0<\beta< 1$, by considering $\lambda_\ell=\frac{\log(\ell)^{\frac{\alpha}{\beta}}}{\ell}$ we obtain
 \begin{align*}
 \Vert f_{H,\mathbf{z},\lambda}-f_{\rho}\Vert_{\rho}^2&\leq
C\left[\frac{\log(\ell)^{\alpha}}{\ell^{\beta}}+\frac{\tau^2}{\ell^{\beta}}\left([\log(\ell)-\log(\log(\ell)^{\frac{\alpha}{\beta}})]^{\alpha}+\frac{1}{\log(\ell)^{\frac{\alpha}{\beta}}}\right)\right]\\
&\leq 3C\tau^2\frac{\log(\ell)^{\alpha}}{\ell^{\beta}}
\end{align*}
with a $\rho^{\ell}$-probability $\geq 1-e^{-4\tau}$ provided that 
\begin{align*}
\ell\geq \text{max}\left(N\frac{\tau^2 \log(\lambda_{\ell}^{-1})^{\alpha}}{\lambda_{\ell}},e^1\log(\ell)^{\frac{\alpha}{\beta}}\right)
\end{align*}
Moreover we have that
\begin{align*}
\frac{\ell \lambda_{\ell}}{\log(\lambda_{\ell}^{-1})^{\alpha}}=\frac{\log(\ell)^{\frac{\alpha}{\beta}}}{[\log(\ell)-\log(\log(\ell)^{\frac{\alpha}{\beta}})]^{\alpha}}
\end{align*}
And as $0<\beta<1$, we finally have that  $\frac{\ell \lambda_{\ell}}{\log(\lambda_{\ell}^{-1})^{\alpha}}\rightarrow \infty$ as $\ell$ goes to infinity and we have
\begin{equation*}
\lim_{\tau\rightarrow +\infty}\lim \sup_{\ell\rightarrow\infty}\sup_{\rho\in \mathcal{F}_{H,\alpha,\beta}} \rho^{\ell}\left(\mathbf{z}:\Vert f_{H,\mathbf{z},\lambda_{\ell}}-f_{\rho}\Vert_{\rho}^2>\tau \frac{\log(\ell)^{\alpha}}{\ell^{\beta}}\right)=0
\end{equation*}

Finally let consider the case where $\beta =1$. By considering $\lambda_{\ell}=\frac{\log(\ell)^{\mu}}{\ell}$ with $\mu>\alpha>0$ we obtain
 \begin{align*}
 \Vert f_{H,\mathbf{z},\lambda}-f_{\rho}\Vert_{\rho}^2\leq 3C\tau^2\frac{\log(\ell)^{\mu}}{\ell}
\end{align*}
with a $\rho^{\ell}$-probability $\geq 1-e^{-4\tau}$ provided that 
\begin{align*}
\ell\geq \text{max}\left(N\frac{\tau^2 \log(\lambda_{\ell}^{-1})^{\alpha}}{\lambda_{\ell}}, e^1\log(\ell)^{\mu} \right)
\end{align*}
Finally as $\mu>\alpha>0$ we have that
 $\frac{\ell \lambda_{\ell}}{\log(\lambda_{\ell}^{-1})^{\alpha}}=\frac{\log(\ell)^{\mu}}{[\log(\ell)-\log(\log(\ell)^{\frac{\mu}{\beta}})]^{\alpha}}\rightarrow \infty$ as $\ell$ goes to infinity and the asymptotic result follows.
 \medbreak
 As a reminder, we have obtained the result for $C=2*\text{max}(B,128*V\text{max}(5*Q,K))$ and $N= \text{max}(256KQ,16K,1)$ where $Q=\gamma^{-\alpha} \left[1+C_0\int_{1}^{\infty} \frac{(\log(u)+1)^{\alpha-1}}{C_0 u+u^2}du\right]$,  $V=\text{max}(L^2,\sigma^2,2BK+ 2B_{\infty})$, and $K=  \sup_{x\in\mathcal{X}}k(x,x)$.

\end{proof}


From the above theorem, we obtain the following asymptotic upper rate of convergence
\begin{equation*}
\lim_{\tau\rightarrow +\infty}\lim \sup_{\ell\rightarrow\infty}\sup_{\rho\in \mathcal{F}_{H,\alpha,\beta}} \rho^{\ell}\left(\mathbf{z}:\Vert f_{H,\mathbf{z},\lambda_{\ell}}-f_{\rho}\Vert_{\rho}^2>\tau a_{\ell}\right)=0
\end{equation*}
if one of the following conditions hold
\begin{itemize}
    \item $\beta> 1$, $\lambda_{\ell}=\frac{1}{\ell^{1/\beta}}$ and $a_{\ell} = \frac{\log(\ell)^{\alpha}}{\ell}$
    \item $\beta = 1$, $\lambda_{\ell}=\frac{\log(\ell)^{\mu}}{\ell}$ and $a_{\ell} = \frac{\log(\ell)^{\mu}}{\ell}$ for $\mu>\alpha>0$
    \item $\beta < 1$, $\lambda_\ell=\frac{\log(\ell)^{\frac{\alpha}{\beta}}}{\ell}$ and $a_{\ell} = \frac{\log(\ell)^{\alpha}}{\ell^{\beta}}$
\end{itemize}

\subsection{Lower rate of convergence}
In order to investigate the optimality of the convergence rates, let us take a look at the lower rates. 
\begin{thm}
\label{thm:lower-rate}
Let $H$ be a separable RKHS on $\mathcal{X}$ with respect to a bounded and measurable kernel k, $q\geq \gamma >0$, $\alpha>0$, $0<\beta\leq 2$ such that $\mathcal{P}_{H,\alpha,q}$ is not empty. Then we have 
\begin{align*}
 \lim_{\tau\rightarrow 0^{+}}\lim\inf_{\ell\rightarrow\infty}\inf_{f_{\mathbf{z}}}\sup_{\rho\in\mathcal{F}_{H,\alpha,q,\beta}} \rho^{\ell}\left(\mathbf{z}:\Vert f_{\mathbf{z}}-f_{\rho}\Vert_{\rho}^2>\tau w_{\ell}\right)=1
\end{align*}
where $w_{\ell}=\log(\ell)^{\alpha}/\ell$. The infimum is taken over all measurable learning methods in $\mathcal{F}_{H,\alpha,q,\beta}$.
\end{thm}

\begin{proof}
We follow the suggestion of presented in section 3 of \cite{fischer2017sobolev} in order to construct a family of probability measures $\rho_{f}\in\mathcal{F}_{H,\alpha,q,\beta}$ parametrized by suitable vectors $f\in H$. Let $\nu\in \mathcal{P}_{H,\alpha,q}$  and let us denote $\hat{\sigma} = \text{min}(\sigma, L)$. Then we define for a measurable function $f: \mathcal{X}\rightarrow \mathcal{Y}$ and $x\in\mathcal{X}$ the distribution $\rho_f(.|x):= \mathcal{N}(f(x),\hat{\sigma}^2)$ as the normal distribution on $\mathcal{Y} = \mathbb{R}$ with mean $f(x)$ and variance $\hat{\sigma}^2$. Hence $\rho_f(A)=\int_{\mathcal{X}}\int_{\mathcal{Y}}\mathbf{1}_{A}(x,y) d\rho_{f}(y|x)d\nu(x)$ defines a probability measure on $\mathcal{X}\times \mathcal{Y}$ with marginal distribution $\nu$ on $\mathcal{X}$, i.e. $(\rho_f)_{\mathcal{X}} = \nu$. 
For $f\in L_2^{d\nu}(\mathcal{X})$ we have $\int_{\mathcal{X}\times\mathcal{Y}} y^2 d\rho_{f}(x,y) = \hat{\sigma}^2 + \Vert f \Vert_{L_2^{d\nu}(\mathcal{X})}^2 <\infty$ and $f_{\rho_{f}}=f$. Moreover, the properties of the normal distribution
implies $\int_{\mathcal{Y}} |y-f(x)|^m d\rho_{f}(y|x)\leq \frac{1}{2}m!\hat{\sigma}^m$ for all $x\in\mathcal{X}$. Hence if $\Vert f \Vert^{2}_{L_{\infty}^{d\nu}(\mathcal{X})}<B_{\infty}$  and there exist $g\in L_2^{d\nu}(\mathcal{X})$ such that $f_{\rho_{f}}=T_{\rho_{f}}^{\beta/2}g$ and $\Vert g\Vert_{L_2^{d\nu}(\mathcal{X})}^2\leq B$, then $\rho_f\in \mathcal{F}_{H,\alpha,q,\beta}$. So we reduced the construction of probability measures to the construction of appropriate functions $f$. To this end we use binary strings $\omega = (\omega_1,...,\omega_m)\in\{-1,1\}^{m}$ and define for $0<\epsilon <1$
\begin{align*}
g_{\omega}&:=\left(\frac{\epsilon}{m}\right)^{1/2}\sum_{i=1}^m \omega_i \mu_i^{-\beta/2} e_i\\
\text{and \quad} f_{\omega}&:=T_{\nu}^{\beta/2}g_{\omega}=\left(\frac{\epsilon}{m}\right)^{1/2}\sum_{i=1}^m \omega_i  e_i
\end{align*}
Because $f_{\omega}$ is a finite linear combination of the eigenvectors $e_i$ of $T_{\nu}$ it holds $\Vert f_{\omega} \Vert^{2}_{L_{\infty}^{d\nu}(\mathcal{X})}<\infty$  and $\Vert g_{\omega}\Vert_{L_2^{d\nu}(\mathcal{X})}^2\leq +\infty$. First we want to establish sufficient conditions on $\epsilon$ and $m$ such that $\Vert f_{\omega} \Vert^{2}_{L_{\infty}^{d\nu}(\mathcal{X})}<B_{\infty}$ and $\Vert g_{\omega}\Vert_{L_2^{d\nu}(\mathcal{X})}^2\leq B$. In the following we denote  $\Vert .\Vert_{\nu}:=\Vert .\Vert_{L_2^{d\nu}(\mathcal{X})}$. 
\begin{lemma}
\label{lem:epsilon_1}
Let $H$ be a separable RKHS on $\mathcal{X}$ with respect to a bounded and measurable kernel k, $q\geq\gamma >0$, $\alpha>0$, $0<\beta\leq 2$ such that $\mathcal{P}_{H,\alpha,q}$ is not empty and let $\nu\in \mathcal{P}_{H,\alpha,q}$. Then there is are constants $U,v> 0$ and $0 < \epsilon_1 \leq 1$ depending only on $\mathcal{F}_{H,\alpha,q,\beta}$, such that $\Vert f_{\omega} \Vert^{2}_{L_{\infty}^{d\nu}(\mathcal{X})}<B_{\infty}$ and $\Vert g_{\omega}\Vert_{\nu}^2\leq B$ holds for all $0 < \epsilon \leq \epsilon_1$ and all $m \leq  U\log(v\epsilon^{-1})^{\alpha}$.
\end{lemma}
\begin{proof}
Let $m\in\mathbb{N}$ and $0<\epsilon<1$. Therefore we have
\begin{align*}
 \Vert g_{\omega}\Vert_{\nu}^2= \frac{\epsilon}{m}\sum_{i=1}^m \omega_i \mu_{i}^{-\beta}
 \leq \epsilon \mu_{m}^{-\beta}
 \leq c^{-\beta}\epsilon e^{q\beta m^{1/\alpha}}
\end{align*}
Moreover we have
\begin{align*}
\Vert f_{\omega} \Vert^{2}_{L_{\infty}^{d\nu}(\mathcal{X})}\leq K\Vert f_{\omega}\Vert_{H}^2\leq \frac{K}{c}\epsilon e^{q m^{1/\alpha}}
\end{align*}
Therefore by considering $U:=\text{min}(\frac{1}{q^{\alpha}},\frac{1}{(q\beta)^{\alpha}})$ and $v:=\text{min}(\frac{BC}{K},B_{\infty}c^{\beta})$, we obtain that for all $\epsilon\leq \epsilon_1:= \text{min}(1,v)$ and all $m \leq  U\log(v\epsilon^{-1})^{\alpha}$, $\Vert f_{\omega} \Vert^{2}_{L_{\infty}^{d\nu}(\mathcal{X})}<B_{\infty}$ and $\Vert g_{\omega}\Vert_{\nu}^2\leq B$.
\end{proof}

If $\omega' = (\omega_1', . . ., \omega_m') \in\{-1,1\}^m$ is an other binary string, we investigate the norm of the difference $f_{\omega}-f_{\omega'}$. We obtain that
\begin{align*}
\Vert f_{\omega}-f_{\omega'}\Vert_{\nu}^2=\frac{\epsilon}{m}\sum_{i=1}^m(\omega_i-\omega_i')^2
\end{align*}
Therefore as $(\omega_i-\omega_i')^2\leq 4$ we obtain that:
\begin{align*}
\Vert f_{\omega}-f_{\omega'}\Vert_{\nu}^2\leq 4\epsilon
\end{align*}
In order to obtain a lower bound,  we assume that $\sum_{i=1}^m (\omega_i - \omega_i')^2 \geq m$, i.e. the distance between $\omega$ and $\omega'$ is large, and finally we obtain
\begin{align*}
\Vert f_{\omega}-f_{\omega'}\Vert_{\nu}^2\geq \epsilon
\end{align*}
The following theorem is a restatement of Theorem 3.1 of \cite{devore2006approximation} in our setting.
\begin{thm}
\label{thm:before-thm}
Let $H$ be a separable RKHS on $\mathcal{X}$ with respect to a bounded and measurable kernel k, $q\geq \gamma>0$, $\alpha>0$, $0<\beta\leq 2$ such that $\mathcal{P}_{H,\alpha,q}$ is not empty and let $\nu\in \mathcal{P}_{H,\alpha,q}$. Let also $\mathbf{z}\rightarrow f_{\mathbf{z}}$ an arbitrary measurable learning method  for $\ell\in\mathbb{N}$ and $\mathbf{z}\in Z^{\ell}$.  Then there exist $0<\epsilon_0\leq 1$ such that for all $\epsilon \leq \epsilon_0$ and for all $\ell\in\mathbb{N}$ there is a $\rho\in \mathcal{F}_{H,\alpha,q,\beta}$ such that 
\begin{align*}
\rho^{\ell}\left(\mathbf{z}: \Vert f_{\mathbf{z}}-f_{\rho}\Vert_{\rho}^2>\frac{\epsilon}{4}\right)\geq\min\left(\frac{N_{\epsilon}^*}{N_{\epsilon}^* +1}, \hat{\eta}\sqrt{N_{\epsilon}^*}e^{-\frac{4\epsilon \ell}{2\hat{\sigma}^2}}\right)
\end{align*}
where $N_{\epsilon}^*=e^{\frac{U}{48}\log(v\epsilon^{-1})^{\alpha}}$ and $\hat{\eta}=e^{-3/e}$.
\end{thm}
\begin{proof}
Thanks to Lemma \ref{lem:epsilon_1},  there exists $1\geq \epsilon_1>0$ and $U,v>0$ such that for all $\omega$, $0<\epsilon\leq \epsilon_1$ and $m\leq U\log(v\epsilon^{-1})^{\alpha}$ we have $\rho_{f_{\omega}}\in\mathcal{F}_{H,\alpha,q,\beta}$. Let us now fix $0<\epsilon<\epsilon_1$, and consider the case where $m=\lfloor U\log(v\epsilon^{-1})\rfloor$. Moreover Theorem \ref{prop:number_of_func} suggests that there are many binary strings with large distances. Indeed as soon as $m\geq 16$, 
there exists $f_{\omega^1},...,f_{\omega^{N_{\epsilon}}}$ where $N_{\epsilon}\geq e^{m/24}\geq N_{\epsilon}^{*} $ such that $\rho_{f_{\omega^1}},...,\rho_{f_{\omega^{N_{\epsilon}}}}\in\mathcal{F}_{H,\alpha,q,\beta}$, and for all $i\neq j\in[|1,N_{\epsilon}|]$ we have: 
\begin{align*}
\label{eq:lower-bound}
4\epsilon \geq \Vert f_{\omega^{i}}-f_{\omega^{j}}\Vert_{\nu}^2\geq \epsilon
\end{align*}
In fact if we assume that $\epsilon\leq\epsilon_0:= \text{min}(\epsilon_1,v e^{-(\frac{17}{U})^{1/\alpha}})$, then $m\geq 16$ is satisfied and the above results hold. Now, given $\ell\in\mathbb{N}$, let
\begin{align*}
A_i:=\left\{\mathbf{z}:\Vert f_{\mathbf{z}}-f_{\omega^{i}}\Vert_{\nu}^2 <\frac{\epsilon}{4} \right\}
\end{align*}
for all $i=1,...,N^{\epsilon}$. Thanks to the lower bound Eq. \ref{eq:lower-bound}, we obtain that $A_i\cap A_j = \emptyset$ 
if $i\neq j$, therefore thanks to Lemma \ref{lem:devore-approx}  we have that there exist $i\in[|1,N_{\epsilon}|]$, such that $\rho=\rho_{f_{\omega^i}}$ and that either
\begin{align*}
p:=\rho^{\ell}(A_i)>\frac{N_{\epsilon}}{N_{\epsilon}+1}\geq \frac{N_{\epsilon}^*}{N_{\epsilon}^*+1}
\end{align*}
 or
\begin{align*}
\frac{4\epsilon \ell}{2\hat{\sigma}^2}\geq -\log(p)+\log(\sqrt{N_{\epsilon}^*})-3/e
\end{align*}
The left-hand side of the latter inequality comes from the fact that we can describe the Kullback-Leibler divergence for these measures. Indeed thanks to Lemma \ref{lem:KL-div} for $f,f'\in L_2^{d\nu}(\mathcal{X})$ and $\ell \geq  1$ it holds $\rho_f^{\ell} \ll \rho_{f'}^{\ell}$ and $\rho_f^{\ell} \gg \rho_{f'}^{\ell}$ . Furthermore, the KL
divergence fulfills
\begin{align*}
\text{KL}(\rho_f^{\ell},\rho_{f'}^{\ell})=\frac{\ell}{2\hat{\sigma}^2}\Vert f-f'\Vert_{L_2^{d\nu}(\mathcal{X})}^2
\end{align*}
And as $f_{\rho}=f_{\rho_{f_{\omega^i}}}=f_{\omega^i}$ we obtain that
\begin{align*}
\rho^{\ell}\left(\mathbf{z}: \Vert f_{\mathbf{z}}-f_{\rho}\Vert_{\rho}^2>\frac{\epsilon}{4}\right)\geq\text{min}\left(\frac{N_{\epsilon}^*}{N_{\epsilon}^* +1}, \hat{\eta}\sqrt{N_{\epsilon}^*}e^{-\frac{4\epsilon \ell}{2\hat{\sigma}^2}}\right)
\end{align*}
\end{proof}
We can now give a proof of the theorem. Given $\tau>0$, for all $\ell\in\mathbb{N}$, let $\epsilon_{\ell}=4\tau\frac{\log(\ell)^{\alpha}}{\ell}$. Since $\epsilon_{\ell}$ goes to $0$ when $\ell$ goes to $\infty$, for $\ell$ large enough, $\epsilon_{\ell}\leq \epsilon_0$, and we can apply Theorem \ref{thm:before-thm} and we obtain
\begin{align*}
\inf_{f_{\mathbf{z}}}\sup_{\rho\in\mathcal{F}_{H,\alpha,q,\beta}}\rho^{\ell}\left(\mathbf{z}: \Vert f_{\mathbf{z}}-f_{\rho}\Vert_{\rho}^2>\tau\frac{\log(\ell)^{\alpha}}{\ell}  \right)\geq \text{min}\left(\frac{N_{\epsilon_{\ell}}^*}{N_{\epsilon_{\ell}}^* +1}, \hat{\eta}\sqrt{N_{\epsilon_{\ell}}^*}e^{-\frac{4\epsilon_{\ell} \ell}{2\hat{\sigma}^2}} \right)
\end{align*}
Moreover we have
\begin{align*}
\sqrt{N_{\epsilon_{\ell}}^*}e^{-\frac{4\epsilon_{\ell} \ell}{2\hat{\sigma}^2}}=e^{\frac{U}{96}[\log(\ell)-\frac{v}{4\tau\log(l)^{\alpha}}]^{\alpha}}\times e^{-\frac{8\tau\log(\ell)^{\alpha}}{\hat{\sigma}^2}}
\end{align*}
Therefore if $\tau$ is small enough such that $\tau<\frac{U\hat{\sigma}^2}{768}$, the quantity $\sqrt{N_{\epsilon_{\ell}}^*}e^{-\frac{4\epsilon_{\ell} \ell}{2\hat{\sigma}^2}}$ goes to $\infty$ as $\ell$ goes to $\infty$. Also if $\ell$ goes to $\infty$, $\frac{N_{\epsilon_{\ell}}^*}{N_{\epsilon_{\ell}}^* +1}$ goes to $1$. Finally we obtain that: 
\begin{align*}
 \lim_{\tau\rightarrow 0^{+}}\lim\inf_{\ell\rightarrow\infty}\inf_{f_{\mathbf{z}}}\sup_{\rho\in\mathcal{F}_{H,\alpha,q,\beta}} \rho^{\ell}\left(\mathbf{z}:\Vert f_{\mathbf{z}}-f_{\rho}\Vert_{\rho}^2>\tau b_{\ell}\right)=1
\end{align*}
\end{proof}

\textbf{Discussion.} When $\beta>1$ the convergence rates of RLS stated in Theorem~\ref{thm:upper-rate} coincide with the minimax lower rates from Theorem \ref{thm:lower-rate}, \textit{i.e.}~the rates are optimal in the minimax sense in this context. Note that, to the best of our knowledge, previous results in nonparametric learning assumed a polynomial eigenvalue decay of the integral operator; see \textit{e.g.} recent works ~\cite{blanchard2018optimal,fischer2017sobolev} and references therein. We show here that optimal rates for regularized least-squares still hold when the eigenvalue decay is geometric. 